\documentclass[twoside,11pt]{article}

\PassOptionsToPackage{dvipsnames}{xcolor}

\usepackage[preprint]{jmlr2e}

\usepackage[margin=1in]{geometry}
\usepackage[utf8]{inputenc}
\usepackage[T1]{fontenc}
\usepackage{url}
\usepackage{booktabs}
\usepackage{amsfonts}
\usepackage{nicefrac}
\usepackage{microtype}
\microtypesetup{expansion=false}
\usepackage{graphicx}
\usepackage{xcolor}
\usepackage{tikz}
\usetikzlibrary{arrows.meta,calc,decorations.pathreplacing,positioning,patterns}
\usepackage{subfig}
\usepackage{amssymb,amsbsy,amsmath,amsfonts,amssymb,amscd}
\usepackage{mathtools}
\usepackage{comment}
\usepackage{verbatim}
\usepackage{enumitem}
\hypersetup{hidelinks}
\usepackage{setspace}

\newtheorem{assum}{Assumption}

\newcommand{\EE}{\mathbb{E}}

\newcommand{\ini}{\mathrm{ini}}

\newcommand{\dis}{\mathrm{dis}}

\newcommand{\rd}{\,\mathrm{d}}

\newcommand{\rev}[1]{{{#1}}}

\makeatletter
\newcommand\appendix@section[1]{%
  \refstepcounter{section}%
  \orig@section*{Appendix \@Alph\c@section. #1}%
  \addcontentsline{toc}{section}{Appendix \@Alph\c@section: #1}%
}
\let\orig@section\section
\g@addto@macro\appendix{\let\section\appendix@section}
\makeatother

\ShortHeadings{Global Convergence and Generalization of Gradient Descent for Multi-Layer ResNets}{Ding, Chen, Li, and Wright}
\firstpageno{1}

\begin{document}

\title{{Global Convergence of Gradient Descent for Multi-Layer ResNets with Homogeneous Activation Functions in the Mean-Field Regime}}

\author{\name Zhiyan Ding \email zyding@umich.edu\\
       \addr Department of Mathematics\\
       University of Michigan\\
       Ann Arbor, MI 48105 USA.
       \AND
       \name Shi Chen \email  schen636@mit.edu\\
       \addr  Department of Mathematics\\
       Massachusetts Institute of Technology\\
       Cambridge, MA 02139 USA.
       \AND
       \name Qin Li \email qinli@math.wisc.edu\\
       \addr Mathematics Department\\
       University of Wisconsin-Madison\\
       Madison, WI 53706 USA.
       \AND
       \name Stephen J. Wright \email swright@cs.wisc.edu\\
       \addr Department of Computer Sciences \\
       University of Wisconsin-Madison\\
       Madison, WI 53706 USA.
       }

\editor{}

\maketitle
\thispagestyle{plain}

\begin{abstract}
Finding the optimal configuration of parameters in ResNet is a nonconvex minimization problem, but first-order methods often find the global optimum when the network is overparameterized and the training algorithm is run for sufficiently many iterations.
We study this phenomenon in the mean-field regime, where the network can be described approximately by an ordinary integral equation (OIE) and the training process of ResNet becomes a gradient-flow partial differential equation (PDE).
Under the condition that the activation function is $2$-homogeneous or partially $1$-homogeneous, we show that this gradient-flow PDE converges to the global minimum.
This result suggests that if the ResNet is sufficiently large, first-order optimization methods can find global minimizers that fit the training data as well.
Further, by controlling the generalization error, we prove that the gradient-flow PDE is stable with respect to its cost function, implying that a finite-sized but sufficiently large dataset drawn from the underlying data distribution can exhibit the properties of the continuous limit.
We give lower bounds on the depth and width of the network for the gradient-flow approximation to hold.
We also lower-bound the size of the training data set needed to attain accurate approximation.
\end{abstract}

\begin{keywords}
   Overparameterization, Generalization, Mean-Field, ResNet, Gradient Flow.
\end{keywords}

\section{Introduction}
Training of neural networks (NN) amounts to finding optimal parameters (the weights for the connections between neurons) so that the output of the NN matches with the given data.
Formulated as an optimization problem, the objective function is nonconvex.
Nevertheless extensive computational experience shows that simple first-order methods such as gradient descent (GD) can still find a global minimum achieving the zero-loss property, if the NN is sufficiently large.

In this paper, we prove that for a sufficiently large ResNet, the problem of optimizing the NN parameters to fit the training data can be cast as an optimization problem with ordinary-integral-equation (OIE) constraints, with the GD algorithm expressed as a gradient flow.
Gradient flow is shown to converge to the ``zero-loss'' global minimum when the activation function is 2-homogeneous or partially 1-homogeneous.

More explicitly: Suppose that the training data items $x\in\mathbb{R}^d$ are drawn from the distribution $\zeta$  with target function $y:\mathbb{R}^d\to\mathbb{R}$. Denoting by $NN(x;\Theta)$ the output of the NN for parameters $\Theta$ with input data $x$, the problem of finding the optimal $\Theta$ can be formulated as
\begin{equation}\label{eqn:NN_optimization}
\min_{\Theta} \, \mathcal{E}(\Theta) = \mathbb{E}_{x\sim\zeta}\left[  \frac{1}{2} \big( y(x) - NN(x; \Theta) \big)^2 \right]\,.
\end{equation}
When the NN is ResNet, we cast this problem as
\begin{align*}
&\min_{\Theta} \, \mathbb{E}_{x\sim\zeta}\left[\frac{1}{2}\big(y(x)-g(z_L(x;\Theta)) \big)^2\right]\\
\text{s.t.} &\begin{cases}
 z_{l+1}(x) &= z_l(x)+\frac{1}{ML}\sum_{m=1}^M f(z_l(x),\theta_{l,m})\,, \quad l=0,\dots,L-1\\
 z_0(x) & =x \,, \quad x\in\mathbb{R}^d\,,
\end{cases}
\end{align*}
where $f$ is the activation function at each neuron, parameterized by $\theta_{l,m}$; $L$ is the number of layers; and $M$ is the number of neurons per layer.
The parameter vector is $\Theta\coloneqq\{\theta_{l,m}\}_{l=0,m=1}^{L-1,M}$.
When $L$ and $M$ are sufficiently large, one can view the structure of ResNet as the time discretization of an OIE, with \rev{$L(z_{l+1}-z_l)$} representing $\frac{\rd z}{\rd t}$ (\rev{viewing $1/L$ as $\Delta t$}), and $\frac{1}{M}\sum_m f(\cdot,\theta_{l,m})$ representing $\int f(\cdot,\theta) \rho(\theta,t_l)\rd\theta$, where $\theta_{l,m}, m=1,\dots,M$ are assumed to be i.i.d. sampled from a distribution $\rho(\theta,t_l)$.
The OIE-constrained optimization casting of the NN-constrained optimization is therefore
\begin{equation}
\begin{aligned}
& \min_{\rho}E(\rho)=\mathbb{E}_{x\sim\zeta}\left[\frac{1}{2}\big( y(x)-g(z(t=1,x;\rho)) \big)^2\right]\\
\text{s.t.} & \begin{cases}
\frac{\rd}{\rd t} z(t,x) &= \int f(z(t,x),\theta)\rho(\theta,t)\rd\theta\,,\quad t\in[0,1]\\
\quad z(t=0) &= x \,, \quad x\in\mathbb{R}^d
\end{cases}
\end{aligned}
\end{equation}
\rev{Although the summation over $\theta_{l,m}$ is taken outside the activation function, this formulation captures the standard case in which the activation function is applied elementwise; see the detailed discussion in Remark~\ref{re:example}.}

This recasting of NN-training also permits reformulation of the training procedure.
When GD is applied to \eqref{eqn:NN_optimization}, with training stepsize approaching zero, we have:
\begin{equation}\label{eqn:gradient_descend}
\frac{\rd}{\rd s}\Theta = -\nabla_\Theta \mathcal{E}(\Theta)\,.
\end{equation}
By applying the mean-field argument of~\citep{ding2021overparameterization}, we obtain gradient flow for the underlying distribution $\rho$ in the OIE-constrained formulation:
\begin{equation}\label{eqn:gradient_flow}
\partial_s\rho = \nabla_\theta\cdot\left(\rho\nabla_\theta\frac{\delta E}{\delta\rho}(\rho)\right)\,.
\end{equation}



Besides its intrinsic interest, this reformulation allows us to translate the analysis of GD for NN training into convergence analysis of the gradient-flow PDE~\eqref{eqn:gradient_flow}. Specifically, we can prove the following results.
\begin{itemize}
\item~\rev{The actual training process~\eqref{eqn:gradient_descend} closely approximates the mean-field limit~\eqref{eqn:gradient_flow} in the sense that
\[
\left|\mathcal{E}(\Theta(s)) - E(\rho(s))\right| \leq \mathcal{O}\left(\max\left\{\frac{1}{\sqrt{M}}, \frac{1}{L}\right\}\right)
\]
with high probability. This result is established in Theorem~\ref{thm:consistent} of Section~\ref{sec_:translation}.
}
\item \rev{$E(\rho_{s=\infty}) = 0$ holds when the function $f$ is either 2-homogeneous, i.e., $f(x,\lambda\theta) = \lambda^2 f(x,\theta)$ for all $\lambda \in \mathbb{R}_+$, or partially 1-homogeneous, i.e., $f(x,\theta) = \theta_{[1]} \widehat{f}(x,\theta_{[2]})$ for $\theta = (\theta_{[1]}, \theta_{[2]}) \in \mathbb{R} \times \mathbb{R}^{k-1}$. These two results are presented in Section~\ref{sec_:2homogeneous}, Theorem \ref{thm:globalminimal1} and Section~\ref{sec_:partially1homogeneous}, Theorem \ref{thm:globalminimal2}.}

\item \rev{The solution to~\eqref{eqn:gradient_flow} is stable with respect to the cost function. Specifically, with high probability,
\[
\left|E(\rho^N(s)) - E(\rho(s))\right| \leq \mathcal{O}\left(\frac{1}{\sqrt{N}}\right),
\]
where $\rho^N(s)$ denotes the solution to~\eqref{eqn:gradient_flow} with the empirical cost function constructed from $N$ i.i.d.\ samples drawn from $\zeta$. This result is presented in Section~\ref{sec_:generalizationproperty}, Theorem \ref{thm:generalization}.}
\end{itemize}
\rev{By assembling the three results above, we finally show that a sufficiently large ResNet trained with finitely many data points via a GD method can almost achieve the global minimum for population risk; see Section~\ref{sec_:generalizationproperty}, Theorem \ref{thm:maintheorem1}.
}

All these results are proved and discussed in the remainder of the paper.
In Section~\ref{sec:literature} we review previous results using various approaches.
Mean-field analysis has appeared in the literature  multiple times~\citep{NEURIPS2018_a1afc58c,MeiE7665,ding2021overparameterization,Zhiyan_thesis} and our derivation in this paper differs only slightly. Our main contribution lies in the gradient-flow analysis and the proof of the generalization property. We define notation and summarize assumptions in Section~\ref{sec:assumptions}, while the following three sections contain the three key components of our investigation.
\begin{itemize}
    \item[--]{Section~\ref{sec_:translation}} presents the translation from~\eqref{eqn:gradient_descend} to~\eqref{eqn:gradient_flow}. The main technical component is the derivation of mean-field limit.
    \item[--]{Section~\ref{sec_:convergetoglobalminimizer}} discusses the global convergence of the gradient-flow PDE. 
    \item[--]{In Section~\ref{sec_:generalizationproperty}} we showcase the generalization property (with $\zeta$ represented by its discrete draws $\zeta_N$)~\rev{and then present the combined complexity result}.
\end{itemize}

\section{Related Work and Perspectives}\label{sec:literature}
There is a vast literature about overparameterization of deep neural networks (DNNs) for training sets of finite size. A number of perspectives have been adopted to explain the success of first-order (GD) optimization methods in this regime. We organize these previous works loosely into three categories: landscape analysis, the neural-tangent kernel (NTK) regime, and mean-field analysis.

Landscape analysis is the most intuitive and thus the earliest approach deployed to understand overparametrization.
Contours of the nonconvex objective function were studied with a view to discovering  which properties make it possible for a first-order method to converge to the minimizer.
As an application of this theory, different NN structures have been evaluated to see how well they match the relevant properties~\citep{pmlr-v70-jin17a,pmlr-v40-Ge15,NIPS2017_f79921bb,ge2018learning,Nguyen2018OptimizationLA,pmlr-v80-du18a,10.1109/TIT.2018.2854560,pmlr-v70-nguyen17a,NIPS2016_f2fc9902,yun2018global}.
This approach limits the types of DNN that can be ``explained;'' most DNN structures do not satisfy the required properties.

Explanations based on the neural tangent kernel (NTK) have become popular.
In a beautiful paper~\citep{NEURIPS2018_5a4be1fa}, the authors lifted the problem of training the weights of NN to the training of the representing function. The function thus moves in training-time along the functional derivative projected on a function space spanned by the local tangent kernel (hence the name neural tangent kernel).
Using the law of large numbers, the authors relate the NN (and its training) to a Gaussian Process (GP) when $M$ (width) is extremely large, leading further to the fact that the kernel barely changes in training time.
For specific cases, one can further show the kernel is positive definite, so that the training loss tends to zero.
Criticism of this paper has focused on the ``non-asymptotic'' side, in the sense that the size of $M$ and its dependence on the size $N$ of training data sets  and the NN depth $L$ was not specified.
Later,~\citep{arora2019exact} spelt out the size of NN for it to be viewed as a Gaussian process.
Similar concepts were exploited in a series of papers~\citep{pmlr-v97-allen-zhu19a,pmlr-v97-du19c,zhang2019convergence,zou2019improved,cao2019generalization,chatterji2021does,du2018gradient,lee2019wide,yang2019scaling,NEURIPS2020_b7ae8fec,Frei2019AlgorithmDependentGB,10.5555/3454287.3454840,NEURIPS2020_9afe487d}, where many authors tend to treat the problem directly in the discrete setting (where $L$, $M$, and $N$ are all finite).
Conceptually, all results in this line of research reduce to the observation that the kernel (or the Gram matrix) stays constant in training time, so an argument similar to one based on the Polyak-Lojasiewicz (PL) condition can be deployed~\cite{liu2022loss}, sometimes in an implicit manner~\citep{pmlr-v97-du19c,pmlr-v97-allen-zhu19a}.
We believe that these observations and results have high value.
A dependence between $M$, $L$, and $N$ is unavoidable in this line of argument:
it is a standard practice to fix $N$ and $L$, and evaluate the dependence of $M$ on them.
This dependence is usually algebraic, which is somewhat
counter-intuitive: If the test and training datasets have good agreement (share the same underlying function)~\citep{zhang2017understanding}, then for $N$ sufficiently large, one would expect the same NN with a fixed size  to generalize well, so $M$ could be expected to be independent of $N$.
Such a property seems impossible to construct when positiveness of the Gram matrix (whose size depends on $N$) lies at the core of the analysis.
Similarly, sending both $M$ and $L$ to infinity should trigger a law of large numbers, so the GP equivalence should be available even when $L$ is not fixed, with $M$ not depending on $L$.
Again, such a property does not seem to be available in the current form of the analysis.



The third approach is the mean-field limit approach that we adopt in this paper. The term ``mean-field'' comes from the fact that a large ensemble of particles collectively form a ``field'' that then drives the dynamics of each sample. Instead of tracing the trajectory of each sample, one can characterize the evolution of the full distribution function that represents the field.
This approach originated in statistical physics, and has been made rigorous under the framework of kinetic theory. In the context of training an overparametrized ResNet, the network weights are viewed as interacting particles that evolve collectively in the training process to reduce the cost function.
A significant advantage of the mean-field approach is that once we derive a formula for the gradient flow, the convergence behavior can be described by standard PDE techniques.
Moreover, in the derivation of the gradient flow limit, the quantities $L$, $M$, and $N$ grow to infinity independently.
This feature of the result eliminates the unnatural dependence of $M$ and $L$ on $N$, and is more consistent with experimental observations~\citep{zhang2017understanding}.
The mean-field perspective has been adopted in many works~\citep{araujo2019meanfield,fang2019convex,nguyen2019mean,pmlr-v97-du19c,chatterji2021does,NEURIPS2018_a1afc58c,MeiE7665,wojtowytsch2020convergence,pmlr-v119-lu20b,doi:10.1287/moor.2020.1118,doi:10.1137/18M1192184,E_2020,jabir2021meanfield,ding2021overparameterization}.
The case of a single hidden layer NN in the regime as $M \to \infty$ is studied by \citet{NEURIPS2018_a1afc58c,MeiE7665,wojtowytsch2020convergence}, who justified the mean-field approach and demonstrated convergence of the gradient flow process to a zero objective.
In the multi-layer case, \citet{pmlr-v119-lu20b} showed the convergence of a PDE that can be viewed as a modified version of the true gradient flow, hinting at convergence of the real mean-field limit.
\citet{nguyen2021rigorous} also gave the global convergence of the mean-field limit of DNN for a certain class of NN structures, but their work excludes such important practical NN structures as ResNet.

The two papers most relevant to ours are~\citep{ding2021overparameterization} and~\citep{NEURIPS2018_a1afc58c}.
The former paper, written by the same authors as the current paper, laid out the mean-field strategy for analyzing multi-layer ResNets.
The global convergence result of that paper depended on a ``full support" condition, which is unrealistic and hard to justify in any practical training process.
We remove this assumption in the current paper.
To enable this improvement, we introduce  several assumptions about the activation function --- namely, the $2$-homogeneous or partially $1$-homogeneous conditions mentioned above.
These conditions confine the activation function to be relatively uniform in a cone and allow us to trace the evolution of weights in the limiting training process, which is the gradient-flow PDE.
Full details of our approach are given in Section~\ref{sec_:convergetoglobalminimizer}.
To summarize: The revised analysis in this paper has the same flavor as that of  \cite{ding2021overparameterization} but there are adjustments that arise from dropping of the full-support hypothesis and the new assumptions on the activation function.

In~\cite{NEURIPS2018_a1afc58c}, global convergence for the limiting gradient-flow PDE was proved under the same homogeneity conditions on the activation functions as we use in this paper.
{(See \citet{petit2026global} for a correction to the partially $1$-homogeneous case)}.
However, \cite{NEURIPS2018_a1afc58c} confine themselves to a 2-layer network.
For the multi-layer ResNet considered in our paper, we need a significantly different set of technical tools to derive the mean-field limit and prove global convergence of the limiting PDE.
In the 2-layer case, the gradient flow PDE arises from a loss function  that is {\em convex} with respect to the underlying probability measure.
In contrast, our multi-layer setting yields nonconvexity in the loss, so we need to use more refined mathematical treatments.

{
Since the first preprint version of this paper appeared in 2021, several works have developed complementary aspects of the mean-field picture.
\citet{chizat2025hidden}
derives and establishes the convergence to the same mean-field limit as ours, but under a different initialization, which is i.i.d. across layers and neurons. This new initialization allows for a different convergence rate to the mean-field limit.
\citet{barboni2025understanding} formulate the training dynamics using conditional optimal transport and prove local exponential convergence under a local Polyak--\L{}ojasiewicz condition.
\citet{petit2026global} complete the Chizat--Bach escape-region argument for shallow models with bounded nonlinearities; their construction underlies our partially $1$-homogeneous proof, which replaces full support by separation and extends the argument to the ResNet setting.
A related long-time convergence result for a regularized continuous mean-field flow is obtained by~\citet{isobe2023convergence}.
See also~\citet{gao2024global,petit2026global} for related results on transformer and attention architectures.
}

The current paper integrates many techniques from other articles in the literature, as we point out at the appropriate places.

\section{Assumptions and Notations}\label{sec:assumptions}
Here we specify notation and list all assumptions to be used in later parts of the paper.

Throughout the paper we only consider probability distributions that have bounded second moments, denoted by $\mathcal{P}^2(\mathbb{R}^k)\coloneqq\{\rho: \int_{\mathbb{R}^k}|\theta|^2\rd\rho(\theta)<\infty\}$.
We assume certain smoothness properties for the activation function $f$, the measuring function $g$, the data $y$, and the input measure $\zeta$.
The assumptions on $f$ are as follows.
\begin{assum}[Assumptions on $f$]\label{assum:f}
 $f:\mathbb{R}^{d}\times \mathbb{R}^k\rightarrow\mathbb{R}^d$ is {a $\mathcal{C}^1$ function} that satisfies the following properties.
\begin{itemize}
\item[1.] (Linear growth) There is a constant $C_1$ such that for all $x\in\mathbb{R}^d$ and $\theta\in\mathbb{R}^k$, we have
\begin{equation}\label{eqn:boundoff}
\left|f(x,\theta)\right|\leq C_1(|\theta|^2+1)(|x|+1)\,.
\end{equation}
\item[2.] (Locally Lipschitz) There is a monotonically increasing function $C_2: \mathbb{R}_+ \rightarrow \mathbb{R}_+$ such that for each $r>0$, all $x \in \mathbb{R}^d$ with $|x|<r$, and all $\theta\in\mathbb{R}^k$, we have that
\begin{equation}\label{eqn:derivativebound}
\left|\partial_{x}f (x,\theta)\right|  \leq C_2(r)(|\theta|^2+1) \quad\text{and}\quad \left|\partial_{\theta}f(x,\theta)\right|  \leq C_2(r)(|\theta|+1)\,.
\end{equation}

\item[3.] (Local smoothness) There is a monotonically increasing function $C_3: \mathbb{R}_+ \rightarrow \mathbb{R}_+$ and an integer $k_1$ with $0<k_1 \le k$ such that  the following property holds: For all  {$\theta=(\theta_{[1]},\theta_{[2]})$ and $\theta'=(\theta_{[1]}',\theta_{[2]}')$ (where $\theta_{[1]},\theta_{[1]}'\in\mathbb{R}^{k_1}$ and $\theta_{[2]},\theta_{[2]}'\in\mathbb{R}^{k-k_1}$) and all $x,x' \in \mathbb{R}^d$, defining $r=\max\{|x|,|x'|,|\theta_{[1]}|,|\theta_{[1]}'|\}$, we have
\begin{equation}\label{eqn:part3}
\max\{ |\partial_x f(x,\theta) - \partial_x f(x',\theta')|,
|\partial_\theta f(x,\theta) - \partial_\theta f(x',\theta')|\}
\leq C_3(r) (|x-x'| + |\theta - \theta'|).
\end{equation}
}
When $k_1<k$, we have in addition, {with $r=\max\{|x|,|\theta_{[1]}|\}$}, that
\begin{equation}\label{eqn:k1smallerthank}
\begin{aligned}
\max\{|\partial_xf(x,\theta)|,|f(x,\theta)|\} \leq C_3(r)\left(|\theta_{[2]}|+1\right), \quad \big|\partial_{\theta_{[1]}}f (x,\theta) \big|  \leq C_3(|x|)\left(|\theta_{[1]}|+1\right)\,.
\end{aligned}
\end{equation}

\item[4.] (Universal kernel) The set of functions defined by $\left\{h\middle|h=\int_{\mathbb{R}^k}f(x,\theta)\rd\rho(\theta),\ \rho\in \mathcal{P}^2(\mathbb{R}^k)\right\}$ is dense in {$\mathcal{C}\left(\overline{\mathcal{B}_{R}(0)};\mathbb{R}^d\right)$}
\footnote{{$\mathcal{C}\left(\overline{\mathcal{B}_{R}(0)};\mathbb{R}^d\right)$} is defined to be the set of functions $f:\mathbb{R}^d\to\mathbb{R}^d$ such that the restriction of $f$ to {the closed ball of size $R$ around the origin $\overline{\mathcal{B}_{R}(0)} = \{x\in\mathbb{R}^d||x|\leq R\}$ is continuous.}}
for all $R>0$.
\end{itemize}
\end{assum}
\rev{\begin{remark}\label{re:example}
The form of $f$ and the above assumptions are satisfied by many commonly used activation functions~\citep{e2020mathematical}. For example, it includes functions of the form $F:\mathbb{R}^{d_1} \rightarrow \mathbb{R}$ defined by
\[
F(x) = \frac{1}{M} c^\top \hat{u}(Ax + b),
\]
where $A \in \mathbb{R}^{M \times d_1}$ and $b, c \in \mathbb{R}^M$ are trainable weights, and $\hat{u}:\mathbb{R}^M \to \mathbb{R}^M$ is given by $(\hat{u}(x))_i = u(x_i)$ for a smooth, 
{non-polynomial} function $u$. In particular, $F(x)$ can be rewritten as
\[
F(x) = \frac{1}{M} \sum_{i=1}^M c_i u(A_{i,:} x + b_i) \coloneqq \frac{1}{M} \sum_{i=1}^M f(x, \underbrace{A_{i,:}, b_i, c_i}_{\coloneqq\theta_i}),
\]
where $A_{i,:}$ denotes the $i$-th row of $A$.
\end{remark}}

Some mild assumptions are also made for $g$ (the measuring function) and $y$ (the data).
\begin{assum}[Assumptions on data]\label{ass:g}
 Let $g$ and $y$ be $\mathcal{C}^2$ functions mapping $\mathbb{R}^d$ to $\mathbb{R}$, and let $\zeta$ be the probability distribution of $x$. We assume the following.
\begin{itemize}
\item[5.] $\zeta(x)$  is compactly supported, meaning that there is $R>0$ such that the support of $\zeta$ is within a ball of size $R$ around the origin, that is, $\mathrm{supp}(\zeta)\subset \mathcal{B}_{R}(0)$.
\item[6.] $y(x)\in L^\infty_{\mathrm{loc}}(\mathbb{R}^d)$, that is,  $
\sup_{|x|\leq R}|y(x)|<\infty$.
\item[7.] $g(x)$ and $\nabla g(x)$ are Lipschitz continuous and $|\nabla g(x)|$ has a positive lower bound, that is, $\inf_{x\in\mathbb{R}^d}|\nabla g(x)|>0$.
\end{itemize}
\end{assum}

\rev{We note that the above two assumptions are sufficient to rigorously establish the convergence of the mean-field limit. Compared with~\citep{ding2021overparameterization}, these assumptions are weaker in the sense that they do not impose any constraint on the growth of $f$ with respect to $|x|$.
To show the global convergence of the gradient-flow PDE~\eqref{eqn:gradient_flow}, we also require one extra assumption (2 homogeneous or partially 1-homogeneous) in the following.
This approach is adapted from \cite{NEURIPS2018_a1afc58c}.
}
\begin{assum}[$2$-homogeneous]\label{assum:2-homogeneous}
The function $f:\mathbb{R}^d\times\mathbb{R}^k \to \mathbb{R}^d$ is $2$-homogeneous, meaning that \begin{equation}\label{eqn:2-homo-assumption}
    f(x,\lambda\theta)=\lambda^2f(x,\theta)
\end{equation}
for all $(x,\theta,\lambda)\in\mathbb{R}^d\times\mathbb{R}^k\times\mathbb{R}_+$.
\end{assum}

Differentiating~\eqref{eqn:2-homo-assumption} with respect to $\lambda$ and evaluating at $\lambda = 1$ immediately suggests that
\[
\langle\theta\,,\nabla_\theta f(x,\theta)\rangle=2f(x,\theta)\,,\quad\forall (x,\theta)\in\mathbb{R}^d\times\mathbb{R}^k\,.
\]
A function that satisfies Assumption~\ref{assum:f}
and this 2-homogeneous assumption is {obtained by taking $k=2d(d+1)$ and
\[
f(x,\theta) = \left[\sigma(Ax+b) - \sigma(\widetilde{A}x+\widetilde{b})\right]\exp(-|x|^2)\,,
\]
where $A,\widetilde{A}\in\mathbb{R}^{d\times d}, b, \widetilde{b}\in\mathbb{R}^d$, $\theta = (A,\widetilde{A},b,\widetilde{b})\in\mathbb{R}^{2d(d+1)}$}, and $\sigma(x)=|\max\{x,0\}|^2$, with the $\max$ operation applied componentwise.


\begin{assum}[Partially-$1$ homogeneous]\label{assum:f2} For $\theta=(\theta_{[1]},\theta_{[2]})$ with $\theta_{[1]}\in\mathbb{R}$ and $\theta_{[2]}\in\mathbb{R}^{k-1}$, we have the following.
\begin{itemize}
\item[1.] (Partially 1-homogeneous in $\theta$) $f$ can be written as $f(x,\theta)=f(x,\theta_{[1]},\theta_{[2]})=\theta_{[1]}\widehat{f}(x,\theta_{[2]})$,
\item[2.] (Locally bounded and smooth) For any $r>0$, $\widehat{f}(x,\theta_{[2]})$ is bounded and  Lipschitz continuous, with  Lipschitz continuous gradient for $(x,\theta_{[2]})\in\mathcal{B}_r(\vec{0})\times\mathbb{R}^{k-1}$.
\end{itemize}
\end{assum}

A function that satisfies both Assumption~\ref{assum:f} and the partially 1-homogeneous assumption is
\[
f(x,\theta)=f(x,\theta_{[1]},\theta_{[2]}) = \theta_{[1]}\sigma\left(\frac{a^\top\sigma_2(|a|)}{|a|}x+\frac{b\sigma_2(|b|)}{|b|}\right)\,,
\]
with $\theta_{[1]}\in\mathbb{R}$ and  $\theta_{[2]} = (a,b)\in\mathbb{R}^d \times \mathbb{R}$.
Suppose $\sigma(x)$ is a regularized ReLU activation function (for example, $\sigma(x)=(x+\eta)^2/(4\eta)\textbf{1}_{x\in [-\eta,\eta]}+x\textbf{1}_{x\in (\eta,\infty)}$, for some small positive $\eta$); and $\sigma_2(x):\mathbb{R}_+\rightarrow\mathbb{R}_+$ and $\sigma_2(x)/x:\mathbb{R}_+\rightarrow\mathbb{R}_+$ are both bounded, Lipschitz, and differentiable with Lipschitz continuous derivatives, then this function $f$ is partially-1 homogeneous and satisfies Assumption~\ref{assum:f} and~\ref{assum:f2}.

\medskip

We now build the metric on the function space for our solutions.
Note that the solution $\rho(\theta,t,s)$ of the problem defined below in~\eqref{eqn:Wassgradientflow} is expected to be a continuous function in $(t,s)$, and a distribution of $\theta$ for each $(t,s)$. For this non-standard probability space, we introduce the following metric.
\begin{definition}\label{def:path}\footnote{$\mathcal{C}([0,T];\mathcal{A})$ is defined to be the set of functions $f(a,t)$ such that for any $t\in[0,T]$, $f(\cdot,t)\in\mathcal{A}$ and $f(\cdot,t)$ is continuous with respect to $t$ under the metric defined on $\mathcal{A}$.
In Definition~\ref{def:path}, we set $T=1$ and $\mathcal{A}=\mathcal{P}^2$, where the natural metric on $\mathcal{A}$ is $W_2$.}
$\mathcal{C}([0,1];\mathcal{P}^2)$ is a collection of continuous paths of probability distribution $\nu(\theta,t)$ $(\theta\in\mathbb{R}^k,t\in[0,1])$ where (a) $\nu(\cdot,t)\in \mathcal{P}^2(\mathbb{R}^k)$ for every fixed $t\in[0,1]$; (b) For any $t_0\in[0,1]$, $\lim_{t\rightarrow t_0}W_2\left(\nu(\cdot,t),\nu(\cdot,t_0)\right)=0$, where $W_2$ is the classical Wasserstein-2 distance.
The space $\mathcal{C}([0,1];\mathcal{P}^2)$ is equipped with the metric $d_{1}\left(\nu_1,\nu_2\right)=\sup_tW_2(\nu_1(\cdot,t),\nu_2(\cdot,t))$.

Accordingly, $\mathcal{C}([0,\infty);\mathcal{C}([0,1];\mathcal{P}^2))$ is a collection of continuous paths of probability distribution $\nu(\theta,t,s)$ (with $\theta\in\mathbb{R}^k$, $t\in[0,1]$, $s\in[0,\infty)$), where (a) $\nu(\cdot,\cdot,s)\in \mathcal{C}([0,1];\mathcal{P}^2)$ for every fixed $s\in[0,\infty)$. (b) For any $s_0\in[0,\infty)$,  $\lim_{s\rightarrow s_0}d_1\left(\nu(\cdot,\cdot,s),\nu(\cdot,\cdot,s_0)\right)=0$ (where $d_1$ is defined above).
The metric in $\mathcal{C}([0,\infty);\mathcal{C}([0,1];\mathcal{P}^2))$ is defined by {$d_{2}\left(\nu_1,\nu_2\right)=\sum\limits_{n\geq 1} 2^{-n} \Big(1\wedge\sup\limits_{0\leq s\leq n}d_1(\nu_1(\cdot,\cdot,s),\nu_2(\cdot,\cdot,s))\Big)$.}
\end{definition}

Since $\mathcal{P}^2$ is complete in $W_2$ distance, $\mathcal{C}([0,1];\mathcal{P}^2)$ and $\mathcal{C}([0,\infty);\mathcal{C}([0,1];\mathcal{P}^2))$ are complete metric spaces under $d_1$ and $d_2$, respectively.
To give a rigorous justification of the mean-field limit, we use the particle representation of $\rho(\theta,t,s)$. Thus, we need at least to assume that we can find a stochastic process $\{\theta(t)\}$ that is drawn from the initial condition $\rho_{\ini}(\theta,t)$.  We call such initial conditions {\em admissible}.
\begin{definition}\label{def:meanadmissible}
We call a continuous path of probability distribution $\nu(\theta,t)\in \mathcal{C}([0,1];\mathcal{P}^2)$ {\em admissible} if it has a particle representation, namely there exists a continuous stochastic process $\theta(t):[0,1]\rightarrow\mathbb{R}^k$
such that for any $t_0\in[0,1]$, we have $\theta(t_0)\sim \nu(\theta,t_0)$ and
\begin{equation}\label{eqn:existenceofpath}
\lim_{t\rightarrow t_0}\mathbb{E}\left(|\theta(t)-\theta(t_0)|^2\right)=0\,.
\end{equation}
Furthermore, $\nu(\theta,t)$ is called {\em limit-admissible} if its $M$-averaged trajectory is bounded and Lipschitz with high probability.
(See Definition \ref{def:meanadmissible2} for a rigorous definition.)
\end{definition}
We note that without the dependence on $t$, probability distributions are always ``admissible,'' in the sense that one can draw a sample from a given distribution.

\section{From Gradient Descent to Gradient Flow}\label{sec_:translation}

This section connects NN training with OIE constrained minimization.
One key aspect of this connection is to translate gradient descent for NN training to gradient flow in the continuous limit.
A second key element is the mean-field limit.
We view the ResNet architecture as a discretization of the OIE formulation.
The weights assigned to the neurons are viewed as i.i.d. samples from an underlying distribution.
NN training then becomes a process of evolving this distribution  to achieve descent in the mismatch functional.




Using previous notation, we denote by $M$ the width of the ResNet and  by $L$ its depth.
Inference on the ResNet, for an input vector $x \in \mathbb{R}^d$ takes the following form:
\begin{equation}\label{eqn:disRes}
z_{l+1}(x) =z_{l}(x)+\frac{1}{ML}\sum^M_{m=1}f(z_{l}(x),\theta_{l,m})\,, \quad l=0,1,\dots,L-1
\end{equation}
with $z_{0}(x) =x$.
$z_L(x)$ is the output from the last layer. We collect all parameters in NN into a vector ${\Theta}_{L,M}=\left\{\theta_{l,m}\right\}^{L-1,M}_{l=0,m=1}$, where $\theta_{l,m} \in \mathbb{R}^k$  for all $l,m$ and $f:\mathbb{R}^d\times\mathbb{R}^k\rightarrow\mathbb{R}^d$ is the activation function.~\rev{As stated in Remark~\ref{re:example}}, we recover ``conventional'' ResNets by setting $f(z,\theta)=U\sigma(W^\top z+b)$ with $\sigma$ being the ReLU activation function, with $\theta=(W,U,b)\in\mathbb{R}^{2d+1}$ and $k=2d+1$.

To stress the dependence on the parameters, we denote by $Z_{{\Theta}_{L,M}}(1;x)=z_L(x)$ the output of the ResNet defined by \eqref{eqn:disRes}. ResNet training seeks the parameter vector $\Theta_{L,M}$ that minimizes the following mismatch or {\em loss} function:
\begin{equation}\label{eqn:lossfunction}
\mathcal{E}({\Theta}_{L,M})=\mathbb{E}_{x\sim\zeta}\left[\frac{1}{2}\left(g(Z_{{\Theta}_{L,M}}(1;x))-y(x)\right)^2\right]\,,
\end{equation}
where $g(x):\mathbb{R}^d\rightarrow\mathbb{R}$ is a given measuring function, $y(x)\in\mathbb{R}$ is the label corresponding to $x$, and $\zeta$ is the probability distribution from which the data $x$ is drawn. Suppose the data given is a finite set $\{x_n\}_{n=1}^N$ i.i.d. drawn from $\zeta$, then we obtain an empirical loss function by using $\zeta_N=\frac{1}{N}\sum^N_{n=1}\delta_{x_n}$ as an approximation to $\zeta$.
In the derivation below, we use $\zeta$ as general notation for the training set, instead of differentiating between $\zeta$ and ${\zeta_N}$.
When the context is clear, the index $\zeta$ is also dropped.

The gradient descent algorithm for NN training generates estimates $\Theta_{L,M}^n$, $n=0,1,2,\dotsc$ of the optimal parameter via steps of the form
\begin{equation*}
\Theta_{L,M}^{n+1}=\Theta_{L,M}^{n}-h\nabla_{{\Theta}} \mathcal{E}(\Theta^n_{L,M})\,,
\end{equation*}
where $h$ is the step length.
By taking limits as $h\to0$, the updating process can be characterized by the following ODE~\citep[Def~2.2]{NEURIPS2018_a1afc58c} (rescaled by $L,M$):
\begin{equation}\label{eqn:gradient_theta_LM}
    \frac{\rd{\Theta}_{L,M}(s)}{\rd s}=-ML\nabla_{{\Theta}} \mathcal{E}\big({\Theta}_{L,M}(s)\big)\,,\quad\text{for}\; s\geq 0\,.
\end{equation}
Here $s$ represents pseudo-time, the continuous analog of the discrete stepping process.

We now send both $M$ and $L$ to $\infty$.
Viewing $z_l(x) = z(t_{l}={l}/{L};x)$, and recasting $\theta_{l,m}$ as $\theta_m(t=l/L)$, the formula \eqref{eqn:disRes} can be viewed as the forward Euler scheme for the following ODE:
\begin{equation}\label{eqn:contRes}
\frac{\rd z(t;x)}{\rd t}=\frac{1}{M}\sum^M_{m=1}f(z(t;x),\theta_m(t))\,,\quad t\in[0,1]\,,
\end{equation}
with $z(0;x)=x$. Assuming then that $\theta_m(t)$ is drawn from an underlying probability measure $\rho(\theta,t)\in \mathcal{C}([0,1];\mathcal{P}^2)$\footnote{This notation is for a collection of probability distribution that is continuous in $t \in [0,1]$ and has bounded second moment in $\theta$ for all $t$. See Definition~\ref{def:path} for rigorous specifications.}, the ODE system for $z$ can be further translated to the following OIE
\begin{equation}\label{eqn:meancontRes}
\frac{\rd z(t;x)}{\rd t}=\int_{\mathbb{R}^k}f(z(t;x),\theta)\rho(\rd\theta,t)\,,\quad t\in[0,1]\,,
\end{equation}
with $z(0;x)=x$. The transition from~\eqref{eqn:disRes} to~\eqref{eqn:contRes} is known as the ``continuous limit" while the transition from~\eqref{eqn:contRes} to~\eqref{eqn:meancontRes} is the ``mean-field limit."
Intuitively, we could expect these approximations to be correct with expected errors of the order of $\frac{1}{L}$ and $\frac{1}{\sqrt{M}}$, respectively.
(The former comes from the error in forward Euler discretization; and the latter is from the standard law of large numbers.)

The relationship between~\eqref{eqn:disRes} and~\eqref{eqn:meancontRes} allows us to recast the training of $\Theta_{L,M}$ as the problem of determining the optimal $\rho$.
In particular, denoting by $Z_{\rho}(1;x)$ the solution to~\eqref{eqn:meancontRes} at time $t=1$ with parameter $\rho$, with loss function
\begin{equation}\label{eqn:lossfunctioncont2}
E(\rho)=\mathbb{E}_{x\sim\zeta}\left[\frac{1}{2}\left(g(Z_{\rho}(1;x))-y(x)\right)^2\right]\,,
\end{equation}
then~\eqref{eqn:gradient_theta_LM}, in the limit as $M\to\infty$, can be translated to the following gradient flow for $\rho$:
\begin{equation}\label{eqn:Wassgradientflow}
\frac{\partial \rho}{\partial s}=\nabla_\theta\cdot\left(\rho\nabla_\theta\left.\frac{\delta E}{\delta \rho}\right|_{\rho(\cdot,\cdot,s)}\right)\,,\quad s>0,\; t\in[0,1]\,,
\end{equation}
with $\rho(\theta,t,0)=\rho_{\ini}(\theta,t)$.
In this equation, $s$ is the training time and $\rho(\theta,t,s)$ is a probability measure on $\theta$ for all $s$ and $t$~\citep{NEURIPS2018_a1afc58c,pmlr-v119-lu20b,ding2021overparameterization}.
$\left.\frac{\delta E}{\delta \rho}\right|_{\rho}(\theta,t)$ is the functional derivative with respect to $\rho$ and is thus a function of $(\theta,t)$ for every fixed $s$.
This functional derivative can be made explicit using the classical calculus-of-variations method:
\begin{equation}\label{eqn:Frechetderivative}
\left.\frac{\delta E}{\delta \rho}\right|_{\rho}(\theta,t)=\mathbb{E}_{x\sim\zeta}\big(p^\top_{\rho}(t;x)f(Z_{\rho}(t;x),\theta)\big)\,,
\end{equation}
where $p_\rho(\cdot;x)$, parameterized by $x$, maps $[0,1]\to\mathbb{R}^d$, and is a vector solution to the following ODE:
\begin{equation}\label{eqn:prho}
\frac{\rd p^\top_\rho}{\rd t}=-p^\top_\rho\int_{\mathbb{R}^k}\partial_zf(Z_\rho,\theta)\rho(\rd\theta,t)\,.
\end{equation}
with $p_\rho(t=1;x)=\left(g(Z_{\rho}(1;x))-y(x)\right)\nabla g(Z_{\rho}(1;x))$.
Here the $d\times d$ matrix $\partial_z f$ is the Jacobian of $f$ with respect to its $z$ argument.
In later sections, to emphasize the dependence on $s$, we use $\frac{\delta E({\rho}(s))}{\delta \rho}$ to denote $\left.\frac{\delta E}{\delta \rho}\right|_{{\rho}(\cdot,\cdot,s)}$.

To summarize: To update ${\rho}(\theta,t,s)$ to ${\rho}(\theta,t,s+\delta s)$ with an infinitesimal $\delta s$, we solve~\eqref{eqn:meancontRes} for $Z_{\rho}(t;x)$ using the given ${\rho}(\theta,t,s)$, then compute $p_{\rho}$ from~\eqref{eqn:prho}.
We can then compute $\frac{\delta E({\rho}(s))}{\delta \rho}(\theta,t)$ which, in turn,  yields ${\rho}(\theta,t,s+\delta s)$ from~\eqref{eqn:Wassgradientflow}.

\begin{remark}
We can also perform GD on the continuous level before carrying the mean-field limit.
Denoting by $Z_{\Theta}(t;x)$ the solution of  \eqref{eqn:contRes} for a given collection of the $M$ functions $\Theta(t)=\{\theta_m(t)\}_{m=1}^M$, we define the loss function
\begin{equation}\label{eqn:lossfunctioncont}
\widehat{E}(\Theta)=\mathbb{E}_{x\sim\zeta}\left[\frac{1}{2}\left(g(Z_{\Theta}(1;x))-y(x)\right)^2\right]\,.
\end{equation}
To train for the configuration of $\Theta(t)$ that minimizes~\eqref{eqn:lossfunctioncont}, we let $\Theta(t)$ flow in the descending direction of $\widehat{E}(\Theta)$.
Denoting by $s$ the pseudo-time of the training process and by $\Theta(s;t)$ the collection of functions at the training time $s$, the equation becomes
\begin{equation}\label{eqn:gradient_theta_M}
\frac{\partial \Theta(s;t)}{\partial s}=-M\frac{\delta \widehat{E}(\Theta(s))}{\delta\Theta}\,,\quad s>0,\quad t\in[0,1],
\end{equation}
where $\frac{\delta \widehat{E}}{\delta\Theta}$ is the functional derivative of $\widehat{E}$ with respect to $\Theta$, and thus a list of $M$ functions of $t$ for every fixed $s$.
\end{remark}

The translation from the true gradient descent performed in practice~\eqref{eqn:gradient_theta_LM} to the gradient-flow PDE~\eqref{eqn:Wassgradientflow} incurs a discretization cost.
The following result shows that this cost diminishes as the size of the ResNet increases.


\begin{theorem}\label{thm:consistent}
Suppose that Assumptions~\ref{assum:f} and~\ref{ass:g} are satisfied. Assume that $\rho_{\ini}(\theta,t)$ is limit-admissible (see Definition~\ref{def:meanadmissible})
and that $\mathrm{supp}_\theta(\rho_{\ini}(\theta,t))\subset \{\theta||\theta_{[1]}|\leq R'\}$ for some $R'>0$ and all $t\in[0,1]$.
Let each initial parameter path $\{\theta_m(0;t)\}^M_{m=1}$ in \eqref{eqn:gradient_theta_M} be $i.i.d.$ drawn from $\rho_{\ini}(\theta,t)$. Let
\begin{itemize}
\item $\Theta_{L,M}(s)=\{\theta_{l,m}(s)\}$ be the solution to~\eqref{eqn:gradient_theta_LM} with initial condition $\theta_{l,m}(s=0)$ being $i.i.d.$ drawn from $\rho_{\ini}(\theta,{l}/{L})$;
\item $\rho(\theta,t,s)$ be the solution to \eqref{eqn:Wassgradientflow} with initial condition $\rho_{\ini}(\theta,t)$.
\end{itemize}
Then for any positive $\epsilon$, $\eta$, and $S$, there exists a constant $C>0$ that depends on $\rho_{\ini}$, $R$ (the radius of the support of data distribution $\zeta$, see Assumptions~\ref{ass:g}), and $S$ such that, given any $s<S$, when
\[
M>\frac{C(\rho_{\ini},R,S)}{\epsilon^2\eta}\,,\quad L>\frac{C(\rho_{\ini},R,S)}{\epsilon}\,,
\]
we have that
\begin{subequations} \label{eq:closeof}
\begin{align}
\mathbb{P}\left(\left|Z_{\Theta_{L,M}(s)}(1;x)-Z_{\rho(s)}(1;x)\right|  \leq\epsilon\right) & \geq 1-\eta, \\
\nonumber
\mathbb{P}\left(\left|\mathcal{E}\big(\Theta_{L,M}(s)\big)-E\big(\rho(\cdot,\cdot,s)\big)\right| \leq \epsilon\right) & \geq 1-\eta\,.
\end{align}
\end{subequations}
Here $Z_{\Theta_{L,M}(s)}(1;x) = z_L(x)$ is defined in~\eqref{eqn:disRes}, and $Z_{\rho(s)}(1;x)$ is defined in~\eqref{eqn:meancontRes} with $\rho$ replaced by $\rho(\cdot,\cdot,s)$.
$\mathcal{E}\big(\Theta_{L,M}(s)\big)$ is defined in~\eqref{eqn:lossfunction} with $\Theta_{L,M}$ replaced by $\Theta_{L,M}(s)$.
$E\left(\rho(\cdot,\cdot,s)\right)$ is defined in~\eqref{eqn:lossfunctioncont2} with $\rho$ replaced by $\rho(\cdot,\cdot,s)$.
\end{theorem}
{
\begin{remark}\label{rmk:different_data}
More generally, let $\bar{\zeta}$ be any probability measure supported in $\mathcal{B}_R(0)$. Keeping the dynamics driven by the training distribution $\zeta$, we may use $\bar{\zeta}$ solely to evaluate the finite-network and mean-field losses:
\[
\mathcal{E}_{\bar{\zeta}}\bigl(\Theta_{L,M}(s)\bigr)
=
\mathbb{E}_{x\sim\bar{\zeta}}
\left[
\frac{1}{2}
\left(g\bigl(Z_{\Theta_{L,M}(s)}(1;x)\bigr)-y(x)\right)^2
\right]
\]
and
\[
E_{\bar{\zeta}}\bigl(\rho(\cdot,\cdot,s)\bigr)
=
\mathbb{E}_{x\sim\bar{\zeta}}
\left[
\frac{1}{2}
\left(g\bigl(Z_{\rho(s)}(1;x)\bigr)-y(x)\right)^2
\right].
\]
Then the conclusion of Theorem~\ref{thm:consistent} continues to hold: under the same requirements on $L$ and $M$,
\[
\mathbb{P}\left(
\left|
\mathcal{E}_{\bar{\zeta}}\bigl(\Theta_{L,M}(s)\bigr)
-
E_{\bar{\zeta}}\bigl(\rho(\cdot,\cdot,s)\bigr)
\right|
\leq \epsilon
\right)
\geq 1-\eta
\]
for every fixed $s\in[0,S]$. The proof is similar to that of Theorem~\ref{thm:consistent}; see Appendix~\ref{sec:proofofthmconsistent}.
\end{remark}
}

The theorem suggests that for every fixed $S>0$, gradient descent for $\Theta_{L,M}$ tracks
the gradient flow of $\rho(\theta,t)$, in the sense that the two losses are close to each other with high probability, when $L$ and $M$ are sufficiently large. The size of the ResNet depends negative-algebraically on $\epsilon$ (the desired accuracy) and $\eta$ (the failure probability). The result translates the evolution (gradient descent) of $\Theta_{L,M}$ to the evolution (gradient flow) of $\rho(\theta,t)$, and thus matches the zero-loss property of the parameter configuration of a finite-sized ResNet to its limiting PDE.

Although the results are not surprising, the proof is technically demanding.
The ResNet formula \eqref{eqn:disRes} can be viewed as a forward-Euler solver for~\eqref{eqn:contRes}, so we expect a first-order error of size $O(1/L)$.
Viewing~\eqref{eqn:contRes} as a Monte Carlo solver for~\eqref{eqn:meancontRes}, we expect an $O(1/\sqrt{M})$ error.
Assuming the parameter $\theta_{l,m}$ is $i.i.d.$ generated from $\rho$, to bound the error between the "continuous" and "discrete" versions of $Z$ by a given $\epsilon$, we expect to require $M\sim O(1/\epsilon^2)$ and $L\sim O(1/\epsilon)$.
Indeed, these are the estimates we obtain in \eqref{eq:closeof}.
The difficulty of the proof comes from the training process: We need to control the propagation of error so that $\rho(s)$ is close to the distribution of $\Theta_{L,M}(s)$ in the sense of expectation as $L,M\rightarrow\infty$.~\rev{In addition,  the lower bounds on $M$ and $L$ are independent, implying that the coupled mean-field error and discretization error must be controlled simultaneously during the evolution.
}
For completeness, we sketch the proof in Appendix~\ref{sec:proofofthmconsistent}.
(Similar proofs can be found in~\citep{Zhiyan_thesis}.)

\section{Convergence to the Global Minimizer}\label{sec_:convergetoglobalminimizer}

We discussed in the previous section the translation of ResNet training by gradient descent into a  gradient-flow PDE.
In this section, we show that the PDE ultimately drives the underlying distribution to a configuration for which $E(\rho(s=\infty))=0$.
These results are presented in Theorem~\ref{thm:globalminimal1} and Theorem~\ref{thm:globalminimal2}.

\subsection{Stationary point and the global minimum of the gradient flow}\label{sec_:steadystate}
We begin by recalling a crucial proposition from \cite{{ding2021overparameterization}} that relates the stationary point and the global minimum.
\begin{proposition}\label{prop:localisglobal}[\cite[{Corollary of Proposition 26}]{ding2021overparameterization}]
{
Suppose that Assumptions~\ref{assum:f} and~\ref{ass:g} hold, and let
\(
\rho\in \mathcal{C}([0,1];\mathcal{P}^2).
\)
If there exist $t_0\in[0,1]$ and $c\in\mathbb{R}$ such that
\[
\frac{\delta E(\rho)}{\delta\rho}(\theta,t_0)=c,
\qquad \forall\theta\in\mathbb{R}^k,
\]
then
\(
E(\rho)=0.
\)
}
\end{proposition}
For completeness, we include the proof of this result in Appendix~\ref{sec:proofofproplocal}.


{Proposition~\ref{prop:localisglobal} provides a zero-loss certificate: if the first variation vanishes identically in the parameter variable at even one depth $t_0$ (that is, $\frac{\delta E(\rho)}{\delta \rho}(\theta,t_0)\equiv0$ for some time $t_0$), then $E(\rho)=0$.}
Our task now becomes to check under what conditions can we drive $\frac{\delta E(\rho_\infty)}{\delta \rho}$ to $0$.

Suppose that $E(\rho(s))$ stabilizes in time, that is, $\partial_sE(\rho)=0$ as $s\to\infty$.
Using the notation $\rho(s=\infty)=\rho_\infty$ according to~\eqref{eqn:Wassgradientflow}, we should have
\[
\nabla_\theta\frac{\delta E\big(\rho_\infty\big)}{\delta \rho}(\theta,t)=\vec{0},\quad \rho_\infty(\theta,t)\text{-}\mathrm{a.e.}\ \theta\in\mathbb{R}^k\,,\quad \mathrm{a.e.}\ t\in[0,1]\,.
\]
However, this condition is not equivalent to $\frac{\delta E(\rho_\infty)}{\delta \rho}(\theta,t_0)=0$, which is the condition that Proposition~\ref{prop:localisglobal} requires.
Some other assumptions are needed.
In~\cite{ding2021overparameterization}, the authors assumed an unrealistic ``full-support'' condition, requiring the final distribution to be supported over the entire domain.
Under the latter condition, we can show that $\rho_\infty$ is a stationary point.
In this article, we claim that for mildly smooth (2-homogeneous or partially 1-homogeneous) activation function, one can already obtain the zero-loss property without the full-support condition.
The homogeneity assumptions are discussed in Section~\ref{sec_:2homogeneous} and Section~\ref{sec_:partially1homogeneous}, respectively.




\subsection{$2$-Homogeneous assumption}\label{sec_:2homogeneous}


\begin{theorem}\label{thm:globalminimal1}
Let Assumptions~\ref{assum:f},~\ref{assum:2-homogeneous}, and \ref{ass:g} hold with $k_1=k$.
Assume that $\rho_{\ini}(\theta,t)$ is limit-admissible and that $\mathrm{supp}_\theta(\rho_{\ini}(\theta,t))\subset \{\theta\mid|\theta|\leq R'\}$ for some $R'>0$ and all $t\in[0,1]$.
Let $\rho_\infty(\theta,t)$ be the long-time limit of~\eqref{eqn:Wassgradientflow}, that is,  $\rho(\theta,t,s)$ converges to $\rho_\infty(\theta,t)$ in $\mathcal{C}([0,1];\mathcal{P}^2)$ as $s\rightarrow\infty$.
Then
\[
E(\rho_{\infty})=0
\]
if both of the following conditions hold.
\begin{itemize}
\item Separation initialization: There exists $t_0\in[0,1]$ such that $\mathrm{supp}_\theta(\rho_{\ini}(\theta,t_0))$ separates\footnote{We say that a set $C$ separates the sets $A$ and $B$ if any continuous path with endpoints in $A$ and $B$ intersects $C$.} the spheres $r_a\mathbb{S}^{k-1}$ and $r_b\mathbb{S}^{k-1}$, for some $0<r_a<r_b$.
\item {Sard-type condition: At the depth $t_0$ above, set
\(
q_\infty=\frac{\delta E(\rho_\infty)}{\delta\rho}(\theta,t_0)|_{\mathbb S^{k-1}}.
\)
Either $q_\infty$ is constant, or its attained regular values are dense
in its range $q_\infty(\mathbb S^{k-1})$\footnote{Precisely, define
\[
\operatorname{Reg}(q_\infty)
=\left\{c\in q_\infty(\mathbb S^{k-1}):
Dq_\infty(\omega):T_\omega\mathbb S^{k-1}\to\mathbb R
\text{ is nonzero for every }\omega\in q_\infty^{-1}(c)
\right\}.
\]
In the nonconstant case, we assume
$\overline{\operatorname{Reg}(q_\infty)}
=q_\infty(\mathbb S^{k-1})$, where the closure is taken in $\mathbb R$.}.}
\end{itemize}
\end{theorem}
We give the proof of this result below, following a remark and a proposition.

{\begin{remark}
This Sard-type condition is adopted from~\citet{NEURIPS2018_a1afc58c}.
This property essentially states that most of the points in the range of $q_\infty$ lie in an open set.
At every point in the preimage of a regular value, the gradient of $q_\infty$ is nonzero; consequently, the corresponding level set is locally a smooth hypersurface.
The Sard-type condition is a mild nondegeneracy requirement on the level sets of $q_\infty$ and is satisfied by most commonly seen regular functions, except functions that oscillate wildly.
For example, by Sard's theorem, this condition holds whenever $q_\infty\in\mathcal C^{k-1}$ for $k\geq2$.
\end{remark}
}

The proof of Theorem \ref{thm:globalminimal1} (given below) uses a contradiction argument. Since $\rho(\theta,t,s)$ converges to $\rho_\infty(\theta,t)$ in $\mathcal{C}([0,1];\mathcal{P}^2)$, the second moment of $\theta$ of $\rho(\theta,t,s)$ should be uniformly bounded in $t\in[0,1],s\in[0,\infty)$. By contrast, we show that if $E(\rho_{\infty})>0$, this second moment blows up to infinity as $s\rightarrow\infty$, contradicting  boundedness.
Specifically, we have the following proposition.
\begin{proposition}\label{prop:2-homogeneous} Suppose the conditions of Theorem~\ref{thm:globalminimal1} hold.
{Let $t_0$ be the depth in the separation-initialization condition of Theorem~\ref{thm:globalminimal1}} such that
\[
\lim_{s\rightarrow\infty}\int_{\mathbb{R}^k}|\theta|^2\rd\rho(\theta,t_0,s)<\infty\,.
\]
Then
\begin{equation}\label{eqn:notsatisfy}
\frac{\delta E\big(\rho_\infty\big)}{\delta \rho}(\theta,t_0)= 0\,, {\qquad \forall \theta \in \mathbb{R}^k} \,.
\end{equation}
\end{proposition}
The proof appears in Appendix~\ref{sec:proofofP2-h} but we give some intuition for it here.
According to the gradient-flow structure, each i.i.d sample drawn from $\rho_s$ moves along the direction of $-\nabla\frac{\delta E(\rho_s)}{\delta\rho}$. Since $f$ is $2$-homogeneous, $\frac{\delta E(\rho_s)}{\delta\rho}$ is also $2$-homogeneous, so that
\begin{equation}\label{eqn:theta_square_intuition}
\partial_s|\theta|^2=-2\left\langle\theta,\nabla_\theta\frac{\delta E(\rho_s)}{\delta\rho}(\theta)\right\rangle=-4\frac{\delta E(\rho_s)}{\delta\rho}(\theta)\,.
\end{equation}
If $E\big(\rho_\infty\big)\neq0$, then by Proposition~\ref{prop:localisglobal}, $\frac{\delta E(\rho_\infty)}{\delta\rho}\neq 0$.
One can further show that there is at least one point $\theta$ in the parameter space for which $\frac{\delta E(\rho_\infty)}{\delta\rho}(\theta)<0$.
Recalling that $\frac{\delta E(\rho_s)}{\delta\rho}$ is $2$-homogeneous, $\frac{\delta E(\rho_s)}{\delta\rho}(\theta)\sim-|\theta|^2$ in the cone expanded around this value of $\theta$.
By substituting into~\eqref{eqn:theta_square_intuition}, we see that $\theta$ will propagate to infinity, driving the second moment of $\rho_s$ to infinity also.

Theorem~\ref{thm:globalminimal1} is a direct corollary of Proposition \ref{prop:2-homogeneous}, in combination with  Proposition~\ref{prop:localisglobal}. Its proof is as follows.

\begin{proof}[Theorem~\ref{thm:globalminimal1}]
Since $\rho(\theta,t,s)$ converges to $\rho_\infty(\theta,t)$ in $\mathcal{C}([0,1];\mathcal{P}^2)$, we have for any $t_0$ that
\[
\sup_{s\geq 0}\int_{\mathbb{R}^k}|\theta|^2\rho(\rd\theta,t_0,s)<\infty\,.
\]
According to Proposition \ref{prop:2-homogeneous}, $\frac{\delta E(\rho_\infty)}{\delta \rho}(\theta,t_0)=0$ for any $\theta\in\mathbb{R}^k$, which implies that $E\big(\rho_\infty\big)=0$, from Proposition \ref{prop:localisglobal}.
\end{proof}

\subsection{$1$-Homogeneous assumption}\label{sec_:partially1homogeneous}

{
\begin{theorem}\label{thm:globalminimal2}
Let Assumptions~\ref{assum:f},~\ref{assum:f2}, and~\ref{ass:g} hold
with $k_1=1$, and write
$\theta=(\theta_{[1]},\theta_{[2]})\in\mathbb R\times\mathbb R^{k-1}$.
Assume that $\rho_{\ini}(\theta,t)$ is limit-admissible and that
$\mathrm{supp}_\theta(\rho_{\ini}(\theta,t))
\subset\{(\theta_{[1]},\theta_{[2]})\mid |\theta_{[1]}|\leq R'\}$ for some $R'>0$ and all
$t\in[0,1]$.
Let $\rho_\infty(\theta,t)$ be the long-time limit
of~\eqref{eqn:Wassgradientflow}, that is,
$\rho(\theta,t,s)$ converges to $\rho_\infty(\theta,t)$ in
$\mathcal C([0,1];\mathcal P^2)$ as $s\to\infty$.
Then
\[
E(\rho_\infty)=0
\]
if both of the following conditions hold.
\begin{itemize}
\item Separation initialization: There exist $t_0\in[0,1]$ and
$r_0>0$ such that
$\mathrm{supp}_\theta(\rho_{\ini}(\theta,t_0))$ separates
$\{-r_0\}\times\mathbb R^{k-1}$ and
$\{r_0\}\times\mathbb R^{k-1}$.
\item Escape-regularity conditions: For the depth $t_0$ above and
$\rho\in\mathcal C([0,1];\mathcal P^2)$, set
\[
 h_{\rho}(\theta_{[2]})
 =\left.\frac{\delta E}{\delta\rho}\right|_{\rho}
 ((1,\theta_{[2]}),t_0).
\]
If $k\geq2$ and $h_{\rho_\infty}$
is nonconstant, the following conditions
hold.
\begin{enumerate}[label=\textup{(ER\arabic*)}]
\item \label{itm:ER1} There are $H_\infty\in
\mathcal C^1(\mathbb S^{k-2})$ and
$G_\infty\in
\mathcal C(\mathbb S^{k-2};\mathbb R^{k-1})$ such that, uniformly for
$\varphi\in\mathbb S^{k-2}$,
\begin{equation}\label{eqn:ppp-radial-limits}
 h_{\rho_\infty}(r\varphi)\longrightarrow H_\infty(\varphi),
 \qquad
 r\nabla h_{\rho_\infty}(r\varphi)\longrightarrow
 G_\infty(\varphi)
 \qquad
 \text{as}
 \qquad
 r\to\infty.
\end{equation}
Here $\nabla$ denotes the Euclidean gradient in $\mathbb R^{k-1}$.
\item \label{itm:ER3} The function $h_{\rho_\infty}$ has an attained nonzero regular value $c$
for which either
\begin{enumerate}[label=\textup{(ER2\alph*)},leftmargin=*]
\item \label{itm:ER3a} $h_{\rho_\infty}^{-1}(c)$ is bounded; or
\item \label{itm:ER3b} $k\ge3$, $h_{\rho_\infty}^{-1}(c)$ is unbounded, and $c$ is also a regular value of $H_\infty$.
\end{enumerate}
\item \label{itm:ER2} Under \ref{itm:ER3b}, along the gradient-flow trajectory $\rho(s)$, the gradient of the first variation converges in the following radially weighted sense:
\begin{equation}\label{eqn:ppp-weighted-stability}
 \lim_{s\to\infty}
 \sup_{\theta_{[2]}\in\mathbb R^{k-1}}(1+|\theta_{[2]}|)
 |\nabla h_{\rho(s)}(\theta_{[2]})-\nabla h_{\rho_\infty}(\theta_{[2]})|
 =0.
\end{equation}
\end{enumerate}
\end{itemize}
\end{theorem}
}

{


The proof appears below following a remark and a proposition.

\begin{remark}\label{rmk:ppp-repair}
The escape-regularity conditions are adapted from~\citet{petit2026global}.
Condition~\ref{itm:ER1} says that $h_\infty$ is asymptotically determined by direction and that its full gradient has the natural order $|w|^{-1}$ at infinity.
Condition~\ref{itm:ER2} is used in the present proof only when the selected regular level is unbounded in~\ref{itm:ER3b}.
Under~\ref{itm:ER1}, condition~\ref{itm:ER3} follows if the critical values of $h_{\rho_\infty}$ and, when $k\geq3$, of $H_\infty$ have Lebesgue measure zero.
By Sard's theorem, it is enough to assume
$h_{\rho_\infty}\in\mathcal C^{k-1}$ and, when $k\geq3$,
$H_\infty\in\mathcal C^{k-2}$. See~\citet[Proposition~1 and Remark~4]{petit2026global} for further discussion of these sufficient conditions.
\end{remark}
}


Similar to the case of 2-homogeneous, the proof of Theorem~\ref{thm:globalminimal2} relies on a contradiction argument.
\begin{proposition}\label{prop:1-homogeneous}
Suppose that the conditions of Theorem~\ref{thm:globalminimal2} are satisfied.
{Let $t_0$ be the depth in the separation-initialization condition of Theorem~\ref{thm:globalminimal2}} such that
\[
\lim_{s\rightarrow\infty}\int_{\mathbb{R}^k}|\theta|^2\rho(\rd\theta,t_0,s)<\infty\,,
\]
we have
\begin{equation}\label{eqn:notsatisfy1}
\frac{\delta E\big(\rho_\infty\big)}{\delta \rho}(\theta,t_0)=0\,.
\end{equation}
\end{proposition}
This proposition is similar to Proposition~\ref{prop:2-homogeneous},  but differs in the assumptions on the activation functions.
The $2$-homogeneous and partially $1$-homogeneous cases require different technical components in the proof; see Appendix~\ref{sec:proofofP1-h}.
Theorem~\ref{thm:globalminimal2} is a direct corollary of Proposition~\ref{prop:1-homogeneous}.
\begin{proof}[Theorem~\ref{thm:globalminimal2}]
Since $\rho(\theta,t,s)$ converges to $\rho_\infty(\theta,t)$ in $\mathcal{C}([0,1];\mathcal{P}^2)$, we have for any $t_0$ that
\[
\sup_{s\geq 0}\int_{\mathbb{R}^k}|\theta|^2\rho(\rd\theta,t_0,s)<\infty\,.
\]
According to Proposition \ref{prop:1-homogeneous}, $
\frac{\delta E\big(\rho_\infty\big)}{\delta \rho}(\theta,t_0)=0$ for all $\theta\in\mathbb{R}^k$, which implies that $E\big(\rho_\infty\big)=0$ according to Proposition~\ref{prop:localisglobal}, completing the proof.
\end{proof}

\subsection{Convergence of the ResNet training}\label{sec:ResNet_time}
With Theorem~\ref{thm:consistent} presenting the equivalence of ResNet training and the gradient flow, and Theorem~\ref{thm:globalminimal1} (or~\ref{thm:globalminimal2}) presenting the convergence-in-training-time of the gradient flow, we arrive naturally at the training convergence of ResNet.
\begin{theorem}\label{thm:maintheorem_zero_formal}
Let the conditions in Theorem~\ref{thm:globalminimal1} (or~\ref{thm:globalminimal2}) hold. Then for any positive $\epsilon$ and $\eta$, there exist positive constants $C_0$ depending on $\rho_{\ini}(\theta,t)$ and $\epsilon$, and $C$ depending on $\rho_{\ini}(\theta,t)$ and $s$ such that
 when
\[
s>C_0(\rho_{\ini},\epsilon)\,,\quad M>\frac{C(\rho_{\ini},s)}{\epsilon^2\eta}\,,\quad L>\frac{C(\rho_{\ini},s)}{\epsilon}\,,
\]
we have
\[
\mathbb{P}\left(\left|\mathcal{E}(\Theta_{L,M}(s))\right|\leq \epsilon\right)\geq 1-\eta\,.
\]
\end{theorem}

The result is the final conclusion of the mean-field analysis and the gradient-flow analysis conducted above.
It claims that for large enough ResNet (with $L>1/\epsilon$ and $M>1/\epsilon^2$), the long-time training equilibrium finds the zero-loss of the mismatch function.

The result resonates with those obtained in~\citep{NEURIPS2018_a1afc58c} for the 2-layer NN, and extends greatly those in~\citep{ding2021overparameterization}.
The change from 2-layer to multi-layer significantly complicates the analysis, since the current $\rho(\theta,t,s)$ is a probability measure for all $s$ and $t$, and they are coupled in $t$ for every $s$.
This coupling is not present in a 2-layer NN, and it makes the conventional gradient flow structure inapplicable to the current setting.
\citet{ding2021overparameterization} contains a discussion of global-in-time convergence, but with a much stronger assumption: the full-support assumption for $\rho(\theta,t,s=\infty)$ over the space of  $\theta$.
In the current paper, this support condition is replaced by a much more realistic homogeneity condition, adapted from~\citep{NEURIPS2018_a1afc58c}.

\section{Generalization Property}\label{sec_:generalizationproperty}

We now address the generalization property, which concerns the ability of the NN to capture the behavior of the full distribution $\zeta$ using a finite distribution $\zeta_N = \frac{1}{N}\sum_{n=1}^N\delta_{x_n}$  constructed from $N$ points i.i.d. sampled from $\zeta$.
We have the following results concerning closeness of the objective functions resulting from the two distributions $\zeta$ and $\zeta_N$.

\begin{theorem}\label{thm:generalization}
Suppose that Assumptions~\ref{assum:f} and~\ref{ass:g} hold.
Assume that $\rho_{\ini}(\theta,t)$ is limit-admissible and that $\mathrm{supp}_\theta(\rho_{\ini}(\theta,t))\subset \{\theta||\theta_{[1]}|\leq R'\}$ for some $R'>0$ and all $t\in[0,1]$.
Let $\rho_{\ini}(\theta,t)$ be the initial condition and suppose that
\begin{itemize}
\item $\rho(\theta,t,s)$ solves \eqref{eqn:Wassgradientflow} with $E=E_\zeta$ defined in~\eqref{eqn:lossfunctioncont2} according to $\zeta$,
\item $\rho^N(\theta,t,s)$ solves \eqref{eqn:Wassgradientflow} with $E=E_{\zeta_N}$ defined in~\eqref{eqn:lossfunctioncont2} according to $\zeta_N$.
\end{itemize}
Then for any positive $\epsilon$, $\eta$, and $S$, there exists a constant $C>0$ that depends on $\rho_{\ini}(\theta,t)$, $R$ (See Assumptions~\ref{ass:g}), and $S$ such that, given any $s<S$ and
\[
N>\frac{C(\rho_{\ini},R,S)}{\epsilon^2\eta}\,,
\]
we have:
\[
\mathbb{P}\left(\left|E_\zeta\left(\rho^N(\cdot,\cdot,s)\right)-E_\zeta\left(\rho(\cdot,\cdot,s)\right)\right|\leq   \epsilon\right)\geq 1-\eta\,.
\]
Here, $\mathbb{P}$ is the probability measure with respect to the i.i.d. samples $\{x_n\}_{n=1}^N$ drawn from $\zeta$.
\end{theorem}

This result is highly intuitive.
As the sample size $N$ increases, the training from the samples and the training from the underlying distribution converge, and the convergence rate $\frac{1}{\sqrt{N}}$ is typical of Monte Carlo.~\rev{We note that, although the result is intuitively correct, a rigorous proof of the stability is highly technical. To establish the above result, one must demonstrate uniform stability---in both $\theta$ and $t$---of the functional derivative $\frac{\delta E}{\delta \rho}$ with respect to $E$, and carefully control the perturbation error throughout the entire evolution.
} The proof appears in Appendix~\ref{sec:generalizationproof}.

Theorem~\ref{thm:generalization} gives an upper bound of the population risk. A natural corollary, when combined with Theorem~\ref{thm:globalminimal1} (or Theorem~\ref{thm:globalminimal2}), claims that a sufficiently large ResNet trained with finitely many data points via a GD method can almost achieve the global minimum for population risk.
\begin{theorem}\label{thm:maintheorem1}
Let the conditions in Theorem~\ref{thm:globalminimal1} (or~\ref{thm:globalminimal2}) hold.
Then for any positive $\epsilon$ and $\eta$, there exist positive constants $C_0$ depending on $\rho_{\ini}(\theta,t)$ and $\epsilon$, and $C$ depending on $\rho_{\ini}(\theta,t)$, $R$, and $s$,  such that
 when $s>C_0(\rho_{\ini},\epsilon)$ and
\[
M>\frac{C(\rho_{\ini},R,s)}{\epsilon^2\eta}\,,\quad L>\frac{C(\rho_{\ini},R,s)}{\epsilon}\,,\quad N>\frac{C(\rho_{\ini},R,s)}{\epsilon^2\eta},
\]
we have
\[
\mathbb{P}\left(\left|\mathcal{E}_\zeta(\Theta^{{N}}_{L,M}(s))\right|\leq \epsilon\right)\geq 1-\eta\,,
\]
where $\mathcal{E}_\zeta$ is defined in~\eqref{eqn:lossfunction} and $\Theta^{{N}}_{L,M}$ evolves along~\eqref{eqn:gradient_theta_LM} with the negative gradient of $\mathcal{E}_N=\mathcal{E}_{\zeta_N}$.
\end{theorem}

\begin{proof}[Theorem~\ref{thm:maintheorem1}]
We give the proof for the case in which the conditions of Theorem~\ref{thm:globalminimal1} are satisfied.
(The proof that uses the conditions of Theorem~\ref{thm:globalminimal2} is similar.)

Since
\begin{equation}\label{eqn:triangle}
\left|\mathcal{E}_\zeta(\Theta^{{N}}_{L,M})\right|\leq \left|\mathcal{E}_\zeta(\Theta^{{N}}_{L,M})-E_\zeta(\rho^N)\right|+\left|E_\zeta(\rho^N)-E_\zeta(\rho)\right|+\left|E_\zeta(\rho)\right|\,,
\end{equation}
we can show $\left|\mathcal{E}_\zeta(\Theta^{{N}}_{L,M})\right|<\epsilon$ by showing that the three terms on the right hand side are each bounded by $\epsilon/3$.

First, since $\mathrm{supp}(\zeta)\subset \mathcal{B}_{R}(\vec{0})$, we have from Theorem~\ref{thm:globalminimal1} that for any $\epsilon>0$, there exists a positive constant $C_0$ depending on $\rho_{\ini}$ and $\epsilon$ such that
\begin{equation}\label{eqn:mainclose1}
E_\zeta\left(\rho(\cdot,\cdot,s)\right)<\epsilon/3\,,\quad\forall s>C_0(\rho_{\ini}(\theta,t),\epsilon)\,.
\end{equation}
According to Theorem \ref{thm:generalization}, there exists a constant $C>0$ depending on $\rho_{\ini}$, $R$, and $s$ such that for
\[
N>\frac{C(\rho_{\ini},R,s)}{\epsilon^2\eta}\,,
\]
we have:
\begin{equation}\label{eqn:mainclose2}
\mathbb{P}\left(\left|E_\zeta\left(\rho^N(\cdot,\cdot,s)\right)-E_\zeta\left(\rho(\cdot,\cdot,s)\right)\right|\leq \epsilon/3\right)\geq 1-\eta/2\,.
\end{equation}


{It remains to control the first term on the right-hand side of \eqref{eqn:triangle}. By Remark~\ref{rmk:different_data}, the approximation estimate may be applied conditionally on the training sample, with a constant uniform over all empirical training distributions supported in $\mathcal{B}_R(0)$. Consequently, there exists a constant $C''=C''(\rho_{\ini},R,s)>0$ such that, for every realization satisfying $|x_i|<R$, $i=1,\ldots,N$, whenever
\[
M>\frac{C''}{\epsilon^2\eta}
\qquad\text{and}\qquad
L>\frac{C''}{\epsilon},
\]
we have
\begin{equation}\label{eqn:mainclose3}
\mathbb{P}\left(
 \left.
 \left|
 \mathcal{E}_\zeta\bigl(\Theta^{N}_{L,M}(s)\bigr)
 -
 E_\zeta\bigl(\rho^{N}(s)\bigr)
 \right|
 \leq \frac{\epsilon}{3}
 \,\right|\,x_1,\ldots,x_N
\right)
\geq 1-\frac{\eta}{2}.
\end{equation}
Since $|x_i|<R$ almost surely, averaging this conditional estimate over the training sample yields the corresponding unconditional bound. Combining it with \eqref{eqn:mainclose1} and \eqref{eqn:mainclose2} and applying the union bound completes the proof.
}


\end{proof}

\acks{Q.L. acknowledges support from Vilas Early Career award and ONR-N00014-21-1-214. The research of Z.D., S.C., and Q.L. is supported in part by NSF-CAREER-1750488 and Office of the Vice Chancellor for Research and Graduate Education at the University of Wisconsin Madison with funding from the Wisconsin Alumni Research Foundation. The work of Z.D., S.C., Q.L., and S.W. is supported in part by NSF via grant DMS-2023239. S.W. also acknowledges support from NSF Award 2224213 and  AFOSR FA9550-21-1-0084.}

{
\section*{Statement on AI Usage}
The earlier version of this work, including the continuous and mean-field limit analysis and the global-convergence framework for $2$-homogeneous and partially $1$-homogeneous activation functions, was developed by the authors without AI assistance. The generalization results and their proofs in Section~\ref{sec_:generalizationproperty} and Appendix~\ref{sec:generalizationproof}, although added after the first arXiv version, were completed by the authors before the present revision, without AI assistance. In preparing the present revision, the authors used GPT-5.6 Sol Ultra to assist with reviewing and correcting the assumptions and proofs of the global-convergence results, streamlining these proofs, and improving the exposition. In particular, the model helped identify a gap in the earlier proof of the partially $1$-homogeneous case and pointed the authors to the relevant argument in~\citet{petit2026global}, which informed the revised assumptions and proof. All AI-assisted suggestions were independently checked, verified, and edited by the authors, who take full responsibility for the contents of the paper, including any errors or omissions.
}

\newpage

\begin{appendix}
\section{Introduction to Appendix}
The appendix contains proofs and supporting analysis for the  theorems in the main text.
We start by proving well-posedness of the ResNet ODE and gradient flow.
Next, we justify the continuous and mean-field limit.
Finally, we give more details to complete the proof of the global convergence of the gradient-flow PDE and prove the generalization property.
The sections are organized as follows.
\begin{enumerate}[wide,labelindent=0pt]
    \item[Appendix \ref{proofofthmwellposed}:] Well-posedness and regularity. We first summarize the well-posedness and stability of the continuous limit ODE \eqref{eqn:contRes}, the mean-field limit OIE \eqref{eqn:meancontRes}, and the associated gradient flows ~\eqref{eqn:gradient_theta_LM},~\eqref{eqn:gradient_theta_M}, and \eqref{eqn:Wassgradientflow}. The results are collected in Theorems \ref{thm:wellposed1}, \ref{thm:wellposed}. Following the well-posedness result, we also summarize regularity estimates of the Fr\'echet derivatives~\eqref{eqn:Frechetderivative} and~\eqref{eqn:prho} in Lemma \ref{lemmadeltaprhoupdate}.
    \item[Appendix \ref{sec:proofofthmconsistent}:] A sketch of the proof of Theorem~\ref{thm:consistent}, the continuous and mean-field limit.
    \item[Appendix \ref{sec:proofofconvergencetoglobalminimal}:] This section presents detailed proofs of the technical propositions underlying the global convergence result, namely Propositions~\ref{prop:localisglobal}, \ref{prop:2-homogeneous}, and \ref{prop:1-homogeneous}. A complete proof of Proposition~\ref{prop:localisglobal} appears in Section~\ref{sec:proofofproplocal}. In Section~\ref{sec:blowupmoment}, we prove Propositions~\ref{prop:2-homogeneous}  and \ref{prop:1-homogeneous}.
    \item[Appendix \ref{sec:generalizationproof}:] This section studies the generalization property of the gradient-flow PDE and proves Theorem~\ref{thm:generalization}.
\end{enumerate}


Throughout the appendices, {\em we denote by $C(\cdot)$ a generic constant whose value depends only on its arguments}.
We omit the argument $R$ (see Assumption~\ref{ass:g}) in $C(\cdot)$, since this generic constant always depends on $R$.
The precise value of this constant may change each time it is invoked.
Throughout Appendices~\ref{proofofthmwellposed}-\ref{sec:generalizationproof}, we assume that Assumption~\ref{assum:f} holds for some $0<k_1\leq k$ and omit index $\zeta$ for simplicity. We also note that the main strategies of the proof in Appendices \ref{proofofthmwellposed} and \ref{sec:proofofthmconsistent} are similar to the existing paper~\citep{ding2021overparameterization}. Thus, we omit some detailed calculations in this paper, which can also be found in \cite[Sections 4.8, 4.9]{Zhiyan_thesis}.

\section{Well-Posedness and Regularity Estimate}\label{proofofthmwellposed}
We collect some a-priori estimates in this section. In particular, we will present the well-posedness of the continuous limit and mean-field limit of ResNet~\eqref{eqn:contRes} and~\eqref{eqn:meancontRes} respectively, and present the training processes, characterized by gradient flows~\eqref{eqn:gradient_theta_LM},~\eqref{eqn:gradient_theta_M}, and \eqref{eqn:Wassgradientflow}, are also well-defined.

The summary is:

\begin{theorem}\label{thm:wellposed1}
Suppose that Assumption \ref{assum:f} holds and that $|x|<R$. The ResNet has unique solutions in the continuous and mean-field limit:
\begin{enumerate}[wide,labelindent=0pt]
\item[--] The solution $z(t;x)$ to~\eqref{eqn:contRes} is unique and $\mathcal{C}^1$ in time, if $\{\theta_{m}(t)\}_{m=1}^M$ is a collection of $M$ continuous functions;
\item[--] The solution $z(t;x)$  to~\eqref{eqn:meancontRes} is unique and $\mathcal{C}^1$ in time, if $\rho\in \mathcal{C}([0,1];\mathcal{P}^2)$.
\end{enumerate}
Moreover, the solution is bounded and stable with respect to $\rho$. Let $\mathcal{L}_i$ be the second moments of $\rho_i$, that is,
\begin{equation} \label{eq:Li}
    \mathcal{L}_{i}=\int^1_0\int_{\mathbb{R}^k}|\theta|^2 \rho_i(\rd\theta,t)\rd t, \quad i=1,2\,,
\end{equation}
then
\begin{equation}\label{boundofxsolution}
\left|Z_{\rho_i}(t;x)\right|\leq C(\mathcal{L}_i)\,,
\end{equation}
and for all $t\in[0,1]$:
\begin{equation}\label{stabilityofxsolution}
\left|Z_{\rho_1}(t;x)-Z_{\rho_2}(t;x)\right|\leq C(\mathcal{L}_1,\mathcal{L}_2)\left(\int^1_0W^2_2(\rho_1(\cdot,t),\rho_2(\cdot,t))\rd t\right)^{1/2}\leq  C(\mathcal{L}_1,\mathcal{L}_2)d_1(\rho_1,\rho_2)\,.
\end{equation}
\end{theorem}

\rev{The proof of the above theorem follows directly from standard well-posedness and regularity results for ordinary differential equations; see~\cite[Theorems 4.8.1–4.8.2]{Zhiyan_thesis} for a complete proof. The theorem ensures that the ODE systems \eqref{eqn:contRes} and \eqref{eqn:meancontRes} admit unique, bounded, regular, and stable solutions. These properties are essential for applying Gronwall’s inequality in both the mean-field convergence proof and the subsequent stability analysis.}

\begin{theorem}\label{thm:wellposed}
Suppose that Assumptions~\ref{assum:f} and~\ref{ass:g} are satisfied. The training processes are well-posed. In particular,
\begin{enumerate}[wide,labelindent=0pt]
\item[--] The solution $\Theta^\zeta_{L,M}$  to~\eqref{eqn:gradient_theta_LM} is unique.
\item[--] The solution $\{\theta_{m}(s;t)\}_{m=1}^M$ to~\eqref{eqn:gradient_theta_M} is unique and continuous in $(s,t)$, if the initial data $\{\theta_{m}(0;t)\}_{m=1}^M$ are continuous in $t$ for $s=0$.
\item[--] The solution $\rho(\theta,t,s)$ to~\eqref{eqn:Wassgradientflow} is unique in $\mathcal{C}([0,\infty);\mathcal{C}([0,1];\mathcal{P}^2))$ and is admissible, if the initial condition $\rho_{\ini}(\theta,t)$ is admissible.
Moreover, the energy decays in training time:
\begin{equation}\label{eqn:decayEsmean-field}
\frac{\rd E(\rho(\cdot,\cdot,s))}{\rd s}\leq 0\,.
\end{equation}
\end{enumerate}
\end{theorem}
\rev{The proof of the above theorem involves a standard fixed-point argument; see~\cite[Theorems 4.8.3–4.8.4]{Zhiyan_thesis} for a complete proof. In the above theorem,~\eqref{eqn:decayEsmean-field} ensures that the decay of energy can be preserved when considering the continuous and mean-field limit. }

Next, to prepare for the later proof, we estimate the regularity of the Fr\'echet derivative $\frac{\delta E(\rho)}{\delta \rho}$. More specifically, under proper assumptions of $f$, we show that its gradient is locally Lipschitz in each component.~\rev{The locally Lipschitz property ensures that the functional $E$ is stable with varying $\rho$ and $\theta$, which is important in the mean-field convergence and stability analyses.}

According to~\eqref{eqn:Frechetderivative}, the Fr\'echet derivative can be computed, similarly to \cite{pmlr-v119-lu20b}, as follows:
\begin{equation}\label{eqn:frechet_s}
\frac{\delta E(\rho)}{\delta \rho}(\theta,t)=\mathbb{E}_{x\sim\zeta}\big(p^\top_\rho(t;x)f(Z_\rho(t;x),\theta)\big)\,,
\end{equation}
where $p_\rho(t;x)$ is the solution to the following ODE:
\begin{equation}\label{eqn:prhoappendix}
\left\{
\begin{aligned}
\frac{\partial p^\top_\rho}{\partial t} &=-p^\top_\rho\int_{\mathbb{R}^k}\partial_zf(Z_\rho(t;x),\theta)\rho(\rd\theta,t)\,,\\
p_\rho(1;x)& =\left(g(Z_{\rho}(1;x))-y(x)\right)\nabla g(Z_{\rho}(1;x))\,.
\end{aligned}
\right.
\end{equation}

The boundedness and regularity estimations are summarized in the following lemma and the proof can be found in~\citep[Lemmas 4.8.7, 4.8.8]{Zhiyan_thesis}.~\rev{The lemma below first establishes a regularity result for a broad class of activation functions~$f$, which may be of independent interest; in the special case where $f$ is $2$-homogeneous or partially $1$-homogeneous (as considered in this paper), these estimates can be substantially strengthened.
}

\begin{lemma}\label{lemmadeltaprhoupdate}
Suppose that Assumption~\ref{assum:f} holds, $|x|<R$, and that $\rho_1,\rho_2\in \mathcal{C}([0,1];\mathcal{P}^2)$.
Define $\mathcal{L}_1$ and $\mathcal{L}_2$ as in \eqref{eq:Li},
\[
 \mathcal{L}^{\sup}_i=\sup_{t\in[0,1]}\int_{\mathbb{R}^k}|\theta|^2\rho_i(\rd\theta,t)\,,
\]
and
\[
\mathcal{R}=\inf_{r}\left\{r\middle|\mathrm{supp}(\rho_1)\cup\mathrm{supp}(\rho_2)\subset\left\{\theta\middle|\left|\theta_{[1]}\right|<r\right\}\right\}\,.
\]
Then we have the following results.
\begin{itemize}

\item Boundedness of $p_\rho$:
\begin{equation}\label{boundofprho}
\left|p_{\rho_1}(t;x)\right|\leq C(\mathcal{L}_1)\,.
\end{equation}

\item Lipschitz continuity in $\theta$ and $t$: There exists $Q_1:\mathbb{R}^2\rightarrow\mathbb{R}^+$ that depends increasingly on both its arguments such that
\begin{equation}\label{Lipshictz}
\begin{aligned}
&\left|\nabla_\theta\frac{\delta E(\rho_1)}{\delta \rho}(\theta_1,t_1)-\nabla_\theta\frac{\delta E(\rho_1)}{\delta \rho}(\theta_2,t_2)\right|\\
& \leq Q_1\Big(\mathcal{L}_1, \max_{i=1,2}(|\theta_{i,[1]}|)\Big)|\theta_1-\theta_2|+Q_1\Big(\mathcal{L}^{\sup}_1,
\max_{i=1,2}(|\theta_{i,[1]}|)\Big)(|\theta_2|+1)|t_1-t_2|\,.
\end{aligned}
\end{equation}
\item Lipschitz continuity in $\rho$: There exists $Q:\mathbb{R}^4\rightarrow\mathbb{R}^+$ that increasingly depends on its arguments such that
\begin{equation}\label{eqn:Deltabound2}
\begin{aligned}
\left|\nabla_\theta\frac{\delta E(\rho_1)}{\delta \rho}(\theta,t)-\nabla_\theta\frac{\delta E(\rho_2)}{\delta \rho}(\theta,t)\right|&\leq Q(\mathcal{L}_1,\mathcal{L}_2,|\theta_{[1]}|,\mathcal{R})(1+|\theta|)\left(\int^1_0W^2_2(\rho_1(\cdot,t),\rho_2(\cdot,t))\rd t\right)^{1/2}\\
&\leq Q(\mathcal{L}_1,\mathcal{L}_2,|\theta_{[1]}|,\mathcal{R})(1+|\theta|)d_1(\rho_1,\rho_2)\,,
\end{aligned}
\end{equation}
where $d_1$ is defined in Definition \ref{def:path}.
\end{itemize}

If $f$ further satisfies either Assumption~\ref{assum:2-homogeneous} or Assumption~\ref{assum:f2}, we have the following improved estimations:

\begin{itemize}
\item Stability of $p_\rho$:
\begin{equation}\label{stabilityofprhoupdate}
\left|p_{\rho_1}(t;x)-p_{\rho_2}(t;x)\right|\leq C(\mathcal{L}_1,\mathcal{L}_2)d_1(\rho_1,\rho_2)\,.
\end{equation}
\item Lipschitz continuity in $\rho$:
\begin{equation}\label{eqn:Deltabound2update}
\left|\nabla_\theta\frac{\delta E(\rho_1)}{\delta \rho}(\theta,t)-\nabla_\theta\frac{\delta E(\rho_2)}{\delta \rho}(\theta,t)\right|\leq C(\mathcal{L}_1,\mathcal{L}_2)d_1(\rho_1,\rho_2)\left(|\theta|+1\right)\,,
\end{equation}
where $d_1$ is defined in Definition~\ref{def:path}.
\end{itemize}

If $f$ further satisfies Assumption~\ref{assum:2-homogeneous}, we have the following improved estimation:

\begin{itemize}
\item Lipschitz continuity in $\theta$ and $t$: for any $(\theta_1,t_1),\,(\theta_2,t_2)\in \mathbb{R}^k\times[0,1]$, we have
\begin{equation}\label{Lipshictzupdate}
\left|\nabla_\theta\frac{\delta E(\rho_1)}{\delta \rho}(\theta_1,t_1)-\nabla_\theta\frac{\delta E(\rho_1)}{\delta \rho}(\theta_2,t_2)\right|\leq C(\mathcal{L}_1)|\theta_1-\theta_2|+C(\mathcal{L}^{\sup}_1)(|\theta_2|+1)|t_1-t_2|\,.
\end{equation}

\end{itemize}
\end{lemma}

\section{Proof of Theorem \ref{thm:consistent}}\label{sec:proofofthmconsistent}
Theorem~\ref{thm:consistent} links the loss defined by $\Theta_{L,M}(s)$ with that defined by $\rho(\theta,t,s)$ for all $s$.
The continuous and mean-field limits are obtained, with both $L$ and $M$ sent to infinity.
We decompose this result into two parts, discussing mean-field and continuous limits separately.

We start with the full definition of ``limit-admissible'' for a distribution $\rho$.
\begin{definition}\label{def:meanadmissible2}
For an admissible $\rho(\theta,t)$, we say $\rho(\theta,t)$ is {\em limit-admissible} if the average of a large number of particle representations is bounded and Lipschitz with high probability. That is, for an admissible $\rho(\theta,t)$, there are two constants $C_4$ and $C_5$, both greater than $\sup_{t\in[0,1]}\int_{\mathbb{R}^k}|\theta|^2d\rho(\theta,t)$ such that, for any $M$ stochastic process representations $\{\theta_m(t)\}^M_{m=1}$ that are $i.i.d.$ drawn from $\rho(\theta,t)$, the following properties are satisfied for any $\eta>0$ and $M>\frac{C_4}{\eta}$:
\begin{itemize}
\item[1.] Second moment boundedness in time:
\begin{equation}\label{boundLdis0}
\mathbb{P}\left(\sup_{t\in[0,1]}\frac{1}{M}\sum^M_{m=1}|\theta_m(t)|^2 \leq C_5\right)\geq 1-\eta\,.
\end{equation}
\item[2.] For all $L>0$, we have
\begin{equation}\label{diffrencesmall}
\mathbb{P}\left(\frac{1}{M}\sum^{L-1}_{l=0}\sum^M_{m=1}\int^{\frac{l+1}{L}}_{\frac{l}{L}}\left|\theta_m(t)-\theta_m\left(\frac{l}{L}\right)\right|^2\rd t\leq \frac{C_5}{L^2}\right)\geq 1-\eta\,.
\end{equation}
\end{itemize}
\end{definition}
The “limit-admissible” property guarantees that, with high probability, any parameter trajectory sampled from $\rho$ remains bounded and continuous in empirical averages.



We now state the two theorems that play complementary parts in Theorem~\ref{thm:consistent}.
The first theorem addresses the limit in $M$ under the assumption that $L=\infty$.
This is the mean-field part of the analysis.
\begin{theorem}\label{thm:mean-field}
Let Assumptions~\ref{assum:f} and~\ref{ass:g} hold with some $0< k_1\leq k$. Assume that $\rho_{\ini}(\theta,t)$ is limit-admissible
and $\mathrm{supp}_\theta(\rho_{\ini}(\theta,t))\subset \{\theta||\theta_{[1]}|\leq R'\}$ with some $R'>0$ for all $t\in[0,1]$.
Suppose that $\{\theta_m(0;t)\}^M_{m=1}$ are $i.i.d$ drawn from $\rho_{\ini}(\theta,t)$.
Suppose in addition that
\begin{itemize}
\item $\rho(\theta,t,s)$ solves \eqref{eqn:Wassgradientflow} with the initial condition $\rho_{\ini}(\theta,t)$, and
\item $\theta_m(s;t)$ solves \eqref{eqn:gradient_theta_M} with the initial condition $\theta_m(0;t)$.
\end{itemize}
Then for any $\epsilon,\eta,S>0$, there exists a constant $C(\rho_{\ini}(\theta,t),R,S)>0$ depending on $\rho_{\ini}(\theta,t),R,S$ such that when $M>\frac{C(\rho_{\ini}(\theta,t),R,S)}{\epsilon^2\eta}$,
we have for all $s<S$ that
\begin{align*}
\mathbb{P}\left(\left|Z_{\Theta(s)}(1;x)-Z_{\rho(s)}(1;x)\right|\leq\epsilon\right) & \geq 1-\eta, \\
\mathbb{P}\left(\left|\widehat{E}(\Theta(s;\cdot))-E(\rho(\cdot,\cdot,s))\right|\leq \epsilon\right) & \geq 1-\eta\,.
\end{align*}
\end{theorem}
\rev{We note that the existence and uniqueness of $\rho(\theta,t,s)$ and $\theta_m(s;t)$ are guaranteed in the previous section.} The strategy for the proof is presented in Appendix~\ref{sec:proofofthmmean-field}. The detailed derivation is found in~\cite[Section 4.9.1]{Zhiyan_thesis}. The conclusion of this result suggests that for a $1-\eta$ confidence of an $\epsilon$ accuracy, $M$ grows polynomially with respect to  $1/\epsilon$ and $1/\eta$.

The second result considers the convergence of the parameter configuration for the discrete ResNet \eqref{eqn:disRes} to that for the continuous ResNet  \eqref{eqn:contRes} as $L \to \infty$. This is the continuous-limit part of the analysis.

\begin{theorem}\label{thm:contlimit}
Let Assumptions~\ref{assum:f} and~\ref{ass:g} hold with some $0< k_1\leq k$. Assume that $\rho_{\ini}(\theta,t)$ is limit-admissible
and $\mathrm{supp}_\theta(\rho_{\ini}(\theta,t))\subset \{\theta||\theta_{[1]}|\leq R'\}$ with some $R'>0$ for all $t\in[0,1]$.
Suppose that $\{\theta_m(0;t)\}^M_{m=1}$ are $i.i.d$ drawn from $\rho_{\ini}(\theta,t)$. Suppose in addition that
\begin{itemize}
\item $\theta_m(s;t)$ solves~\eqref{eqn:gradient_theta_M} with initial condition $\theta_m(0;t)$,
\item $\theta_{l,m}(s)$ solves~\eqref{eqn:gradient_theta_LM} with initial condition $\theta_m\left(0;\frac{l}{L}\right)$\,.
\end{itemize}
Then for any $\epsilon,\eta,S>0$, there exists a constant $C(\rho_{\ini}(\theta,t),S)>0$ depending on $\rho_{\ini}(\theta,t),R,S$ such that when $M\geq \frac{C(\rho_{\ini}(\theta,t),R,S)}{\eta}$ and $L\geq \frac{C(\rho_{\ini}(\theta,t),R,S)}{\epsilon}$, we have for all $s<S$ that
\begin{align*}
\mathbb{P}\left(\left|Z_{\Theta(s)}(1;x)-Z_{\Theta_{L,M}(s)}(1;x)\right|\leq\epsilon\right) & \geq 1-\eta\,, \\
\mathbb{P}\left(\left|\widehat{E}(\Theta(s;\cdot))-\mathcal{E}(\Theta_{L,M}(s))\right|\leq \epsilon\right) & \geq 1-\eta\,.
\end{align*}
\end{theorem}
The proof for this theorem involves very tedious derivation and is omitted from here. We refer interested readers to~\cite[Section 4.9.2]{Zhiyan_thesis}.

This theorem shows that when the width is large enough, for any fixed training time $s<S$, with high probability, the difference between the loss functions defined by the discrete ResNet and its continuous counterpart decreases to $0$ as $L\rightarrow\infty$.


\subsection{Convergence to Mean-Field PDE}
\label{sec:proofofthmmean-field}

We provide the intuition for showing Theorem~\ref{thm:mean-field}. The proof strategy is very similar to other mean-field results in this context, see also~\citep{ding2021overparameterization}.

To show the ``equivalence'' between~\eqref{eqn:gradient_theta_M} and~\eqref{eqn:Wassgradientflow}, we can test them on the same smooth function $h(\theta)$. Testing~\eqref{eqn:Wassgradientflow} amounts to multiplying both sides of the equation by $h$. From integration by parts we have
\[
\frac{\rd}{\rd s}\int_{\mathbb{R}^k}h\rd\rho(\theta) = -\int_{\mathbb{R}^k}\nabla_\theta h \cdot \nabla_\theta\frac{\delta E(\rho(s))}{\delta \rho}\rd{\rho}\,,
\]
that is,
\[
\frac{\rd}{\rd s}\mathbb{E}(h)=-\EE\left(\nabla_\theta h \cdot \nabla_\theta\frac{\delta E(\rho(s))}{\delta \rho}\right)\,.
\]
\rev{Next, we test~\eqref{eqn:gradient_theta_M} on $h$. In this case, $\mathbb{E}(h)=\frac{1}{M}\sum_{m=1}^Mh(\theta_m)$. Then,}
\[
\frac{\rd}{\rd s}\mathbb{E}(h) = \frac{1}{M}\sum_{m=1}^M\nabla_\theta h(\theta_m) \cdot \frac{\rd \theta_m}{\rd s} = -\sum_{m=1}^M\nabla_\theta h(\theta_m) \cdot \frac{\delta \widehat{E}(\Theta)}{\delta\theta_m}\,.
\]
We see that~\eqref{eqn:Wassgradientflow} and~\eqref{eqn:gradient_theta_M} are equivalent when tested by $h$, if and only if the right hand sides of the two equations above are the same, that is,
\begin{equation}\label{eqn:frechet_derivative_equivalence}
M\frac{\delta \widehat{E}(\Theta)}{\delta \theta_m}=\nabla_\theta\frac{\delta E(\rho)}{\delta\rho}(\theta_m,t)\,.
\end{equation}
\rev{
Although these two expressions appear quite different, in the empirical‐measure case $\rho = \frac{1}{M}\sum_{m=1}^M \delta_{\theta_m}$, one can rigorously verify \eqref{eqn:frechet_derivative_equivalence} by comparing the definitions of the Fréchet derivatives $\frac{\delta E(\rho)}{\delta \rho}$ and $\frac{\delta \widehat{E}(\Theta)}{\delta \theta_m}$ (see Lemma~33 in \cite{ding2021overparameterization}). Consequently, as $\frac{1}{M}\sum_{m=1}^M \delta_{\theta_m}\rightarrow\rho$ when
  $M\rightarrow\infty$ the two formulations coincide in the limit, thereby justifying the mean‐field convergence.}


To give a quantitative estimate on how quickly~\eqref{eqn:gradient_theta_M} converges to~\eqref{eqn:Wassgradientflow}, we utilize the particle method, a classical strategy for the mean-field limit. We sketch the proof here.
We make use of two particle systems.
In one system, the particles evolve themselves, while in the second, the particles are moved forward according to the underlying field constructed by the limit.
In our situation, the former particle system consists of  the $M$ stochastic processes $\{\theta_m(s;t)\}$ that descend according to $\widehat{E}(\Theta(s;\cdot))$.
The latter particle system will be termed  $\widetilde{\Theta}(s;t)=\{\widetilde{\theta}_m(s;t)\}_{m=1}^M$; ~\rev{it descends according to $E(\rho(\cdot,\cdot,s))$, the limiting loss function, and satisfies the condition $\widetilde{\Theta}(s;t)=\{\widetilde{\theta}_m(s;t)\}_{m=1}^M\sim_{\rm i.i.d.}\rho(\cdot,\cdot,s)$. Intuitively, when $M$ is large enough, we prove}
\begin{equation}\label{eqn:E_close}
\widehat{E}(\Theta(s;\cdot))\approx \widehat{E}\left(\widetilde{\Theta}(s;\cdot)\right)\approx E(\rho(\cdot,\cdot,s))\,.
\end{equation}
\rev{The second approximation comes from the law of large numbers while the first one requires tracing the two different evolving ODEs to show the closeness of the two particle systems $\{\theta_m(s;t)\}$ and $\widetilde{\Theta}(s;t)=\{\widetilde{\theta}_m(s;t)\}_{m=1}^M$.}

\rev{In our case, demonstrating the closeness of the first and third terms in~\eqref{eqn:E_close} reduces to showing the closeness of the outputs of the neural network.} To be more specific, let $\rho(\theta,t,s)$ be the solution to~\eqref{eqn:Wassgradientflow} with admissible initial conditions $\rho_{\ini}(\theta,t)$, and let $\theta_m(s;t)$ be the solution to~\eqref{eqn:gradient_theta_M} with initial conditions $\{\theta_m(0;t)\}^M_{m=1}$ that are $i.i.d$ drawn from $\rho_{\ini}(\theta,t)$.
Using the definition of $E$ in~\eqref{eqn:lossfunctioncont2}, we have
\begin{equation}\label{eqn:differenceofE3}
\begin{aligned}
&\left|E(\rho(\cdot,\cdot,s))-\widehat{E}(\Theta(s;\cdot))\right|\\
& \leq \mathbb{E}_{x\sim\zeta}\left[\left|\frac{1}{2}\left(g(Z_{\rho(s)}(1;x))-y(x)\right)^2-\frac{1}{2}\left(g(Z_{{\Theta(s)}}(1;x))-y(x)\right)^2\right|\right]\\
& \leq \mathbb{E}_{x\sim\zeta}\left(\left|g(Z_{\rho(s)}(1;x))-g(Z_{{\Theta(s)}}(1;x))\right|\left(\left|g(Z_{\rho(s)}(1;x))+g(Z_{{\Theta(s)}}(1;x))\right|+|y(x)|\right)\right)\\
& \leq C(\mathcal{L}_s)\mathbb{E}_{x\sim\zeta}(\left|g(Z_{\rho(s)}(1;x))-g(Z_{{\Theta(s)}}(1;x))\right|)\\
& \leq C(\mathcal{L}_s)\mathbb{E}_{x\sim\zeta}(\left|Z_{\rho(s)}(1;x)-Z_{{\Theta(s)}}(1;x)\right|)\,,
\end{aligned}
\end{equation}
where $\mathcal{L}_s=\max\left\{\int^1_0\int_{\mathbb{R}^k}|\theta|^2\rd\rho(\theta,t,s)\rd t,\frac{1}{M}\int^1_0\sum^M_{m=1}\left|\theta_m(s;t)\right|^2\rd t\right\}$.
In this derivation, we used the Lipschitz property of $g$, the boundedness of $y$  (required in Assumptions~\ref{ass:g}), and the boundedness of $Z_\rho$ and $Z_{{\Theta(s)}}$.
Boundedness of $Z_\rho$ was already shown in Theorem~\ref{thm:wellposed1}, and the bound for $Z_{\Theta(s)}$, a newly induced quantity, still needs to be shown. The constant $C$ depends on the support of $\rho_\ini$, as well as on $s$ and the Lipschitz constant of $g$.~\rev{Therefore, instead of directly showing~\eqref{eqn:E_close}, it follows from \eqref{eqn:differenceofE3} that it suffices to control
\begin{equation}\label{eqn:z_diff}
\left|Z_{\rho(s)}(1;x)-Z_{{\Theta(s)}}(1;x)\right|\,.
\end{equation}}
\rev{To do so, we employ the particle method and invent the new particle system $\widetilde{\Theta}$ and show
\[
\left|Z_{\rho(s)}(1;x)-Z_{{\widetilde{\Theta}(s)}}(1;x)\right|,\quad \left|Z_{\widetilde{\Theta}(s)}(1;x)-Z_{{\Theta(s)}}(1;x)\right|
\]
separately.}

According to~\eqref{eqn:frechet_derivative_equivalence}, we can reformulate the original particle system as
\begin{equation}\label{eqn:Wassgradientflowsdisappendix}
\frac{\rd\theta_m(s;t)}{\rd s}=-\nabla_{\theta}\frac{\delta E(\rho^{\dis}(s))}{\delta \rho}(\theta_m(s;t))\,,\quad\forall (t,s)\in[0,1]\times[0,\infty)\,,
\end{equation}
where we denote
\begin{equation}\label{def:ensemble_theta}
\rho^{\dis}(\theta,t,s)=\frac{1}{M}\sum^M_{m=1}\delta_{\theta_m(s;t)}(\theta)\,.
\end{equation}
The new system is driven by underlying flow governed by the limit. Define $\widetilde{\Theta}(s)=\{\widetilde{\theta}_m(s;t)\}_{m=1}^M$, where $\widetilde{\theta}_m$ solves
\begin{equation}\label{eqn:Wassgradientflowsdiscoupling}
\frac{\rd \widetilde{\theta}_m(s;t)}{\rd s}=-\nabla_{\theta}\frac{\delta E(\rho(s))}{\delta \rho}\big(\widetilde{\theta}_m(s;t)\big)\,,\quad\forall (t,s)\in[0,1]\times[0,\infty)\,,
\end{equation}
with initial condition \begin{equation}\label{eqn:sameinitial}\widetilde{\theta}_m(0;t) = \theta_m(0;t)\,.\end{equation}
As a consequence, we have $\widetilde{\theta}_m(s;t)\sim \rho(\theta,t,s)$ for all $(t,s)$.
The corresponding ensemble distribution is
\begin{equation}\label{def:LSmean-field}
\widetilde{\rho}^{\dis}(\theta,t,s)=\frac{1}{M}\sum^M_{m=1}\delta_{\widetilde{\theta}_m(s;t)}(\theta)\,.
\end{equation}
Now we have available particle system $\Theta(s)=\{\theta_m(s)\}$, a newly invented particle system $\widetilde{\Theta}(s)=\{\widetilde{\theta}_m(s)\}$ and the mean-field flow $\rho$.
Accordingly, there are three versions of $Z$: $Z_{\rho(s)}$ that solves~\eqref{eqn:meancontRes} using $\rho(s)$ and $Z_{\Theta(s)}$ and $Z_{\widetilde{\Theta}(s)}$ that solve~\eqref{eqn:contRes} using $\Theta(s)$ and $\widetilde{\Theta}(s)$, respectively.
We use the following relabelling for convenience:
\begin{equation}\label{def:z_notation}
Z_s=Z_{\rho(s)}\,,\quad Z^\dis_s=Z_{\Theta(s)}\,,\quad \widetilde{Z}^{\dis}_s=Z_{\widetilde{\Theta}(s)}\,.
\end{equation}
Similarly, there are three sets of $p$: $p_{\rho(s)}$, $p_{\rho^\dis(s)}$, and $p_{\widetilde{\rho}^\dis(s)}$ that solve~\eqref{eqn:prho} using $\rho(s)$, $\rho^{\dis}(s)$, and $\widetilde{\rho}^\dis(s)$, respectively. We relabel similarly to \eqref{def:z_notation} and write
\begin{equation}\label{def:p_notation}
p_s=p_{\rho(s)}\,,\quad p^\dis_s=p_{\rho^\dis(s)}\,,\quad \widetilde{p}^\dis_s=p_{\widetilde{\rho}^\dis(s)}\,.
\end{equation}

Since $\widetilde{\Theta}(s)$ serves as a bridge, we translate the control of~\eqref{eqn:z_diff} to:
\begin{equation}\label{eqn:z_diff_2}
\left|Z_s(t;x)-Z^\dis_s(t;x)\right|\leq \left|Z_s(t;x)-\widetilde{Z}^\dis_s(t;x)\right|+\left|\widetilde{Z}^\dis_s(t;x)-Z_s^\dis(t;x)\right|\,.
\end{equation}

A bound on $\left|Z_s(t;x)-\widetilde{Z}^\dis_s(t;x)\right|$ can be found by using the law of large numbers, and a bound on $\left|\widetilde{Z}^\dis_s(t;x)-Z_s^\dis(t;x)\right|$ translates to controlling $\sum_{m}\int_0^1|\theta_m(s;t)-\widetilde{\theta}_m(s;t)|\rd{t}$, which can be evaluated by the difference between equations~\eqref{eqn:Wassgradientflowsdisappendix} and~\eqref{eqn:Wassgradientflowsdiscoupling}.
Since these two equations have the same initial data, the difference between $\theta_m$ and $\widetilde{\theta}_m$ can then be controlled when the right-hand side forcing terms are close.~\rev{We note that the rigorous analysis is coupled in nature: the bound on $\left|\widetilde{Z}^\dis_s(t;x)-Z_s^\dis(t;x)\right|$ depends on controlling $\sum_{m}\int_0^1|\theta_m(s;t)-\widetilde{\theta}_m(s;t)|\,\rd{t}$, while the evolution of the difference $|\theta_m(s;t)-\widetilde{\theta}_m(s;t)|$ is itself influenced by $\left|\widetilde{Z}^\dis_s(t;x)-Z_s^\dis(t;x)\right|$ through the loss function. A rigorous proof thus requires careful control of the propagation of random error throughout the training dynamics.
}

\section{Proof of Global Convergence Results}\label{sec:proofofconvergencetoglobalminimal}
We give the proof for the global convergence of the PDE in this section.

\subsection{Proof of Proposition~\ref{prop:localisglobal}}
\label{sec:proofofproplocal}

\rev{In this section, we prove Proposition~\ref{prop:localisglobal}. We first recall the formula for $\frac{\delta E(\rho)}{\delta\rho}$ in~\eqref{eqn:Frechetderivative}:
\[
\left.\frac{\delta E}{\delta \rho}\right|_{\rho}(\theta,t)=\mathbb{E}_{x\sim\zeta}\big(p^\top_{\rho}(t;x)f(Z_{\rho}(t;x),\theta)\big)\,,
\]
where $p_\rho$ is defined according to~\eqref{eqn:prho}. In the case when $E(\rho)\neq 0$, we can find a strictly positive lower bound for the $\EE_{x\sim\zeta}\left(|p_{\rho}(t;x)|^2\right)$, which is shown in Lemma~\ref{lem:lowerboundofprho}. Then, to find $\nu(\theta)$ such that
$\int_{\mathbb{R}^k}\rd\nu(\theta)=0$ and $
\int_{\mathbb{R}^k}\frac{\delta E(\rho)}{\delta \rho}(\theta,t_0)\rd\nu(\theta)<0$, it suffices to find a probability measure $\hat{\nu}$ such that
\[
\int_{\mathbb{R}^k}f(Z_{\rho}(t;x),\theta)\rd\hat{\nu}(\theta)\approx \int_{\mathbb{R}^k}f(Z_{\rho}(t;x),\theta)\rd\rho(\theta)+p_{\rho}(t;x)\Rightarrow -p_{\rho}(t;x)\approx \int_{\mathbb{R}^k}f(Z_{\rho}(t;x),\theta)\rd(\rho-\hat{\nu})\,.
\]
and let $\nu=\rho(\cdot,t_0)-\hat{\nu}$. $\hat{\nu}$ can be constructed based on the universal kernel property in Assumption~\ref{assum:f}, after applying the pushforward map to recover $x$ from $Z_{\rho}(t;x)$.
}

We first  prove a lower bound for $p_\rho$ in the following lemma.
\begin{lemma}\label{lem:lowerboundofprho} Suppose that $\rho\in \mathcal{C}([0,1];\mathcal{P}^2)$ and that $p_\rho$ is a solution to \eqref{eqn:prho}. Denoting
\[
\mathcal{L}_\rho=\int^1_0\int_{\mathbb{R}^k}|\theta|^2\rho(\rd\theta,t)\mathrm{d}t\,,
\]
then for any $t\in[0,1]$ we have that
\begin{equation}\label{eqn:lowerboundofprho}
\EE_{x\sim\zeta}\left(|p_\rho(t;x)|^2\right)\geq Q(\mathcal{L}_{\rho})E(\rho)\,,
\end{equation}
where $Q:\mathbb{R}_+\rightarrow\mathbb{R}_+$ is a decreasing function.
\end{lemma}
\begin{proof}
Recall that the initial condition for $p_\rho$ in~\eqref{eqn:prho} is:
\[
p_\rho(1;x)=\left(g(Z_{\rho}(1;x))-y(x)\right)\nabla g(Z_{\rho}(1;x))\,,
\]
so from~\rev{Assumption~\ref{ass:g}}, we have
\[
\EE_{x\sim\zeta}\left(|p_{\rho}(1;x)|^2\right)\geq \Big(\inf_{x\in\mathbb{R}^d}|\nabla g(x)|\Big)^2E(\rho)\,.
\]
Further, since the equation is linear, we have
\[
\frac{\partial p^\top_\rho}{\partial t}=-p^\top_\rho\int_{\mathbb{R}^k}\partial_zf(Z_\rho,\theta)\rho(\rd\theta,t)\,.
\]
According to equation~\eqref{eqn:derivativebound} in Assumption~\ref{assum:f}, we obtain
\begin{align*}
\left|\int_{\mathbb{R}^k}\partial_zf(Z_{\rho}(t;x),\theta)\rho(\rd\theta,t)\right| & \leq C(Z_{\rho}(t;x))\int_{\mathbb{R}^k}(|\theta|^2+1)\rho(\rd\theta,t)\\
& \leq C(\mathcal{L}_\rho)\int_{\mathbb{R}^k}(|\theta|^2+1)\rho(\rd\theta,t)\,,
\end{align*}
where we use \eqref{boundofxsolution} in the second inequality. By combining the last two bounds in the usual way, we obtain
\[
\frac{\rd|p_\rho(t;x)|^2}{\rd t}\leq \left(2C(\mathcal{L}_{\rho})\int_{\mathbb{R}^k}(|\theta|^2+1)\rho(\rd\theta,t)\right)|p_\rho(t;x)|^2\,,
\]
By solving the equation, we have
\[
|p_\rho(t;x)|^2\geq |p_{\rho}(1;x)|^2\exp\left(-2C(\mathcal{L}_{\rho})\int^1_t\int_{\mathbb{R}^k}(|\theta|^2+1)d\rho(\theta,s)\mathrm{d}s\right)\geq C(\mathcal{L}_{\rho})|p_{\rho}(1;x)|^2\,.
\]
The proof is finalized by taking expectation on both sides, and noting that
monotonicity comes from the format of the exponential term.
\end{proof}

We are now ready to prove Proposition~\ref{prop:localisglobal}.
\begin{proof}[Proposition~\ref{prop:localisglobal}]
Denote
\[
\mathcal{L}_{\rho}=\int^1_0\int_{\mathbb{R}^k}|\theta|^2\rho(\rd\theta,t)\rd t\,.
\]
According to existence and uniqueness of the solution to \eqref{eqn:meancontRes}, for any $t\in[0,1]$, we can construct a map $\mathcal{Z}_t$ such that
\[
\mathcal{Z}_t(x)=Z_{\rho}\left(t;x\right)\,.
\]
Since the trajectory can be computed backwards in time, $\mathcal{Z}^{-1}_t$ is well defined.
Further, we denote $\zeta^\ast_t=(\mathcal{Z}_t)_{\sharp}\zeta$ to be the pushforward of $\zeta$ under map $\mathcal{Z}_t$ and let
\[
p^\ast(t;x)=p_\rho\left(t;\mathcal{Z}^{-1}_t(x)\right)\,.
\]
By Assumption \ref{assum:f} and classical ODE theory, $\mathcal{Z}_t$ and $\mathcal{Z}^{-1}_t$ are both continuous maps in $x$, and so are $p_\rho(t;x)$ and $p^\ast(t;x)$. With the change of variables, we have for all $t\in[0,1]$ that
\begin{equation}\label{changeofvariable}
\frac{\delta E(\rho)}{\delta \rho}(\theta,t)=\int_{\mathbb{R}^d}p^\top_\rho(t;x)f(\mathcal{Z}_t(x),\theta)\rd\zeta(x)=\int_{\mathbb{R}^d}(p^\ast(t;x))^\top f(x,\theta)\rd\zeta^\ast_t(x)\,.
\end{equation}
For a fixed $t_0\in[0,1]$,  we have boundedness of $\mathcal{Z}_t$ from Theorem \ref{thm:wellposed1}.
As a consequence, $\zeta^\ast_{t_0}(x)$ has a compact support since $\zeta(x)$ does.
We denote the size of the support by $R^\ast$, defined to be a real number such that  $\text{supp}\left(\zeta^\ast_{t_0}(x)\right)\subset \left\{x:|x|<R^\ast\right\}$.

We now derive a general formula for $\int\frac{\delta E(\rho)}{\delta \rho}(\theta,t)\rd\nu$.
Recalling \eqref{changeofvariable}, we have
\begin{equation}\label{eqn:tryprovenegative2}
\begin{aligned}
\int_{\mathbb{R}^k}\frac{\delta E(\rho)}{\delta \rho}(\theta,t_0)\rd\nu(\theta) & =\int_{\mathbb{R}^d}(p^\ast(t_0;x))^\top\left(\int_{\mathbb{R}^k}f(x,\theta)\rd\nu(\theta)\right)\rd\zeta^\ast_{t_0}(x)\\
& =\int_{\mathbb{R}^d}(p^\ast(t_0;x))^\top\left(\int_{\mathbb{R}^k}f(x,\theta)\rd{\nu}(\theta)+p^\ast(t_0;x)\right)\rd\zeta^\ast_{t_0}(x)\\
& \quad-\int_{\mathbb{R}^d}(p^\ast(t_0;x))^\top p^\ast(t_0;x) \rd\zeta^\ast_{t_0}(x)\,.
\end{aligned}
\end{equation}

Noticing that according to Lemma~\ref{lem:lowerboundofprho}, if $E(\rho)\neq 0$, the second term above is strictly negative (less than $-Q(\mathcal{L}_\rho)E(\rho)$), the goal then is to find $\nu$ for which $\int\rd\nu = 0$ that makes  the first term  small, so that the right-hand side in \eqref{eqn:tryprovenegative2} is negative. Defining the continuous function $h$ to be
\[
h(x)=p^\ast\left(t_0;x\right)+\int_{\mathbb{R}^k}f(x,\theta)\rho(\rd\theta,t_0)\,,
\]
then according to Assumption~\ref{assum:f}, for arbitrarily small $\epsilon$, there is a $\hat{\nu}$ so that $\int\hat{\nu}=1$ and
\[
\left\|h-\int_{\mathbb{R}^k}f(\cdot,\theta)\rd\hat{\nu}(\theta)\right\|_{L^\infty(\mathcal{B}_{R^\ast}(0))}\leq \epsilon\,.
\]
Setting $\nu = \rho-\hat{\nu}$ and substituting into the first term of~\eqref{eqn:tryprovenegative2}, we obtain
\begin{equation}\label{prhomustarsmall}
\begin{aligned}
&\int_{\mathbb{R}^d}(p^\ast(t_0;x))^\top\left(\int_{\mathbb{R}^k}f(x,\theta)\rd{\nu}(\theta)+p^\ast(t_0;x)\right)\rd\zeta^\ast_{t_0}(x)\\
& \leq\int_{\mathbb{R}^d}\left|p^\ast(t_0;x)\right|\left|h(x)-\int_{\mathbb{R}^k}f(x,\theta)\rd\hat{\nu}(\theta)\right|\rd\zeta^\ast_{t_0}(x)\\
& \leq \left\|p^\ast(t_0;\cdot)\right\|_{L^\infty(\mathcal{B}_{R^\ast}(0))}\left\|h-\int_{\mathbb{R}^k}f(\cdot,\theta)\rd\hat{\nu}(\theta)\right\|_{L^\infty(\mathcal{B}_{R^\ast}(0))}\,.
\end{aligned}
\end{equation}
By choosing $\epsilon$ small enough that \eqref{prhomustarsmall} is less than $\frac{1}{2}Q(\mathcal{L}_\rho)E(\rho)$, we have from \eqref{eqn:tryprovenegative2} $\int\frac{\delta E(\rho)}{\delta\rho}(\theta,t_0)\rd\nu(\theta)<0$.

{
Now suppose that the hypothesis of Proposition~\ref{prop:localisglobal} holds.  If $E(\rho)>0$, the preceding construction produces a signed measure $\nu$ with $\nu(\mathbb R^k)=0$ and the strict inequality above.  On the other hand,
\[
\int_{\mathbb R^k}\frac{\delta E(\rho)}{\delta\rho}(\theta,t_0)\rd\nu(\theta)
=c\,\nu(\mathbb R^k)=0,
\]
which is a contradiction.  Hence $E(\rho)=0$, completing the proof.}
\end{proof}

\subsection{Blow-up of Second Moment}
\label{sec:blowupmoment}

In this section, we prove Proposition \ref{prop:2-homogeneous} and Proposition \ref{prop:1-homogeneous}.~\rev{In particular, in the case when $\frac{\delta E\big(\rho_\infty\big)}{\delta \rho}(\theta,t_0)\neq0$, we will use homogeneity and separation property to construct a time-dependent set $\mathcal{A}(s)$ in which the second moment of $\rho(\rd\theta,t_0,s)$ blows up when $s$ approaches $\infty$. This contradicts the bounded second moment assumptions of Proposition \ref{prop:2-homogeneous} and Proposition \ref{prop:1-homogeneous}.}

\rev{Although the rigorous proof of two Propositions is technically involved, the intuitive construction of $\mathcal{A}(s)$ is simple and relies crucially on the homogeneous property. Take the 2-homogeneous case as an example. Let $h_s(\theta) = \frac{\delta E(\rho(s))}{\delta \rho}(\theta, t_0)$. If $\frac{\delta E(\rho_\infty)}{\delta \rho}(\theta, t_0) \not\equiv 0$, we can show that there exists a sufficiently large $S$ and a set $A(S)$ on the unit sphere such that $h_S(\theta) < 0$ for all $\theta \in A(S)$. Using the $2$-homogeneous assumption of $f$, we obtain that, for any $\theta/|\theta| \in A(S)$,
\begin{equation}\label{eqn:theta_blow_up}
\left\langle \theta, -\nabla_\theta h_S(\theta) \right\rangle = -2 h_S(\theta) \geq 2 \left( \inf_{\theta_0 \in A(S)} |h_S(\theta_0)| \right) |\theta|^2.
\end{equation}
According to~\eqref{eqn:Wassgradientflow}, the evolution of particles during training satisfies
\[
\frac{\rd \theta(s; t_0)}{\rd s} = -\nabla_{\theta} \frac{\delta E(\rho(s))}{\delta \rho} \left( \theta(s; t_0), t_0 \right) = -\nabla_\theta h_s(\theta(s; t_0)).
\]
This implies that $\left\langle \theta, -\nabla_\theta h_s(\theta) \right\rangle$ controls the rate of growth of $|\theta(s)|^2$. According to~\eqref{eqn:theta_blow_up}, a negative upper bound on $h_s(\theta)$ leads to exponential blow-up of $|\theta(s)|^2$. To arrive at a contradiction, it suffices to verify two properties:
1. $\rho(S)$ assigns positive measure to a rescaled set of $A(S)$;
2. The measure of $\rho(s)$ on the rescaled set does not vanish as $s \to \infty$. To establish the first property, we rely on the separation property of $\rho$ proven in Lemma~\ref{lem:seperateionpreserve} and Lemma~\ref{lem:seperateionpreserve2}, which was also discussed in~\citep{NEURIPS2018_a1afc58c}. The second property requires a more refined analysis of the evolution of $\theta$ using the Sard-type regularity assumption and a careful construction of the set $\mathcal{A}(S)$, which will be presented in detail later in this section and is the main technical contribution of this section.
}

\subsubsection{Proof of Proposition \ref{prop:2-homogeneous}}\label{sec:proofofP2-h}

\rev{Now, we first introduce the following lemma, which shows that the separation property is preserved in the training process. Our proof of this result is adapted from \citep{NEURIPS2018_a1afc58c}.}
%
\begin{lemma}\label{lem:seperateionpreserve}
Suppose the conditions of Theorem \ref{thm:globalminimal1} hold, and suppose that $\rho_{\ini}(\theta,t)$ is admissible with compact support. Let $\rho(\theta,t,s)\in \mathcal{C}([0,\infty);\mathcal{C}([0,1];\mathcal{P}^2))$ solve~\eqref{eqn:Wassgradientflow}. If there exists $t_0\in[0,1]$, so that the initial condition $\rho_{\ini}(\theta,t_0)$ separates the spheres $r_a\mathbb{S}^{k-1}$ and $r_b\mathbb{S}^{k-1}$ for some $0<r_a<r_b$, then for any $s_0\in[0,\infty)$, $\rho(\theta,t_0,s_0)$ separates the spheres $r'_a\mathbb{S}^{k-1}$ and $r'_b\mathbb{S}^{k-1}$ for some $0<r'_a<r'_b$.
\end{lemma}

\begin{proof}[Lemma~\ref{lem:seperateionpreserve}] For every fixed $0<s_0<\infty$, we note that the particle representation $\theta_\rho(s;t_0)$ of $\rho(\theta,t_0,s)$ updates the following equation:
\begin{equation}\label{eqn:particle1}
\left\{
\begin{aligned}
\frac{\rd \theta_\rho(s;t_0)}{\rd s} & =-\nabla_{\theta}\frac{\delta E(\rho(s))}{\delta \rho}\left(\theta_\rho(s;t_0),t_0\right),\quad s\in[0,s_0]\\
\theta_\rho(0;t_0) & =\theta\,.
\end{aligned}
\right.
\end{equation}
Define the map $\mathcal{P}_s(\theta)$ to be the solution map that solves the equation above for given initial condition $\theta$ up to time $s$.
{Since $f(z,\cdot)$ is continuously differentiable and
$2$-homogeneous, $\partial_\theta f(z,\vec0)=0$.  Formula
~\eqref{eqn:Frechetderivative} therefore shows that $\mathcal P_s(\vec0)=\vec0$.
By~\eqref{Lipshictzupdate}, on every finite training-time interval the velocity field is globally Lipschitz in $\theta$. Hence $\mathcal{P}_s$ is a bi-Lipschitz homeomorphism of $\mathbb{R}^k$.
If $K$ is a uniform Lipschitz constant for the velocity on $[0,s_0]$, Gronwall estimates give, for every $y\in\mathbb R^k$,
\begin{equation}\label{eqn:2h-flow-radial-bounds}
 e^{-Ks}|y|\leq|\mathcal P_s^{-1}(y)|\leq e^{K s}|y|.
\qquad 0\leq s\leq s_0.
\end{equation}
The characteristic representation in Theorem~\ref{thm:wellposed} gives
\[
\rho(\cdot,t_0,s)=(\mathcal P_s)_\#\rho_{\ini}(\cdot,t_0)\,,
\qquad
\mathrm{supp}_\theta(\rho(\theta,t_0,s))
=\mathcal P_s(\mathrm{supp}_\theta(\rho_{\ini}(\theta,t_0))) \,.
\]

For the fixed time $s_0$, set
\[
 r_a'=e^{-Ks_0}r_a,
 \qquad
 r_b'=e^{Ks_0}r_b.
\]
Then $0<r_a'<r_b'$.  Suppose, toward a contradiction, that
$\mathrm{supp}_\theta(\rho(\theta,t_0,s_0))$ does not separate
$r_a'\mathbb S^{k-1}$ and $r_b'\mathbb S^{k-1}$.  There is a continuous path
$\gamma:[0,1]\to\mathbb R^k$ such that
\[
 \gamma(0)\in r_a'\mathbb S^{k-1},\qquad
 \gamma(1)\in r_b'\mathbb S^{k-1},\qquad
 \gamma([0,1])\cap
 \mathrm{supp}_\theta(\rho(\theta,t_0,s_0))=\varnothing.
\]
Pull this path back by the inverse flow:
$\widetilde\gamma=\mathcal P_{s_0}^{-1}\circ\gamma$
.  It is continuous,
avoids $\mathrm{supp}_\theta(\rho_{\ini}(\theta,t_0))$, and
\eqref{eqn:2h-flow-radial-bounds} gives
\[
 |\widetilde\gamma(0)|\leq e^{Ks_0}r_a'=r_a,
 \qquad
 |\widetilde\gamma(1)|\geq e^{-Ks_0}r_b'=r_b.
\]
Let $\tau_b$ be the first time at which
$|\widetilde\gamma(\tau_b)|=r_b$, and let $\tau_a\leq\tau_b$ be the
last time before $\tau_b$ at which
$|\widetilde\gamma(\tau_a)|=r_a$.
The restriction $\widetilde\gamma|_{[\tau_a,\tau_b]}$ therefore joins
$r_a\mathbb S^{k-1}$ to $r_b\mathbb S^{k-1}$ while avoiding the
initial support, contradicting its assumed separation.  Hence the
support at time $s_0$ separates $r_a'\mathbb S^{k-1}$ and
$r_b'\mathbb S^{k-1}$.
}
\end{proof}

We are now ready to prove Proposition~\ref{prop:2-homogeneous}.
\begin{proof}[Proposition~\ref{prop:2-homogeneous}]

Define the functions $h_\infty$ and $h_s$ as follows:
\[
h_\infty(\theta)=\frac{\delta E(\rho_\infty)}{\delta \rho}(\theta,t_0),\quad h_s(\theta)=\frac{\delta E(\rho(s))}{\delta \rho}(\theta,t_0)\,.
\]
Suppose, for a contradiction, that $h_\infty\not\equiv0$.
\rev{As discussed in the beginning of this section, we first show there exists a cone-like set $\mathcal{A}$ such that $h_{s}(\theta)$ is negative in $\mathcal{A}$ for sufficiently large $s$. Because $h_s$ converges to $h_{\infty}$ when $s\rightarrow\infty$, it suffices to show $h_{\infty}$ satisfies similar properties.}

\medskip
\noindent{\emph{Step 1: Find a negative direction for the limiting first variation.}}\par
\nopagebreak[4]

{We first show that there exists $\theta_0\neq0$ such that $h_\infty(\theta_0)<0$. The argument draws on many of the ingredients used in the proof of Proposition~\ref{prop:localisglobal}.}
Recall from \eqref{eqn:Frechetderivative} that
\begin{equation}\label{def:deltarho3}\frac{\delta E(\rho)}{\delta \rho}(\theta,t_0)=\EE_{x\sim\zeta}\left(p^\top_\rho(t_0,x) f({Z_\rho}(t_0;x),\theta)\right)\,.
\end{equation}
Under the assumption $h_\infty\not\equiv0$, there must exist a $\theta^\ast$ such that $h_{\infty}(\theta^\ast)\neq0$. From \eqref{def:deltarho3}, by H\"older's inequality,
\[
0< \left|\frac{\delta E(\rho_\infty)}{\delta \rho}(\theta^\ast,t_0)\right|\leq \left( \mathbb{E}_{x\sim\zeta}\left(|p_{\rho_\infty}(t_0;x)|^2\right)\right)^{1/2}\left( \mathbb{E}_{x\sim\zeta}\left(|f({Z_{\rho_\infty}}(t_0;x),\theta^\ast)|^2\right)\right)^{1/2}\,,
\]
which implies
\[
\mathbb{E}_{x\sim\zeta}\left(|p_{\rho_\infty}(t_0;x)|^2\right)>0\,.
\]
{Let $\mathcal Z_{t_0}(x)=Z_{\rho_\infty}(t_0;x)$. By the same change-of-variables and universal-kernel approximation used in the proof of Proposition~\ref{prop:localisglobal}, applied to the continuous target
\[
z\longmapsto
-p_{\rho_\infty}\bigl(t_0;\mathcal Z_{t_0}^{-1}(z)\bigr),
\]
for every $\delta>0$ there exists $\nu\in\mathcal P^2(\mathbb R^k)$ such that
\[
\sup_{x\in\mathrm{supp}(\zeta)}
\left|\int_{\mathbb{R}^k}f(Z_{\rho_\infty}(t_0;x),\theta)\rd\nu(\theta)
+p_{\rho_\infty}(t_0;x)\right|<\delta\,.
\]
It follows that
\[
\begin{aligned}
\int_{\mathbb{R}^k}h_\infty(\theta)\rd\nu(\theta)
&=\mathbb{E}_{x\sim\zeta}\left[p_{\rho_\infty}(t_0;x)^\top
\left(\int_{\mathbb{R}^k}f(Z_{\rho_\infty}(t_0;x),\theta)\rd\nu(\theta)
+p_{\rho_\infty}(t_0;x)\right)\right]\\
&\quad-\mathbb{E}_{x\sim\zeta}\left[|p_{\rho_\infty}(t_0;x)|^2\right]\\
&\leq \delta\,\mathbb{E}_{x\sim\zeta}|p_{\rho_\infty}(t_0;x)|
-\mathbb{E}_{x\sim\zeta}|p_{\rho_\infty}(t_0;x)|^2<0
\end{aligned}
\]
for $\delta$ sufficiently small. Hence $h_\infty(\theta_0)<0$ for some $\theta_0\neq0$, since $h_\infty(0)=0$.}
Since $f$ is $2$-homogeneous, by~\eqref{eqn:Frechetderivative}, $h_\infty$ is also $2$-homogeneous, so that
\[
h_\infty(\theta_0/|\theta_0|)<0\,.
\]
Because of continuity, there is a small neighborhood around $\theta_0/|\theta_0|$ in $\mathbb{S}^{k-1}$ where $h_\infty$ is strictly negative.

\medskip
\noindent{\emph{Step 2: Build a regular negative patch on the sphere.}}\par
\nopagebreak[4]
{
Recall $q_\infty=h_\infty|_{\mathbb S^{k-1}}$ and let $m=\min_{\mathbb S^{k-1}}q_\infty<0$.
If $q_\infty$ is constant, choose a number $-\epsilon$ strictly between $m$ and $0$ and set $A=\mathbb S^{k-1}$.
Then $q_\infty<-\epsilon$ on $A$ and $\partial A=\varnothing$.
If $q_\infty$ is nonconstant, then the
Sard-type condition forces $k\geq2$, since $\mathbb S^0$ has no attained regular values.  Hence $\mathbb S^{k-1}$ is connected and $q_\infty(\mathbb S^{k-1})$ is an interval.
We may therefore choose a negative regular value $-\epsilon$ satisfying $m<-\epsilon<0$, and set
\[
A=\left\{\widetilde{\theta}\in\mathbb S^{k-1}:
q_\infty(\widetilde{\theta})<-\epsilon\right\}.
\]
Then $A\neq\varnothing$ and $\partial A=q_\infty^{-1}(-\epsilon)$ is a compact regular level set.
In both cases, define
\[
\eta=\begin{cases}
\displaystyle\min_{\widetilde{\theta}\in\partial A}
\left|\nabla_{\mathbb{S}^{k-1}}q_\infty(\widetilde{\theta})\right|,
&\partial A\neq\varnothing,\\
1,&\partial A=\varnothing.
\end{cases}
\]
Thus $\eta>0$ by the Sard-type condition and the compactness of $\partial A$; in the constant case, all boundary assertions below are vacuous.
If ${n}_{\widetilde{\theta}}$ denotes the outer normal to $\partial A$ within the sphere, then
\(
\nabla_{\mathbb{S}^{k-1}}q_\infty(\widetilde{\theta})\cdot {n}_{\widetilde{\theta}}\geq\eta\,,
\)
for all
\(
\widetilde{\theta}\in\partial A\,.
\)
}

{
This statement for $h_\infty$ extends to $h_s$ for sufficiently large $s$.
Since
\(
d_1(\rho(s),\rho_\infty)
\to 0
\)
and \(t\mapsto\rho_\infty(\cdot,t)\) is \(W_2\)-continuous on \([0,1]\), the second moments
\(
\mathcal L_s=\int_0^1\int_{\mathbb R^k}|\theta|^2\rho(\rd\theta,t,s)\,\rd t
\)
are uniformly bounded for all sufficiently large \(s\). Therefore, applying~\eqref{eqn:Deltabound2update} with \(\rho_1=\rho(s)\) and \(\rho_2=\rho_\infty\) gives, for every \(R>0\),
\[
\sup_{|\theta|\leq R}
|\nabla h_s(\theta)-\nabla h_\infty(\theta)|
\leq C_Rd_1(\rho(s),\rho_\infty)\to 0,
\]
so \(\nabla h_s\to\nabla h_\infty\) locally uniformly.
Moreover, 2-homogeneity gives $h_s(0)=h_\infty(0)=0$, and hence
\[
h_s(\theta)-h_\infty(\theta)
=\int_0^1\theta^\top\big(\nabla h_s(r\theta)-\nabla h_\infty(r\theta)\big)\rd r.
\]
Thus the functions also converge locally uniformly. Therefore
\[
h_s\to h_\infty\quad\mathrm{in}\quad\mathcal{C}^1_{\mathrm{loc}}(\mathbb{R}^k)\quad\mathrm{as}\quad s\to\infty\,.
\]
Consequently, there exists $S>0$ such that for any $s\geq S$, we have
\[
\begin{cases}
h_{s}|_{\mathbb{S}^{k-1}}\left(\widetilde{\theta}\right) <-\epsilon/2,\quad & \forall \,\widetilde{\theta}\in A\,,\\
\nabla_{\mathbb{S}^{k-1}}h_{s}|_{\mathbb{S}^{k-1}}\left(\widetilde{\theta}\right)\cdot {n}_{\widetilde{\theta}} >\tfrac12 \eta\,,\quad & \forall \,\widetilde{\theta}\in \partial A\,.
\end{cases}
\]
}



\rev{Thanks to the homogeneous property, we can extend $A$ on unit sphere into a cone-structure set.} Extending this patch on the unit sphere to the whole domain (See Figure~\ref{fig:2-homo_proof}), we define the cone:
\[
\mathcal{A}=\left\{\theta\in\mathbb{R}^k\middle||\theta|>0,\,\theta/|\theta|\in A\right\}\,.
\]
Using the  2-homogeneous property of $h_s$, we have for $s\geq S$ that
\begin{equation}\label{hsnegative}
\begin{cases}
h_s(\theta)<-\frac{\epsilon|\theta|^2}{2},\quad & \forall \theta\in\mathcal{A}\,,\\
\nabla_\theta h_s(\theta)\cdot n_{\theta}>0\,,\quad & \forall \theta\in \partial \mathcal{A} \backslash \{\vec{0}\} \,,
\end{cases}
\end{equation}
where $n_{\theta}$ is the outer normal vector to $\partial \mathcal{A}$.


\begin{figure}[htbp]
  \centering
  \scalebox{0.85}{%
  \begin{tikzpicture}[
  x=1cm,
  y=1cm,
  >=Latex,
  line cap=round,
  line join=round,
  every node/.style={font=\small}
]
  \coordinate (O) at (0,0);
  \coordinate (Upper) at (23:8.30);
  \coordinate (Lower) at (-23:8.30);

  \path[fill=ForestGreen!9]
    (O) -- (Upper) -- (Lower) -- cycle;

  \draw[gray!42,thin,dashed]
    (23:1.55) arc[start angle=23,end angle=-23,radius=1.55];
  \draw[gray!42,thin,dashed]
    (23:3.65) arc[start angle=23,end angle=-23,radius=3.65];
  \draw[gray!42,thin,dashed]
    (23:6.25) arc[start angle=23,end angle=-23,radius=6.25];

  \draw[ForestGreen!80!black,very thick] (O) -- (Upper)
    node[pos=.82,above left=1pt] {$\partial\mathcal A$};
  \draw[ForestGreen!80!black,very thick] (O) -- (Lower)
    node[pos=.82,below left=1pt] {$\partial\mathcal A$};
  \node[ForestGreen!70!black] at (5.25,0) {$\mathcal A$};

  \coordinate (Bu) at (4.55,1.90);
  \coordinate (Bl) at (4.55,-1.90);
  \fill[ForestGreen!80!black] (Bu) circle (1.15pt);
  \fill[ForestGreen!80!black] (Bl) circle (1.15pt);
  \draw[-{Latex[length=3mm,width=2mm]},RoyalBlue,very thick]
    (Bu) -- ++(0.24,-1.17)
    node[pos=.55,right=4pt,fill=white,inner sep=1.2pt]
      {$-\nabla h_s$};
  \draw[-{Latex[length=3mm,width=2mm]},RoyalBlue,very thick]
    (Bl) -- ++(0.24,1.17)
    node[pos=.57,right=4pt,fill=white,inner sep=1.2pt]
      {$-\nabla h_s$};

  \coordinate (Alpha) at (1.67,-0.25);
  \fill[BrickRed] (Alpha) circle (2.05pt);
  \node[BrickRed,above left=1pt] at (Alpha) {$\alpha$};
  \draw[BrickRed,very thick,
        -{Latex[length=3.2mm,width=2.2mm]}]
    (Alpha)
      .. controls (2.55,-0.46) and (3.25,-0.25) .. (4.02,0.04)
      .. controls (5.00,0.42) and (5.86,0.78) .. (6.93,1.05);
  \node[BrickRed,above right=1pt,align=left]
    at (6.2,0.1)
    {$\widehat\theta(s;\alpha)$};


  \fill[black] (O) circle (1.2pt);
  \node[below left=1pt] at (O) {$0$};
\end{tikzpicture}
  }
  \caption{A two-dimensional cross-section of the invariant cone in the
  $2$-homogeneous proof.  Along the lateral boundary,
  $-\nabla h_s$ points into $\mathcal A$, while
  \eqref{eqn:fact2} gives outward radial growth.  Thus a characteristic
  starting from $\alpha\in\mathcal A$ remains in the cone and escapes
  to infinity.}
  \label{fig:2-homo_proof}
\end{figure}

\medskip
\noindent{\emph{Step 3: Prove invariance and radial growth.}}\par
\nopagebreak[4]
~\rev{Next, following the discussion at the beginning of this section, we now focus on the trajectory of $\theta(s)$ within the set $\mathcal{A}$ and show that its second moment blows up. To this end, we introduce a new system governed by the gradient flow of $\rho_s$. Denote by $\widehat{\theta}(s;\alpha)$ the solution to the following ODE:}
\begin{equation}\label{eqn:ODE_theta_hat}
\left\{
\begin{aligned}
\frac{\rd \widehat{\theta}\left(s;\alpha\right)}{\rd s} & =-\nabla_\theta\frac{\delta E(\rho(s))}{\delta \rho}\left(\widehat{\theta}\left(s;\alpha\right),t_0\right)=-\nabla_{\theta}h_s\left(\widehat{\theta}\left(s;\alpha\right)\right),\quad s>S\\
\widehat{\theta}\left(S;\alpha\right) & =\alpha,
\end{aligned}
\right.
\end{equation}
where $\alpha\in\mathbb{R}^k$.~\rev{We now claim two properties about $\mathcal{A}$:
\begin{itemize}
\item  If the ODE starts from some $\alpha\in\mathcal{A}$, then for any $s\geq S$, the particle stays within $\mathcal{A}$, that is,
\begin{equation}\label{eqn:fact1}
\widehat{\theta}\left(s;\alpha\right)\in\mathcal{A}\,.
\end{equation}
\item If $\alpha\in\mathcal{A}$, $|\widehat{\theta}(s;\alpha)|^2$ increases exponentially in $s$ for $s>S$.
\end{itemize}}
\rev{The first property guarantees that the measure of $\mathcal{A}$ under $\rho$ does not decrease in $s$ when $s>S$.} It can be proved using the property of $h_s$. First, according to~\eqref{hsnegative}, when $\theta\in \partial \mathcal{A}\cap \{|\theta|>0\}$, $\nabla_\theta h_s(\theta)$ points outwards, away from $\mathcal{A}$. Thus, $\widehat{\theta}\left(s;\alpha\right)$ can not go outside $\mathcal{A}$ through a boundary point that is not $\vec{0}$. Furthermore, also notice that $\widehat{\theta}\left(s;\vec{0}\right)=\vec{0}$. Combining this fact with the uniqueness of the solution to \eqref{eqn:ODE_theta_hat} (classical ODE theory), we get $\widehat{\theta}\left(s;\alpha\right)\neq \vec{0}$ for all $s$ if $\alpha\neq \vec{0}$. As a consequence, we conclude that if $\alpha\in\mathcal{A}$ then for any $s>S$, $\widehat{\theta}\left(s;\alpha\right)\in\mathcal{A}$  and is trapped in $\mathcal{A}$ if the starting point $\alpha\in\mathcal{A}$.

\rev{The second property is a direct result of the following inequality:}
\begin{equation}\label{eqn:fact2}
\frac{\rd \left|\widehat{\theta}\left(s;\alpha\right)\right|^2}{\rd s}
=-2\left\langle \widehat{\theta}\left(s;\alpha\right),\nabla_{\theta}h_s\left(\widehat{\theta}\left(s;\alpha\right)\right)\right\rangle
=-4h_s\left(\widehat{\theta}\left(s;\alpha\right)\right)>2\epsilon\left|\widehat{\theta}\left(s;\alpha\right)\right|^2\,,
\end{equation}
where we use the 2-homogeneous property of $h_s$ in the second equality and \eqref{hsnegative} in the final inequality.

\medskip
\noindent{\emph{Step 4: Seed positive mass in the cone and derive the contradiction.}}\par
\nopagebreak[4]
{
According to Lemma~\ref{lem:seperateionpreserve}, there are $0<r_a'<r_b'$ such that $\mathrm{supp}_\theta(\rho(\theta,t_0,S))$ separates $r_a'\mathbb{S}^{k-1}$ and $r_b'\mathbb{S}^{k-1}$. Choose any $\widetilde{\theta}\in A$. The radial path joining $r_a'\widetilde{\theta}$ to $r_b'\widetilde{\theta}$ therefore meets the support at a point of $\mathcal{A}$. Since $\mathcal A$ is open
, we may choose $0<\beta<r_a'$ and $\gamma>0$ such that
\begin{equation}\label{measurelowerbound}
\int_{\mathcal{A}\cap\left(\mathcal{B}_{\beta}\left(\vec{0}\right)\right)^c}\rho(\rd\theta,t_0,S)>\gamma\,.
\end{equation}
For every $\alpha\in\mathcal{A}$, the invariance~\eqref{eqn:fact1} and~\eqref{eqn:fact2} give
\[
|\widehat{\theta}(s;\alpha)|^2
\geq \exp(2\epsilon(s-S))|\alpha|^2\,,\qquad s\geq S\,.
\]
For the characteristic flow in~\eqref{eqn:ODE_theta_hat},
\(
\rho(\cdot,t_0,s)=(\widehat\theta(s;\cdot))_\#\rho(\cdot,t_0,S)
\)
, for all
\(
s\geq S.
\)
Using this characteristic representation and~\eqref{measurelowerbound}, we conclude that
\[
\begin{aligned}
\int_{\mathbb{R}^k}|\theta|^2\rho(\rd\theta,t_0,s)
&=\int_{\mathbb{R}^k}|\widehat{\theta}(s;\alpha)|^2\rho(\rd\alpha,t_0,S)\\
&\geq\int_{\mathcal{A}\cap(\mathcal{B}_\beta(\vec{0}))^c}
|\widehat{\theta}(s;\alpha)|^2\rho(\rd\alpha,t_0,S)\\
&\geq\gamma\beta^2\exp(2\epsilon(s-S))\to\infty\,.
\end{aligned}
\]
This finishes the proof.}

\end{proof}

\subsubsection{Proof of Proposition \ref{prop:1-homogeneous}}\label{sec:proofofP1-h}

\rev{In this section, we establish the analogous proposition for the partially $1$-homogeneous case.
}

{
In the following, we write $a=\theta_{[1]}$, $w=\theta_{[2]}$, and
\[
h_s(w)=\left.\frac{\delta E}{\delta\rho}\right|_{\rho(s)}
((1,w),t_0) \,,
\qquad
h_\infty(w)=\left.\frac{\delta E}{\delta\rho}\right|_{\rho_\infty}
((1,w),t_0)\,.
\]
Partial $1$-homogeneity gives
\begin{equation}\label{eqn:1h-characteristics-ppp}
\left.\frac{\delta E}{\delta\rho}\right|_{\rho(s)}((a,w),t_0)
=a h_s(w),
\qquad
(\dot a,\dot w)=\left(-h_s(w),-a\nabla h_s(w)\right).
\end{equation}
We note that by Assumption~\ref{assum:f2}, using \eqref{stabilityofxsolution} and \eqref{stabilityofprhoupdate},
\begin{equation}\label{eqn:hs_long_time_bound}
\|h_s-h_\infty\|_{L^\infty(\mathbb R^{k-1})}
+
\|\nabla h_s-\nabla h_\infty\|_{L^\infty(\mathbb R^{k-1})}
\le
C\,d_1(\rho(s),\rho_\infty)
\to 0\,,
\qquad
\text{as }
s\to\infty.
\end{equation}

We first record the preservation of separation ingredient similar to the 2-homogeneous case.

\begin{lemma}
\label{lem:seperateionpreserve2}
Under the assumptions of Theorem~\ref{thm:globalminimal2}, if
$\mathrm{supp}_\theta(\rho_{\ini}(\theta,t_0))$ separates
$\{-r_0\}\times\mathbb R^{k-1}$ and
$\{r_0\}\times\mathbb R^{k-1}$, then for every finite $s_0\geq0$,
$\mathrm{supp}_\theta(\rho(\theta,t_0,s_0))$ separates
$\{-r_0'\}\times\mathbb R^{k-1}$ and
$\{r_0'\}\times\mathbb R^{k-1}$ for some finite $r_0'>0$.
\end{lemma}

\begin{proof}[Lemma~\ref{lem:seperateionpreserve2}]
We follow the same characteristic-flow and pullback strategy as in the proof of
Lemma~\ref{lem:seperateionpreserve}.
For $0\leq s\leq s_0$, let $\mathcal P_s(a,w)$ be the position at time $s$ of the characteristic~\eqref{eqn:1h-characteristics-ppp} starting from $(a,w)$ at time $0$.
The finite-horizon state and adjoint bounds and continuity, Assumption~\ref{assum:f2}, and Lemma~\ref{lemmadeltaprhoupdate} give uniform bounds on $h_s$, $\nabla h_s$, and the Lipschitz constant of $\nabla h_s$ for $s\in[0,s_0]$. Thus the well-posedness of $\mathcal P_s(a,w)$ and $\mathcal P_s^{-1}(a,w)$ follows.

It remains to replace the radial estimate used there by an additive
estimate in the first coordinate.  Put
\[
 B_s=\int_0^s\|h_u\|_{L^\infty(\mathbb R^{k-1})}\,\rd u<\infty.
\]
Since $\dot a(u)=-h_u(w(u))$, with first-coordinate projection $\pi_1(a,w)\coloneqq a$, Gronwall estimates give
\begin{equation}\label{eqn:1h-first-coordinate-displacement}
 \big|\pi_1(\mathcal P_s^{-1}(a,w))-a\big|\leq B_s,
 \qquad 0\leq s\leq s_0.
\end{equation}
For the fixed time $s_0$, set $r_{s_0}=r_0+B_{s_0}$.  If the support
at time $s_0$ did not separate the hyperplanes at $\pm r_{s_0}$,
there would be a continuous path $\gamma$ joining them while avoiding
that support.  Its pullback
$\widetilde\gamma=\mathcal P_{s_0}^{-1}\circ\gamma$ would avoid the
initial support, and~\eqref{eqn:1h-first-coordinate-displacement}
would give, after orienting $\gamma$ from the negative hyperplane to
the positive one,
\[
 \pi_1(\widetilde\gamma(0))\leq-r_{s_0}+B_{s_0}=-r_0,
 \qquad
 \pi_1(\widetilde\gamma(1))\geq r_{s_0}-B_{s_0}=r_0.
\]
As in the proof of Lemma~\ref{lem:seperateionpreserve}, continuity then
provides a subpath joining the two original hyperplanes while avoiding
the initial support, a contradiction.
\end{proof}

\begin{proof}[Proposition~\ref{prop:1-homogeneous}]

Similar to the proof of the 2-homogeneous case, we suppose, for a contradiction, that $h_\infty\not\equiv0$.


If $h_\infty$ is a constant, set $c=h_\infty$.
If $h_\infty$ is nonconstant, take the value $c$ in~\ref{itm:ER3}.
We present the argument when \(c<0\), writing \(c=-\eta\) with \(\eta>0\). The case \(c>0\) is identical after replacing the negative sublevel set and positive-\(a\) escape tube by the positive superlevel set
$
A^{+}=\{w:h_\infty(w)>c\}
$
and the negative-\(a\) tube
$
(-\infty,0)\times A^{+}.
$






\medskip
\noindent\emph{Step 1: Construct the invariant tube.}\par
\nopagebreak[4]

If $h_\infty\equiv-\eta$, uniform convergence in~\eqref{eqn:hs_long_time_bound} gives $S_0<\infty$ such
that
\begin{equation}\label{eqn:1h_constant_hs}
 h_s\leq-\eta/2\qquad\hbox{on }\mathbb R^{k-1},\quad s\geq S_0.
\end{equation}
In this case, we set $A=\mathbb R^{k-1}$.

Suppose now that $h_\infty$ is nonconstant and set
\[
 A=\{w\in\mathbb R^{k-1}:h_\infty(w)<-\eta\}.
\]
Since $-\eta=c$ is an attained regular value, we have
\(
 A\neq\varnothing,
\)
\(
 \partial A=h_\infty^{-1}(c),
\)
and the outer normal
\(
 n_A=\frac{\nabla h_\infty}{|\nabla h_\infty|} .
\)
We claim the following lower bound on $A$:
\begin{equation}\label{eqn:1h-weighted-boundary}
 \beta:=\inf_{w\in\partial A}
 (1+|w|)|\nabla h_\infty(w)|>0.
\end{equation}
Indeed, under~\ref{itm:ER3a}, this follows from compactness of $\partial A$ and regularity.
Under~\ref{itm:ER3b}, suppose instead that the infimum were zero and choose a corresponding sequence $w_n\in\partial A$.  A bounded subsequence would converge to a point of $\partial A$ where $\nabla h_\infty$ vanishes, again contradicting regularity.  Thus, after passing to a subsequence, write $w_n=r_n\varphi_n$, where $r_n\to\infty$ and $\varphi_n\to\varphi$.
Condition~\ref{itm:ER1} then gives
\[
 H_\infty(\varphi)=-\eta,\qquad
 G_{\infty}(\varphi)=0.
\]
Moreover, since
\(
 \nabla_{\mathbb S^{k-2}}\!\left[h_\infty(r\,\cdot)\right](\varphi)
 =rP_{\varphi^\perp}\nabla h_\infty(r\varphi)
\),
we get from condition~\ref{itm:ER1} that $h_\infty(r\,\cdot)\to H_\infty$ in $\mathcal C^1$ on the sphere. Thus, we have
\[
 \nabla_{\mathbb S^{k-2}}H_\infty(\varphi)
 =P_{\varphi^\perp}G_{\infty}(\varphi)=0,
\]
contradicting regularity of $-\eta$ for $H_\infty$ in~\ref{itm:ER3b}.
This proves \eqref{eqn:1h-weighted-boundary}.

By the uniform convergence in~\eqref{eqn:hs_long_time_bound}, after increasing $S_0$, we have
\begin{equation}\label{eqn:1h-tube-estimates}
 h_s\leq-\eta/2\quad\hbox{on } \overline{A}
 \qquad
 \nabla h_s(w)\cdot n_A(w)
 \geq\frac{\beta}{2(1+|w|)}>0
 \quad\hbox{on }\partial A, \quad s\geq S_0.
\end{equation}
The second inequality follows from the compactness of $\partial A$ in the case of~\ref{itm:ER3a}, and follows directly from~\ref{itm:ER2} in the case of~\ref{itm:ER3b}.

In both the constant and nonconstant cases, define the open tube
\[
 \mathcal A=(0,\infty)\times A.
\]
It is forward invariant under~\eqref{eqn:1h-characteristics-ppp}.  Indeed,
on its lateral boundary,
\[
 \dot w\cdot n_A
 =-a\,\nabla h_s\cdot n_A<0,
\]
while on $\{0\}\times A$ one has $\dot{a}=-h_s>0$.  The constant case has no lateral boundary. Thus, a similar argument as the proof of Proposition~\ref{prop:2-homogeneous} shows that $\mathcal A$ is invariant for the flow.

Finally, by our choice of $A$ and $S_0$, every characteristic in $\mathcal A$ satisfies
\begin{equation}\label{eqn:1h-tube-thrust}
 \dot{a}\geq\eta/2,\qquad s\geq S_0.
\end{equation}

Figure~\ref{fig:1-homo_proof} summarizes this invariant-tube geometry.


\begin{figure}[htbp]
  \centering
  \scalebox{0.85}{%
  \begin{tikzpicture}[
  x=1.05cm,
  y=0.90cm,
  line cap=round,
  line join=round,
  >={Stealth[length=2.1mm,width=1.5mm]},
  every node/.style={font=\small}
]
  \def\wleft{-1.45}
  \def\wright{1.45}
  \def\atop{4.62}

  \fill[ForestGreen!10] (\wleft,0.025) rectangle (\wright,\atop);
  \draw[-{Stealth}, ForestGreen!75!black, very thick]
    (\wleft,0) -- (\wleft,\atop);
  \draw[-{Stealth}, ForestGreen!75!black, very thick]
    (\wright,0) -- (\wright,\atop);
  \draw[ForestGreen!75!black, very thick]
    (\wleft,0) -- (\wright,0);
  \node[ForestGreen!45!black, fill=white, fill opacity=0.88,
        text opacity=1, inner sep=2.5pt]
    at (-0.8,3.55) {$\mathcal A$};

  \draw[decorate,
        decoration={brace,mirror,amplitude=4pt,raise=3pt},
        ForestGreen!70!black, semithick]
    (\wleft,0) -- (\wright,0)
    node[midway,below=10pt,ForestGreen!45!black] {$A$};

  \draw[-{Stealth}, BrickRed, very thick]
    (1,2.7) -- (1,3.5);
  \node[BrickRed, anchor=west, fill=white, inner sep=1.5pt]
    at (1.7,3) {$\dot{a}=-h_s(w)\geq \eta/2>0$ for $s\geq S_0$};

  \draw[-{Stealth}, RoyalBlue, very thick]
    (-1.88,1.65) -- (-1.12,1.65);
  \draw[-{Stealth}, RoyalBlue, very thick]
    (1.88,1.65) -- (1.12,1.65);
  \node[RoyalBlue!80!black, fill=white, inner sep=1.5pt]
    at (3,1.43) {$\dot w\cdot n_A<0$};

  \fill[RoyalBlue!85!black] (-0.88,0.25) circle (1.7pt);
  \draw[-{Stealth}, RoyalBlue!85!black, ultra thick]
    (-0.88,0.25)
      .. controls (-0.94,0.90) and (0.63,1.02) .. (0.52,1.78)
      .. controls (0.43,2.45) and (-0.48,2.75) .. (-0.25,3.30)
      .. controls (-0.04,3.79) and (0.16,4.15) .. (0.16,4.55);

  \draw[-{Stealth}, semithick, black!75]
    (-2.05,0) -- (2.18,0) node[below left=-1pt] {$w$};
  \draw[-{Stealth}, semithick, black!75]
    (-2,-0.1) -- (-2,5) node[left=2pt] {$a$};

\end{tikzpicture}
  }
  \caption{The invariant-tube argument for the partially
  $1$-homogeneous flow~\eqref{eqn:1h-characteristics-ppp}.  The vertical
  thrust $\dot{a}=-h_s(w)>0$ and the inward lateral velocity from
  \eqref{eqn:1h-tube-estimates} keep the characteristic in
  $\mathcal A=(0,\infty)\times A$ while $a$ grows. }
  \label{fig:1-homo_proof}
\end{figure}

\medskip
\noindent\emph{Step 2: Force the second moment to diverge.}\par
\nopagebreak[4]
Similar to \citep[Proposition C.4]{NEURIPS2018_a1afc58c}, we claim that there exists $S_1\geq S_0$ such that \(\mathrm{supp}_\theta(\rho(\cdot,t_0,S_1))\cap\mathcal A\neq\varnothing\) (see detailed proof in Appendix \ref{sec:proofofclaim})
. Hence, we can find \(z_1\in\mathrm{supp}_\theta(\rho(\cdot,t_0,S_1))\cap\mathcal A\) and an open neighborhood $U\subset\mathcal A$ of $z_1$ such that
\begin{equation}\label{eqn:positive-escape-mass}
 \rho(U,t_0,S_1)=:\gamma>0,\qquad
 a_\ast:=\inf_{(a,w)\in U} a>0.
\end{equation}
Forward invariance of the tube and~\eqref{eqn:1h-tube-thrust} give, along every characteristic starting in $U$,
\[
 a(s)\geq a_\ast+\frac{\eta}{2}(s-S_1).
\]
Following the characteristic flow, it yields
\[
 \int_{\mathbb R^k}|\theta|^2\rho(\rd\theta,t_0,s)
 \geq
 \gamma\left(a_\ast+\frac{\eta}{2}(s-S_1)\right)^2
 \to\infty,
\]
contrary to the assumed uniform second-moment bound.
Therefore $h_\infty\equiv0$.  Partial homogeneity finally gives
\[
 \left.\frac{\delta E}{\delta\rho}\right|_{\rho_\infty}
 ((a,w),t_0)=a h_\infty(w)=0,
\]
which is~\eqref{eqn:notsatisfy1}.
\end{proof}
}

{
\subsection{Claim in the Proof of Proposition \ref{prop:1-homogeneous}}
\label{sec:proofofclaim}

We show in this step that the support characteristic enters the tube $\mathcal{A}$ after finite time, that is, for some \(S_1\geq S_0\),
\[
\mathrm{supp}_\theta\left(\rho(\theta,t_0,S_1)\right)\cap\mathcal A\neq\varnothing.
\]
Since \(\mathcal A\) is open, this will imply that a positive amount of mass has entered the invariant tube, that is, $\rho(\cdot,t_0,S_1)(\mathcal A)>0$.

The proof strategy can be summarized as follows. Fix \(S\geq S_0\).
By Lemma~\ref{lem:seperateionpreserve2}, there exists \(r_S>0\) such that, for every \(w_*\in\mathbb R^{k-1}\), the vertical segment $[-r_S,r_S]\times\{w_*\}$ intersects \(\mathrm{supp}_\theta(\rho(\theta,t_0,S))\) at $(a_1,w_*)$ for some \(a_1\in[-r_S,r_S]\).
In each of the cases below, we choose \(w_*\) according to the geometry of \(A\) and show that the characteristic starting from \((a_1,w_*)\) enters the escape tube \(\mathcal A\) in finite time.


If $h_\infty$ is constant, fix $S\geq S_0$ and any $w\in\mathbb{R}^{k-1}$.
Along the characteristic, $\dot{a}\geq\eta/2$ globally by \eqref{eqn:1h_constant_hs}, so it reaches $a>0$ in finite time and enters $\mathcal A$.

Suppose henceforth that $h_\infty$ is nonconstant.
We treat conditions~\ref{itm:ER3a} and~\ref{itm:ER3b} separately. Condition~\ref{itm:ER3a} requires \(\partial A\) to be bounded, but does not require the sublevel set \(A\) itself to be bounded. We therefore divide this case further according to whether \(A\) has a bounded connected component or all of its connected components are unbounded.

\medskip
\noindent\emph{Case I: \ref{itm:ER3a} and $A$ has a bounded component.}\par
\nopagebreak[4]

Suppose now that \(h_\infty\) is nonconstant and that \ref{itm:ER3a} holds. Thus
$
\partial A=h_\infty^{-1}(-\eta)
$
is bounded. Choose a connected component \(D\) of \(A\) such that \(D\) is \emph{bounded}. Since \(h_\infty=-\eta\) on \(\partial D\) and \(h_\infty<-\eta\) in \(D\), \(h_\infty\) attains its minimum over \(\overline D\) at an interior point \(w_\ast\). Write
$$
h_\infty(w_\ast)=-m,\qquad m>\eta.
$$
Then \(\nabla h_\infty(w_\ast)=0\). Choose \(r>0\) such that
$$
\overline{B_r(w_\ast)}\subset D,
\qquad
h_\infty\leq-\frac{3m}{4}
\quad\text{on }B_r(w_\ast).
$$
We record other quantities that will be used below. By Assumption~\ref{assum:f2}, after increasing \(\overline L\geq1\) if necessary,
\begin{equation}\label{eqn:1h_grad_h_infty_lipschitz}
|\nabla h_\infty(w)|
\leq\overline L|w-w_\ast|,
\qquad w\in B_r(w_\ast).
\end{equation}
Fix the time \(S_0\) in~\eqref{eqn:1h-tube-estimates}. The assumed initial bound on the first coordinate implies that every point of \(\mathrm{supp}_\theta\left(\rho(\theta,t_0,S_0)\right)\) satisfies
$$
|a|\leq
M:=R'+\int_0^{S_0}\|h_u\|_{L^\infty}\,\rd u<\infty.
$$
Choose \(\delta_M>0\) so that
$$
\delta_M\leq\frac m4,
\qquad
\frac{\delta_M}{\overline L}
\exp\!\left(\frac{\overline L M^2}{m}\right)
<\frac r2,
$$
and then choose \(S\geq S_0\) sufficiently large that, for every
\(s\geq S\),
\begin{equation}\label{eqn:1h_h_converge_bound}
\sup_{w\in B_r(w_\ast)}
\left(
|h_s(w)-h_\infty(w)|
+
|\nabla h_s(w)-\nabla h_\infty(w)|
\right)
\leq\delta_M.
\end{equation}

Separation at time \(S\), applied at \(w_\ast\), gives
$$
(a_1,w_\ast)\in \mathrm{supp}_\theta\left(\rho(\theta,t_0,S)\right)
$$
for some \(a_1\in\mathbb R\). We claim that $a_1\geq-M$, under our choice of $S$.
There is nothing to prove if \(a_1\geq0\). If \(a_1<0\), trace the corresponding characteristic backward from \(S\) to \(S_0\).
This backward characteristic remains in $(-\infty,0)\times D$.
Indeed, on \(\{0\}\times D\), the forward velocity satisfies
$$
\dot a=-h_s(w)\geq\frac{\eta}{2}>0,
$$
and, on the lateral boundary
\((-\infty,0)\times\partial D\),
we have
$
\dot w\cdot n_D
=-a\,\nabla h_s(w)\cdot n_D>0,
$
where \(n_D\) is the outward normal to \(D\).
Thus the forward vector field points out of \((-\infty,0)\times D\), or equivalently the backward vector field points into it. Along this characteristic, \(\dot a\geq\eta/2\) in forward time.
Its ancestor at time \(S_0\) belongs to \(\mathrm{supp}_\theta\left(\rho(\theta,t_0,S_0)\right)\), and hence
$
a_1=a(S)\geq a(S_0)\geq-M.
$

We claim that the characteristic starting from $(a_1,w_\ast)$ enters $\mathcal{A}$ in finite time. If \(a_1>0\), then \((a_1,w_\ast)\in\mathcal A\), and the claim is proved.
If \(a_1=0\), then \(\dot a>0\) and \(\dot w=0\), so the characteristic
enters \(\mathcal A\) immediately. It remains to consider
$
-M\leq a_1<0.
$
Let \((a(s),w(s))\) be the characteristic starting from
\((a_1,w_\ast)\) at time \(S\), and set
$$
q(s):=|w(s)-w_\ast|.
$$
Stop the characteristic when either \(a\) first reaches \(0\) or \(q\) first reaches \(r\). Before this stopping time, the choice of $\delta_M$ and $r$ gives
\begin{equation}\label{eqn:1h_a_neg_dot_a_lower}
\dot a=-h_s(w)\geq\frac m2
\end{equation}
and by~\eqref{eqn:1h_grad_h_infty_lipschitz} and~\eqref{eqn:1h_h_converge_bound}, almost everywhere,
$$
\dot q
\leq|\dot w|
=(-a)|\nabla h_s(w)|
\leq(-a)(\overline Lq+\delta_M).
$$
A Grönwall estimate then shows that
\begin{equation}\label{eqn:1h-bounded-crossing}
q(s)
\leq
\frac{\delta_M}{\overline L}
\left[
\exp\!\left(\frac{\overline L M^2}{m}\right)-1
\right]
<\frac r2.
\end{equation}
Thus the first stopping event cannot be \(q=r\). Moreover, by~\eqref{eqn:1h_a_neg_dot_a_lower} and \(a_1\geq-M\), we obtain that \(a\) reaches \(0\) in at most \(2M/m\) units of time.
It reaches \(a=0\) while
$
w\in B_{r/2}(w_\ast)\subset D.
$
Thus the characteristic enters
$
(1,\infty)\times D\subset\mathcal A
$
at $S_1=S + 2M/m + 2/\eta+1$.

\medskip
\noindent\emph{Case II: \ref{itm:ER3a} and all connected components of $A$ are unbounded.}\par
\nopagebreak[4]
Under the same setting as Case I, but suppose now that \(D\) is \emph{unbounded}.
Since \(\partial D\subset\partial A\) and \(\partial A\) is bounded by
\ref{itm:ER3a}, \(\partial D\) is bounded.
Fix \(S\geq S_0\), and let \(r_S\) be the separation radius at time \(S\). Define
$$
G:=\sup_{s\geq S,w\in\mathbb{R}^{k-1}}|\nabla h_s(w)|<\infty,
\qquad
T_\ast:=\frac{2(r_S+1)}{\eta},
$$
and
$$
B:=\max\{r_S,1\},
\qquad
L:=BGT_\ast.
$$
\(G\) is finite because Assumption~\ref{assum:f2} gives $\|\nabla h_\infty\|_{L^\infty} < \infty$ and~\eqref{eqn:hs_long_time_bound} gives $\sup_{s\geq S} \|\nabla h_s - \nabla h_\infty \|_{L^\infty} < \infty$.
Because \(D\) is unbounded and \(\partial D\) is bounded, we can choose \(w_\ast\in D\) such that
$$
\operatorname{dist}(w_\ast,\partial D)>L.
$$

Separation gives
$$
(a_1,w_\ast)\in \mathrm{supp}_\theta\left(\rho(\theta,t_0,S)\right),
\qquad
a_1\in[-r_S,r_S].
$$
We claim that the characteristic starting from $(a_1,w_\ast)$ enters $\mathcal{A}$ in finite time.
If \(a_1>0\), this point already lies in \(\mathcal A\). Otherwise,
follow its characteristic until it either reaches \(a=1\) or leaves
\(D\).
As long as the characteristic remains in \(D\),
$$
\dot a=-h_s(w)\geq\frac{\eta}{2}.
$$
Hence one of the two stopping events must occur within time \(T_\ast\).
Before that event, $-r_S\leq a(s)\leq1$, and therefore
$$
|w(s)-w_\ast
|
\leq
\int_S^s|a(u)|\,|\nabla h_u(w(u))|\,\rd u
\leq BGT_\ast=L.
$$
The characteristic cannot reach \(\partial D\), because its initial
distance from \(\partial D\) is strictly greater than \(L\). It must
therefore reach \(a=1\) while still in \(D\), and hence it enters
\(\mathcal A\).

\medskip
\noindent\emph{Case III: \ref{itm:ER3b}.}\par
\nopagebreak[4]
Assume that \(h_\infty\) is nonconstant and that
\ref{itm:ER3b} holds. We again show that the characteristic
enters \(\mathcal A\) in finite time by suitable choice of $(a_1,w_\ast)$.

Since $\partial A=h_\infty^{-1}(-\eta)$
is unbounded, there are $w_n=r_n\varphi_n\in\partial A$, and $r_n\to\infty$.
After passing to a subsequence,
\(\varphi_n\to\varphi_0\in\mathbb S^{k-2}\). Since
\(h_\infty(w_n)=-\eta\), the uniform radial convergence in
\ref{itm:ER1} gives
$$
|H_\infty(\varphi_n)+\eta|
=
|H_\infty(\varphi_n)-h_\infty(r_n\varphi_n)|
\leq
\sup_{\varphi\in\mathbb S^{k-2}}
|H_\infty(\varphi)-h_\infty(r_n\varphi)|
\to0.
$$
Therefore,
$$
H_\infty(\varphi_0)=-\eta
\qquad
\nabla_{\mathbb S^{k-2}}H_\infty(\varphi_0)\neq0\,,
$$
where the second equation follows from the choice of $\eta$ and \ref{itm:ER3b}.
Moving slightly from \(\varphi_0\) in the direction of decreasing \(H_\infty\), we choose a nearby direction \(\varphi_\ast\), an open spherical neighborhood \(\mathcal U\) of \(\varphi_\ast\), and \(\delta>0\) such that
$$
H_\infty\leq-\eta-2\delta
\qquad\text{on }\overline{\mathcal U}.
$$

We next record uniform estimates in the corresponding far-field angular region. By \ref{itm:ER1}, there are \(R_0\geq1\) and \(C_0<\infty\) such that, for every \(r\geq R_0\) and
\(\varphi\in\overline{\mathcal U}\),
$$
|h_\infty(r\varphi)-H_\infty(\varphi)|\leq\delta,
\qquad
|\nabla h_\infty(r\varphi)|\leq\frac{C_0}{r}.
$$
By~\eqref{eqn:hs_long_time_bound} and~\ref{itm:ER2}, we can choose \(S\geq S_0\) sufficiently large that
$$
\|h_s-h_\infty\|_{L^\infty}
 +\sup_w(1+|w|)|\nabla h_s(w)-\nabla h_\infty(w)|\leq\min\{\delta,1\},
\qquad s\geq S,
$$
It follows that, whenever \(s\geq S\), \(|w|\geq R_0\), and \(w/|w|\in\overline{\mathcal U}\), we have $h_\infty(w)\leq-\eta-\delta$, and
\begin{equation}\label{eqn:1h-far-field-estimates}
h_s(w)\leq-\eta<-\frac{\eta}{2},
\qquad
|\nabla h_s(w)|\leq\frac{C}{|w|},
\end{equation}
where \(C=C_0+1\). Indeed, the first two inequalities follow from the uniform function limits, while
$$
|\nabla h_s(w)|
\leq
|\nabla h_\infty(w)|
+
|\nabla h_s(w)-\nabla h_\infty(w)|
\leq
\frac{C_0}{|w|}
+
\frac{1}{1+|w|}
\leq\frac{C}{|w|}.
$$
Fix this \(S\), and let \(r_S\) be the corresponding separation
radius. Importantly, \(r_S\) is independent of the point \(w\) to
which separation is applied. Set
$$
B:=\max\{r_S,1\},
\qquad
T_\ast:=\frac{2(r_S+1)}{\eta},
\qquad
C_\ast:=2BCT_\ast.
$$
Also let
$$
d_\ast:=
\operatorname{dist}\!\left(
\varphi_\ast,
\mathbb S^{k-2}\setminus\mathcal U
\right)>0.
$$
Choose \(R\) sufficiently large that
$$
R\geq2R_0,
\qquad
\frac{C_\ast}{R}<\frac R2,
\qquad
\frac{2C_\ast}{R^2}<d_\ast.
$$

Now we apply separation at \(w_\ast=R\varphi_\ast\). It gives a point
$$
(a_1,R\varphi_\ast)
\in\operatorname{supp}_\theta
\bigl(\rho(\theta,t_0,S)\bigr),
\qquad
|a_1|\leq r_S.
$$
We claim that the characteristic starting from $(a_1,w_\ast)$ enters $\mathcal{A}$ in finite time.
Since \(R\varphi_\ast\in A\), this point already belongs to \(\mathcal A\) if \(a_1>0\). Suppose therefore that \(a_1\leq0\), and let \((a(s),w(s))\) be the characteristic starting from this point.
Stop the characteristic at the first time when one of the following events occurs:
$$
a=1,\qquad
|w|=\frac R2,\qquad
\frac{w}{|w|}\notin\mathcal U,\qquad
s=S+T_\ast
.
$$
Before this stopping time, the far-field estimates
\eqref{eqn:1h-far-field-estimates} apply. In particular,
$$
\dot a=-h_s(w)\geq\frac{\eta}{2}.
$$
Thus \(a\) is increasing and, before reaching \(1\), $-r_S\leq a_1\leq a(s)\leq1$.
Moreover, since \(|w|\geq R/2\),
\begin{equation}\label{eqn:1h-large-radius-crossing}
|w(s)-R\varphi_\ast|
\leq
\int_S^s|a(u)|\,|\nabla h_u(w(u))|\,\rd u
\leq
B\frac{2C}{R}T_\ast
=\frac{C_\ast}{R}.
\end{equation}
This estimate rules out both possible spatial exits. Indeed,
$$
|w(s)|
\geq R-|w(s)-R\varphi_\ast|
> \frac R2,
$$
so the radius cannot reach \(R/2\). Furthermore,
$$
\left|
\frac{w(s)}{|w(s)|}-\varphi_\ast
\right|
\leq
\frac{2|w(s)-R\varphi_\ast|}{R}
\leq\frac{2C_\ast}{R^2}
<d_\ast,
$$
so the direction cannot leave \(\mathcal U\).
Finally, the time cutoff cannot occur before $a$ reaches $1$: if the characteristic remained in the far-field region until $S+T_\ast$, then
$$
a(S+T_\ast)
\geq
a_1
+\frac{\eta}{2}T_\ast
\geq
-r_S+(r_S+1)=1.
$$
Therefore the first stopping event is \(a=1\), and the characteristic enters $(1,\infty)\times A \subset\mathcal A$ at $S_1 \leq S+T_\ast$.
}

\section{Generalization Error. Proof for Theorem~\ref{thm:generalization}} \label{sec:generalizationproof}
This section contains the proof for showing the generalization error is small. This is to justify that training the neural network using a finite number of data points well approximates training using the whole underlying distribution. To do so, we consider two different gradient-flow PDEs:
\begin{equation}\label{eqn:Wassgradientflow1}
\frac{\partial \rho}{\partial s}=\nabla_\theta\cdot\left(\rho\nabla_\theta\frac{\delta E_\zeta(\rho(s))}{\delta \rho}\right)\,,\quad s>0,\; t\in[0,1]\,,
\end{equation}
and
\begin{equation}\label{eqn:Wassgradientflow2}
\frac{\partial \rho^{N}}{\partial s}=\nabla_\theta\cdot\left(\rho^{N}\nabla_\theta\frac{\delta E_N(\rho^{N}(s))}{\delta \rho}\right)\,,\quad s>0,\; t\in[0,1]\,,
\end{equation}
with the same training-time initial condition $\rho(\theta,t,0)=\rho^{N}(\theta,t,0)=\rho_\ini(\theta,t)$. We also used the shorthand notation $E_N = E_{\zeta_N}$.

From the previous sections we already have that $E(\rho)$ and $E_{N}(\rho^{N})$ converge to zero as $s\to\infty$. In this section, the aim is to show that when $N\gg 1$, $E(\rho^{N})-E(\rho)\sim\frac{1}{\sqrt{N}}$ for all $s<S$, following the law of large numbers.~\rev{Unlike the mean‐field convergence analysis, this section focuses on the stability of the training process with respect to the loss function $E$ (or dataset). Although the objective differs, the main strategy remains the same: demonstrating the closeness between the training parameters $\theta$ and the network output $Z(t)$.}

Define
\[
\mathcal{L}^{\sup}_\ini=\sup_{0\leq t\leq1}\int_{\mathbb{R}^k}|\theta|^2\rho_\ini(\rd\theta,t)\,.
\]
First, we bound $Z$ and $p$ in the following lemma (see the proof in \cite[Lemma 4.10.1]{Zhiyan_thesis}):
\begin{lemma}\label{lem:gboundzp}
Suppose that Assumptions~\ref{assum:f} and~\ref{ass:g} are satisfied. Assume that $\rho_{\ini}(\theta,t)$ is limit-admissible
and $\mathrm{supp}_\theta(\rho_{\ini}(\theta,t))\subset \{\theta||\theta_{[1]}|\leq R'\}$ with some $R'>0$ for all $t\in[0,1]$.
Then, for fixed $S>0$, there exists a constant $C(S,R',\mathcal{L}^{\sup}_\ini)$ depending only on $S,R',\mathcal{L}^{\sup}_\ini$ such that if $s\in[0,S]$, then
\[
\sup_{|x|<R,t\in[0,1]}\left\{|Z_{\rho(s)}(t;x)|,|Z_{\rho^{N}(s)}(t;x)|,|p_{\rho(s)}(t;x)|,|p_{\rho^{N}(s)}(t;x)|\right\}\leq C(S,R',\mathcal{L}^{\sup}_\ini)
\]
\end{lemma}

Now, we are ready to prove Theorem \ref{thm:generalization}.
\begin{proof}
First, similar to mean-field limit (Appendix \ref{sec:proofofthmmean-field}), we utilize the particle method. Define $\theta(s;t)$, $\widetilde{\theta}(s;t)$ as the solution to
\begin{equation}\label{eqn:gthetast}
    \frac{\rd \theta(s;t)}{\rd s}=-\nabla_{\theta}\frac{\delta E_\zeta(\rho(s))}{\delta \rho}\big(\theta(s;t),t\big)
\end{equation}
and
\begin{equation}\label{eqn:gwthetast}
\frac{\rd \widetilde{\theta}(s;t)}{\rd s}=-\nabla_{\theta}\frac{\delta E_{N}(\rho^{N}(s))}{\delta \rho}\big(\widetilde{\theta}(s;t),t\big)
\end{equation}
for $(t,s)\in[0,1]\times[0,\infty)$ with the same initial condition $\widetilde{\theta}(0;t) = \theta(0;t)$ drawn from $\rho_\ini(\theta,t)$. As a consequence, we have $\theta(s;t)\sim \rho(\theta,t,s)$ and $\widetilde{\theta}(s;t)\sim \rho^{N}(\theta,t,s)$. Note that once $\zeta$ and $\zeta_N$ are fixed, the only randomness in the two systems above comes from the drawing of the initial condition. The systems are deterministic.

Similar to the derivation in \eqref{eqn:differenceofE3}, we also have
\begin{equation}\label{eqn:gdE}
\left|E_\zeta(\rho(\cdot,\cdot,s))-E_\zeta(\rho^{N}(\cdot,\cdot,s))\right|\leq C(S,R',\mathcal{L}^{\sup}_\ini)\sup_{t\in[0,1],|x|<R}\left|Z_{\rho(s)}(t;x)-Z_{\rho^{N}(s)}(t;x)\right|\,.
\end{equation}
Thus, it suffices to bound $\left|Z_{\rho(s)}(t;x)-Z_{\rho^{N}(s)}(t;x)\right|$. According to Theorem \ref{thm:wellposed1}, \eqref{stabilityofxsolution},
\begin{equation}\label{eqn:gdZ}
\begin{aligned}
\sup_{t\in[0,1],|x|<R}\left|Z_{\rho(s)}(t;x)-Z_{\rho^{N}(s)}(t;x)\right|&\leq C(S,R',\mathcal{L}^{\sup}_\ini)\left(\int^1_0W^2_2(\rho(\cdot,t,s),\rho^{N}(\cdot,t,s))\rd t\right)^{1/2}\\
&\leq C(S,R',\mathcal{L}^{\sup}_\ini)\left(\int^1_0\mathbb{E}_{\ini}\left|\theta(s;t)-\widetilde{\theta}(s;t)\right|^2\rd t\right)^{1/2}\,,
\end{aligned}
\end{equation}
where $\mathbb{E}_{\ini}$ takes the expectation over the randomness in drawing the initial condition $\widetilde{\theta}(0;t)=\theta(0;t)$ from $\rho_\ini(\theta,t)$. In the following part of the proof, we focus on controlling $\left|\theta(s;t)-\widetilde{\theta}(s;t)\right|^2$. Define
\[
\Delta_{t}(s)=\theta(s;t)-\widetilde{\theta}(s;t)\,.
\]
We note that $|\Delta_t(0)|=0$. By taking the difference of~\eqref{eqn:gthetast} and~\eqref{eqn:gwthetast} and multiplying both sides by $\Delta_{t}(s)$, we obtain
\begin{equation}\label{eqn:gd}
\begin{aligned}
&\frac{\rd|\Delta_{t}(s)|^2}{\rd s}\\
=&-2\left\langle\Delta_{t}(s),\nabla_{\theta}\frac{\delta E_\zeta(\rho(s))}{\delta \rho}\left(\theta(s;t),t\right)-\nabla_{\theta}\frac{\delta E_{N}(\rho^{N}(s))}{\delta \rho}\left(\widetilde{\theta}(s;t),t\right)\right\rangle\\
=&-2\left\langle\Delta_{t}(s),\underbrace{\nabla_{\theta}\frac{\delta E_\zeta(\rho(s))}{\delta \rho}\left(\theta(s;t),t\right)-\nabla_{\theta}\frac{\delta E_\zeta(\rho^{N}(s))}{\delta \rho}\left(\theta(s;t),t\right)}_{\textrm{(I)}}\right\rangle\\
&-2\left\langle\Delta_{t}(s),\underbrace{\nabla_{\theta}\frac{\delta E_\zeta(\rho^{N}(s))}{\delta \rho}\left(\theta(s;t),t\right)-\nabla_{\theta}\frac{\delta E_\zeta(\rho^{N}(s))}{\delta \rho}\left(\widetilde{\theta}(s;t),t\right)}_{\textrm{(II)}}\right\rangle\\
&-2\left\langle\Delta_{t}(s),\underbrace{\nabla_{\theta}\frac{\delta E_\zeta(\rho^{N}(s))}{\delta \rho}\left(\widetilde{\theta}(s;t),t\right)-\nabla_{\theta}\frac{\delta E_{N}(\rho^{N}(s))}{\delta \rho}\left(\widetilde{\theta}(s;t),t\right)}_{\textrm{(III)}}\right\rangle\,.
\end{aligned}
\end{equation}
The three terms are to be bounded by the continuity of the Fr\'echet derivative $\frac{\delta E_\zeta}{\delta\rho}(\rho)$ on its argument, $\Delta_t(s)$ itself, and the continuous dependence of $E_\zeta$ on $\zeta$.

In particular, to bound (I), we recall the continuity of the Fr\'echet derivative $\frac{\delta E_\zeta}{\delta\rho}(\rho)$ stated in Lemma~\ref{lemmadeltaprhoupdate}, equation~\eqref{eqn:Deltabound2}. Notice
$W^2_2(\rho(\cdot,t,s),\rho^{N}(\cdot,t,s))\leq \mathbb{E}_{\ini}\left|\Delta_{t}(s)\right|^2$ for $t\in[0,1]$, then for all fixed $s$:
\begin{equation}\label{eqn:gI}
|\mathrm{(I)}|\leq C(S,R',\mathcal{L}^{\sup}_\ini)(1+|\theta(s,t)|)\left(\int^1_0\mathbb{E}_{\ini}\left|\Delta_{t}(s)\right|^2\rd t\right)^{1/2}\,.
\end{equation}

The bound for (II) directly comes from Lemma \ref{lemmadeltaprhoupdate}, equation~\eqref{Lipshictz}:
\begin{equation}\label{eqn:gII}
|\mathrm{(II)}|\leq C(S,R',\mathcal{L}^{\sup}_\ini)\left|\Delta_{t}(s)\right|\,.
\end{equation}

Then, we plug \eqref{eqn:gI}, \eqref{eqn:gII} into \eqref{eqn:gd}:
\begin{equation}\label{eqn:gd2}
\begin{aligned}
&\frac{\rd|\Delta_{t}(s)|^2}{\rd s}\\
\leq &C(S,R',\mathcal{L}^{\sup}_\ini)\left|\Delta_{t}(s)\right|^2\\
&+C(S,R',\mathcal{L}^{\sup}_\ini)(1+|\theta(s,t)|)^2\left(\int^1_0\mathbb{E}_{\ini}\left|\Delta_{t}(s)\right|^2\rd t\right)\\
&+C(S,R',\mathcal{L}^{\sup}_\ini)\left|\nabla_{\theta}\frac{\delta E_\zeta(\rho^{N}(s))}{\delta \rho}\left(\widetilde{\theta}(s;t),t\right)-\nabla_{\theta}\frac{\delta E_{N}(\rho^{N}(s))}{\delta \rho}\left(\widetilde{\theta}(s;t),t\right)\right|^2\,.
\end{aligned}
\end{equation}
Noting the second moment bound \cite[Corollary  4.8.13]{Zhiyan_thesis}, we have
\[
\begin{aligned}
&\frac{\rd\int^1_0\mathbb{E}_\ini|\Delta_{t}(s)|^2\rd t}{\rd s}\\
& \leq C(S,R',\mathcal{L}^{\sup}_\ini)\int^1_0\mathbb{E}_{\ini}\left|\Delta_{t}(s)\right|^2\rd t+C(S,R',\mathcal{L}^{\sup}_\ini)\int^1_0\mathbb{E}_\ini\left|(\mathrm{III})\right|^2\rd t\,,
\end{aligned}
\]
which implies, using Gr\"onwall's inequality, that
\[
\begin{aligned}
\int^1_0\mathbb{E}_\ini|\Delta_{t}(s)|^2\rd t \leq C(S,R',\mathcal{L}^{\sup}_\ini) \int^S_0\int^1_0\mathbb{E}_{\ini}\left|(\mathrm{III})\right|^2\rd t\rd s\,,
\end{aligned}
\]
where we use $|\Delta_t(0)|=0$.

To control (III), we recall the definition of $\frac{\delta E_\zeta}{\delta \rho}$ \eqref{eqn:frechet_s}:
\[
\begin{aligned}
|\mathrm{(III)}|\leq &\left|\mathbb{E}_{x\sim\zeta}\left(p^\top_{\rho^{N}(s)}(t;x)\partial_\theta f(Z_{\rho^{N}(s)}(t;x),\tilde\theta(s;t))\right)\right.\\
&-\left.\mathbb{E}_{x\sim{\zeta_N}}\left(p^\top_{\rho^{N}(s)}(t;x)\partial_\theta f(Z_{\rho^{N}(s)}(t;x),\tilde \theta(s;t))\right)\right|\,.
\end{aligned}
\]
Since $Z,p$ are bounded according to Lemma \ref{lem:gboundzp}, using Assumption \ref{assum:f} \eqref{eqn:derivativebound}, we have
\begin{equation}\label{eqn:gIII}
\mathbb{E}_X(|\mathrm{(III)}|^2)\leq \frac{C(S,R',\mathcal{L}^{\sup}_\ini)(1+|\widetilde{\theta}(s;t)|^2)}{N}\,,
\end{equation}
where $\mathbb{E}_X$ takes the expectation on randomness of the finite data set $x_n$ that appears in ${\zeta_N}$. Noting $\mathbb{E}_{\ini}|\widetilde{\theta}(s;t)|^2\leq C$ by \cite[Corollary  4.8.13]{Zhiyan_thesis}, we have
\[
\mathbb{E}_X\left(\int^1_0\mathbb{E}_\ini|\Delta_{t}(s)|^2\rd t\right)\leq \frac{C(S,R',\mathcal{L}^{\sup}_\ini)}{N}\,.
\]
This implies that there exists a constant $C'(S,R',\mathcal{L}^{\sup}_\ini)$ depending only on $S,R',\mathcal{L}^{\sup}_\ini$ such that for any $\epsilon,\eta>0$, when $N>\frac{C'(S,R',\mathcal{L}^{\sup}_\ini)}{\epsilon^2\eta}$, we have
\[
\mathbb{P}\left(\int^1_0\mathbb{E}_\ini|\Delta_{t}(s)|^2\rd t<\epsilon^2\right)>1-\eta\,.
\]
By using this result and \eqref{eqn:gdZ} in conjunction with \eqref{eqn:gdE}, we complete the proof.
\end{proof}


\end{appendix}

\vskip 0.2in
\bibliography{ref}

@article{gao2024global,
  title={Global convergence in training large-scale transformers},
  author={Gao, Cheng and Cao, Yuan and Li, Zihao and He, Yihan and Wang, Mengdi and Liu, Han and Klusowski, Jason M and Fan, Jianqing},
  journal={Advances in Neural Information Processing Systems},
  volume={37},
  pages={29213--29284},
  year={2024}
}

@article{chizat2025hidden,
  title={The Hidden Width of Deep ResNets: Tight Error Bounds and Phase Diagram},
  author={Chizat, L{\'e}na{\"\i}c},
  journal={arXiv preprint arXiv:2509.10167},
  year={2025}
}

@article{isobe2023convergence,
  title={A convergence result of a continuous model of deep learning via {\L{}ojasiewicz--Simon} inequality},
  author={Isobe, Noboru},
  journal={arXiv preprint arXiv:2311.15365},
  year={2023}
}

@article{barboni2025understanding,
  title={Understanding the training of infinitely deep and wide resnets with conditional optimal transport},
  author={Barboni, Rapha{\"e}l and Peyr{\'e}, Gabriel and Vialard, Fran{\c{c}}ois-Xavier},
  journal={Communications on Pure and Applied Mathematics},
  volume={78},
  number={11},
  pages={2149--2205},
  year={2025},
  publisher={Wiley Online Library}
}

@article{petit2026global,
  title={On the global convergence of gradient descent for wide shallow models with bounded nonlinearities},
  author={Petit, Romain and Poon, Clarice and Peyr{\'e}, Gabriel},
  journal={arXiv preprint arXiv:2605.10775},
  year={2026}
}

@inproceedings{zhang2017understanding,
title={Understanding deep learning requires rethinking generalization},
author={C. Zhang and S. Bengio and M. Hardt and B. Recht and O. Vinyals},
booktitle={International Conference on Learning Representations},
year={2017},
url={https://openreview.net/forum?id=Sy8gdB9xx}
}

@article{10.5555/3454287.3454840,
  title={Learning and generalization in overparameterized neural networks, going beyond two layers},
  author={Allen-Zhu, Z. and Li, Y. and Liang, Y.},
  journal={Advances in neural information processing systems},
  volume={32},
  year={2019}
}

@article{Zhiyan_thesis,
  title={Interacting Particles and Mean-field Analysis—{Bayesian} Sampling and Overparameterized Neural Network},
  author={Z. Ding},
  journal={[Doctoral dissertation, The University of Wisconsin - Madison].ProQuest Dissertations Publishing},
  year={2022},
}

@inproceedings{Frei2019AlgorithmDependentGB,
  title={Algorithm-Dependent Generalization Bounds for Overparameterized Deep Residual Networks},
  author={S. Frei and Y. Cao and Q. Gu},
  booktitle = {Advances in Neural Information Processing Systems},
  year={2019}
}

@article{ding2021overparameterization,
      title={Overparameterization of deep ResNet: zero loss and mean-field analysis},
      author={Z. Ding and S. Chen and Q. Li and S. Wright},
      year={2021},
      journal={arXiv/2105.14417},
      primaryClass={cs.LG}
}

@article{chatterji2021does,
  author  = {N. Chatterji and P. Long and P. Bartlett},
  title   = {When Does Gradient Descent with Logistic Loss Find Interpolating Two-Layer Networks?},
  journal = {Journal of Machine Learning Research},
  year    = {2021},
  volume  = {22},
  number  = {159},
  pages   = {1-48},
}

@article{liu2022loss,
  title={Loss landscapes and optimization in over-parameterized non-linear systems and neural networks},
  author={Liu, C. and Zhu, L. and Belkin, M.},
  journal={Applied and Computational Harmonic Analysis},
  volume={59},
  pages={85--116},
  year={2022},
  publisher={Elsevier}
}

@inproceedings{NEURIPS2020_b7ae8fec,
 author = {Liu, C. and Zhu, L. and Belkin, M.},
 booktitle = {Advances in Neural Information Processing Systems},
 pages = {15954--15964},
 title = {On the linearity of large non-linear models: when and why the tangent kernel is constant},
 year = {2020}
}

@article{lee2019wide,
  title={Wide neural networks of any depth evolve as linear models under gradient descent},
  author={Lee, J. and Xiao, L. and Schoenholz, S. and Bahri, Y. and Novak, R. and Sohl-Dickstein, J. and Pennington, J.},
  journal={Advances in neural information processing systems},
  volume={32},
  year={2019}
}

@InProceedings{pmlr-v97-du19c,
  title = 	 {Gradient Descent Finds Global Minima of Deep Neural Networks},
  author =       {Du, S. and Lee, J. and Li, H. and Wang, L. and Zhai, X.},
  booktitle = 	 {Proceedings of the 36th International Conference on Machine Learning},
  year = 	 {2019},
  volume = 	 {97},
}

@inproceedings{NEURIPS2018_5a4be1fa,
 author = {Jacot, A. and Gabriel, F. and Hongler, C.},
 booktitle = {Advances in Neural Information Processing Systems},
 title = {Neural Tangent Kernel: Convergence and Generalization in Neural Networks},
 year = {2018}
}

@inproceedings{NEURIPS2020_9afe487d,
 author = {Chen, Z. and Cao, Y. and Gu, Q. and Zhang, T.},
 booktitle = {Advances in Neural Information Processing Systems},
 pages = {13363--13373},
 title = {A Generalized Neural Tangent Kernel Analysis for Two-layer Neural Networks},
 volume = {33},
 year = {2020}
}

@article{jabir2021meanfield,
      title={Mean-Field Neural {ODEs} via Relaxed Optimal Control},
      author={J.-F. Jabir and D. {\v S}i{\v s}ka and {\L}. Szpruch},
      year={2021},
      journal={arxiv/1912.05475},
}

@article{E_2020,
   title={Machine learning from a continuous viewpoint, {I}},
   volume={63},
   number={11},
   journal={Science China Mathematics},
   author={E, W. and Ma, C. and Wu, L.},
   year={2020},
   month={Sep},
   pages={2233–2266}
}

@article{yang2019scaling,
  title={Scaling limits of wide neural networks with weight sharing: Gaussian process behavior, gradient independence, and neural tangent kernel derivation},
  author={Yang, G.},
  journal={arXiv preprint arXiv:1902.04760},
  year={2019}
}

@article{wojtowytsch2020convergence,
      title={On the Convergence of Gradient Descent Training for Two-layer {ReLU}-networks in the Mean Field Regime},
      author={S. Wojtowytsch},
      year={2020},
      eprint={2005.13530},
      journal={arXiv/2005/13530},
}

@article{nguyen2021rigorous,
      title={A Rigorous Framework for the Mean Field Limit of Multilayer Neural Networks},
      author={P.-M. Nguyen and H. Pham},
      year={2021},
      journal={arxiv/2001.11443},
      archivePrefix={arXiv},
      primaryClass={cs.LG}
}

@article{e2020mathematical,
author = {E , W. and Ma , C. and Wu , L. and Wojtowytsch , S.},
title = {Towards a Mathematical Understanding of Neural Network-Based Machine Learning: What We Know and What We Don't},
journal = {CSIAM Transactions on Applied Mathematics},
year = {2020},
volume = {1},
number = {4},
pages = {561--615},
}

@article{arora2019exact,
      title={On Exact Computation with an Infinitely Wide Neural Net},
      author={S. Arora and S. Du and W. Hu and Z. Li and R. Salakhutdinov and R. Wang},
      year={2019},
      eprint={1904.11955},
      journal={arXiv/1904.11955},
      primaryClass={cs.LG}
}

@article{zhang2019convergence,
      title={Convergence Theory of Learning Over-parameterized {ResNet}: A Full Characterization},
      author={H. Zhang and D. Yu and M. Yi and W. Chen and T. Liu},
      year={2019},
      journal={arXiv/1903.07120},
      primaryClass={cs.LG}
}

@InProceedings{pmlr-v119-lu20b,
  title =    {A Mean Field Analysis Of Deep {R}es{N}et And Beyond: Towards Provably Optimization Via Overparameterization From Depth},
  author =       {Lu, Y. and Ma, C. and Lu, Y. and Lu, J. and Ying, L.},
  booktitle =    {Proceedings of the 37th International Conference on Machine Learning},
  pages =    {6426--6436},
  year =   {2020},
}

@article {MeiE7665,
	author = {Mei, S. and Montanari, A. and Nguyen, P.},
	title = {A mean field view of the landscape of two-layer neural networks},
	volume = {115},
	number = {33},
	pages = {E7665--E7671},
	year = {2018},
	journal = {Proceedings of the National Academy of Sciences}
}

@article{zou2019improved,
  title={An improved analysis of training over-parameterized deep neural networks},
  author={Zou, Difan and Gu, Quanquan},
  journal={Advances in neural information processing systems},
  volume={32},
  year={2019}
}

@InProceedings{pmlr-v70-jin17a,
  title = 	 {How to Escape Saddle Points Efficiently},
  author =       {C. Jin and R. Ge and P. Netrapalli and S. Kakade and M. Jordan},
  booktitle = 	 {Proceedings of the 34th International Conference on Machine Learning},
  pages = 	 {1724--1732},
  year = 	 {2017},
}

@inproceedings{NIPS2017_f79921bb,
 author = {Du, S. and Jin, C. and Lee, J. and Jordan, M. and Singh, A. and P\'oczos, B.},
 booktitle = {Advances in Neural Information Processing Systems},
 title = {Gradient Descent Can Take Exponential Time to Escape Saddle Points},
 year = {2017}
}

@inproceedings{
yun2018global,
title={Global Optimality Conditions for Deep Neural Networks},
author={C. Yun and S. Sra and A. Jadbabaie},
booktitle={International Conference on Learning Representations},
year={2018},
}

@article{10.1109/TIT.2018.2854560,
author = {Soltanolkotabi, M. and Javanmard, A. and Lee, J.},
title = {Theoretical Insights Into the Optimization Landscape of Over-Parameterized Shallow Neural Networks},
year = {2019},
issue_date = {February 2019},
volume = {65},
number = {2},
journal = {IEEE Trans. Inf. Theor.},
month = feb,
pages = {742–769},
numpages = {28}
}

@InProceedings{pmlr-v70-nguyen17a, title = {The Loss Surface of Deep and Wide Neural Networks}, author = {Q. Nguyen and M. Hein}, booktitle = {Proceedings of the 34th International Conference on Machine Learning}, pages = {2603--2612}, year = {2017}}

@InProceedings{pmlr-v80-du18a,
  title = 	 {On the Power of Over-parametrization in Neural Networks with Quadratic Activation},
  author =       {Du, S. and Lee, J.},
  booktitle = 	 {Proceedings of the 35th International Conference on Machine Learning},
  pages = 	 {1329--1338},
  year = 	 {2018},
}

@inproceedings{ge2018learning,
title={Learning One-hidden-layer Neural Networks with Landscape Design},
author={R. Ge and J. Lee and T. Ma},
booktitle={International Conference on Learning Representations},
year={2018},
}

@inproceedings{NIPS2016_f2fc9902,
 author = {Kawaguchi, K.},
 booktitle = {Advances in Neural Information Processing Systems},
 title = {Deep Learning without Poor Local Minima},
 year = {2016}
}

@inproceedings{Nguyen2018OptimizationLA,
  title={Optimization Landscape and Expressivity of Deep {CNNs}},
  author={Q. Nguyen and M. Hein},
  booktitle={Proceedings of the 35th International Conference on Machine
Learning},
  year={2018}
}

@InProceedings{pmlr-v40-Ge15,
  title = 	 {Escaping From Saddle Points-Online Stochastic Gradient for Tensor Decomposition},
  author = 	 {R. Ge and F. Huang and C. Jin and Y. Yuan},
  booktitle = 	 {Proceedings of The 28th Conference on Learning Theory},
  pages = 	 {797--842},
  year = 	 {2015},
}

@article{doi:10.1137/18M1192184,
author = {Sirignano, J. and Spiliopoulos, K.},
title = {Mean Field Analysis of Neural Networks: A Law of Large Numbers},
journal = {SIAM Journal on Applied Mathematics},
volume = {80},
number = {2},
pages = {725-752},
year = {2020},
}

@article{doi:10.1287/moor.2020.1118,
author = {Sirignano, J. and Spiliopoulos, K.},
title = {Mean Field Analysis of Deep Neural Networks},
journal = {Mathematics of Operations Research},
mouth = {Apr},
year = {2021},
}

@inproceedings{
du2018gradient,
title={Gradient Descent Provably Optimizes Over-parameterized Neural Networks},
author={S. Du and X. Zhai and B. Poczos and A. Singh},
booktitle={International Conference on Learning Representations},
year={2019},
}

@InProceedings{pmlr-v97-allen-zhu19a,
  title = 	 {A Convergence Theory for Deep Learning via Over-Parameterization},
  author =       {Allen-Zhu, Z. and Li, Y. and Song, Z.},
  booktitle = 	 {Proceedings of the 36th International Conference on Machine Learning},
  pages = 	 {242--252},
  year = 	 {2019},
}

@article{araujo2019meanfield,
      title={A mean-field limit for certain deep neural networks},
      author={D. Ara\'ujo and R. Oliveira and D. Yukimura},
      year={2019},
      journal={arXiv/1906.00193},
      primaryClass={math.ST}
}

@article{fang2019convex,
      title={Convex Formulation of Overparameterized Deep Neural Networks},
      author={C. Fang and Y. Gu and W. Zhang and T. Zhang},
      year={2019},
      journal={arXiv/1911/07626},
      primaryClass={cs.LG}
}

@article{nguyen2019mean,
      title={Mean Field Limit of the Learning Dynamics of Multilayer Neural Networks},
      author={P. M. Nguyen},
      year={2019},
      journal={arXiv/1902.02880},
      primaryClass={cs.LG}
}

@inproceedings{NEURIPS2018_a1afc58c,
 author = {Chizat, L. and Bach, F.},
 booktitle = {Advances in Neural Information Processing Systems},
 title = {On the Global Convergence of Gradient Descent for Over-parameterized Models using Optimal Transport},
 year = {2018}
}

@article{cao2019generalization,
  title={Generalization bounds of stochastic gradient descent for wide and deep neural networks},
  author={Cao, Y. and Gu, Q.},
  journal={Advances in neural information processing systems},
  volume={32},
  year={2019}
}
\end{document}